\documentclass[11pt]{article}
\usepackage[a4paper,top=3cm,bottom=2cm,left=2cm,right=2cm,marginparwidth=1.75cm]{geometry}

\usepackage[utf8]{inputenc} \usepackage[T1]{fontenc}
\usepackage[numbers]{natbib}
\usepackage[colorlinks=True,citecolor=blue,urlcolor=blue,pagebackref=true,backref]
{hyperref}
\renewcommand*{\backrefalt}[4]{\ifcase #1 \footnotesize{(Not cited.)}\or        \footnotesize{(Cited on page~#2.)}\else      \footnotesize{(Cited on pages~#2.)}\fi}

\usepackage{breakcites}
\usepackage[utf8]{inputenc} \usepackage[T1]{fontenc}    \usepackage{hyperref}       \hypersetup{
  colorlinks = true,
  urlcolor = blue,
  linkcolor = blue,
  citecolor = blue
}
\usepackage{booktabs}       \usepackage{amsfonts}       \usepackage{nicefrac}       \usepackage{microtype}      \usepackage{xcolor}         \usepackage{caption}
\usepackage{subcaption}
\usepackage{float}
\usepackage{amsmath}
\usepackage{amsthm}
\usepackage{amssymb}
\usepackage{mathtools}
\usepackage{natbib}
\usepackage{soul}
\usepackage{bm}
\usepackage{enumitem}
\usepackage{multirow}
\usepackage{wrapfig}
\usepackage{scalerel}
\allowdisplaybreaks
\usepackage{dsfont}
\usepackage[noend]{algpseudocode}
\usepackage{xcolor}
\usepackage{thmtools,thm-restate}
\usepackage{cleveref}

 \newtheorem{theorem}{Theorem}
 \newtheorem{definition}{Definition}
 \newtheorem{lemma}{Lemma}
 
\newtheorem{proposition}{Proposition}
 \newtheorem{corollary}{Corollary}

\renewcommand{\citet}{\citep}

\title{Identification and Adaptive Control of Markov Jump Systems:\\ Sample Complexity and Regret Bounds}
\author{\begin{tabular}{ c c c }
Yahya Sattar\thanks{equal contribution} &  Zhe Du\footnotemark[1]  & Davoud Ataee Tarzanagh \\ Cornell & U Michigan  & U Pennsylvania  \\ ysattar@cornell.edu & zhedu@umich.edu & tarzanaq@upenn.edu \\ \\
Laura Balzano  & Necmiye Ozay & Samet Oymak \\ U Michigan & U Michigan & U Michigan  \\ girasole@umich.edu  & necmiye@umich.edu & oymak@umich.edu
\end{tabular}
}

\date{}
\usepackage{hyperref,graphicx,amsmath,amsfonts,amssymb,bm,url,epsfig,epsf,color,MnSymbol,fmtcount,semtrans,cite,subcaption, multirow,comment}
\usepackage[noend]{algpseudocode}

\usepackage{amssymb}

\newcommand*{\Tr}{{\mathpalette\@transpose{}}}

\newcommand{\ysatt}[1]{\marginpar{\color{magenta}\tiny\ttfamily YS: #1}}
\usepackage[utf8]{inputenc} \usepackage{url}            \usepackage{booktabs}       \usepackage{amsfonts}       \usepackage{nicefrac}       \usepackage{microtype}      

\makeatletter
\providecommand*{\boxast}{\mathbin{\mathpalette\@boxit{*}}}
\newcommand*{\@boxit}[2]{\sbox0{$\m@th#1\Box$}\ifx#1\displaystyle \ht0=\dimexpr\ht0+.05ex\relax \fi
	\ifx#1\textstyle \ht0=\dimexpr\ht0+.05ex\relax \fi
	\ifx#1\scriptstyle \ht0=\dimexpr\ht0+.04ex\relax \fi
	\ifx#1\scriptscriptstyle \ht0=\dimexpr\ht0+.065ex\relax \fi
	\sbox2{$#1\vcenter{}$}\rlap{\hbox to \wd0{\hfill
			\raisebox{\dimexpr.5\dimexpr\ht0+\dp0\relax-\ht2\relax
			}{$\m@th#1#2$}\hfill
		}}\Box
}
\makeatother

\makeatletter
\def\BState{\State\hskip-\ALG@thistlm}
\makeatother

\usepackage{mathtools}

\usepackage{tikz}

\usepackage{cite}
\usepackage[bottom,hang,flushmargin]{footmisc} 

\newcommand{\tsn}[1]{{\left\vert\kern-0.25ex\left\vert\kern-0.25ex\left\vert #1 
		\right\vert\kern-0.25ex\right\vert\kern-0.25ex\right\vert}}

\definecolor{darkred}{RGB}{150,0,0}
\definecolor{darkgreen}{RGB}{0,150,0}
\definecolor{darkblue}{RGB}{0,0,200}
\hypersetup{colorlinks=true, linkcolor=darkblue, citecolor=darkblue, urlcolor=darkblue}

\usepackage{adjustbox}
\usepackage{upgreek}
\usepackage[ruled, vlined, linesnumbered]{algorithm2e}

\usepackage{xstring}

\usepackage{multirow}
\usepackage{calc}

\usepackage{bbm}

\newtheorem{assumption}{Assumption}

\newcommand{\polylog}{\textup{polylog}}

\newcommand{\bv}[1]{{\boldsymbol{#1}}}	\newcommand{\bvgrk}[1]{{\boldsymbol{#1}}}

\newcommand{\va}{\bv{a}}

\newcommand{\vh}{\bv{h}}

\newcommand{\vs}{\bv{s}}

\newcommand{\vu}{\bv{u}}
\newcommand{\vv}{\bv{v}}
\newcommand{\vw}{\bv{w}}
\newcommand{\vx}{\bv{x}}
\newcommand{\vy}{\bv{y}}
\newcommand{\vz}{\bv{z}}

\newcommand{\hb}{\bv{h}}

\newcommand{\ub}{\bv{u}}
\newcommand{\wb}{\bv{w}}
\newcommand{\xb}{\bv{x}}

\newcommand{\yb}{\bv{y}}
\newcommand{\zb}{\bv{z}}

\newcommand{\Jhat}{\hat{J}}

\newcommand{\Ttil}{\tilde{T}}

\newcommand{\Acal}{\mathcal{A}}
\newcommand{\Bcal}{\mathcal{B}}
\newcommand{\Ccal}{\mathcal{C}}
\newcommand{\Dcal}{\mathcal{D}}
\newcommand{\Ecal}{\mathcal{E}}
\newcommand{\Fcal}{\mathcal{F}}

\newcommand{\Hcal}{\mathcal{H}}

\newcommand{\Ncal}{\mathcal{N}}
\newcommand{\Ocal}{\mathcal{O}}

\newcommand{\Scal}{\mathcal{S}}

\newcommand{\vA}{\bv{A}}
\newcommand{\vB}{\bv{B}}
\newcommand{\vC}{\bv{C}}

\newcommand{\vE}{\bv{E}}
\newcommand{\vF}{\bv{F}}
\newcommand{\vG}{\bv{G}}
\newcommand{\vH}{\bv{H}}
\newcommand{\vI}{\bv{I}}

\newcommand{\vK}{\bv{K}}
\newcommand{\vL}{\bv{L}}
\newcommand{\vM}{\bv{M}}
\newcommand{\vN}{\bv{N}}

\newcommand{\vP}{\bv{P}}
\newcommand{\vQ}{\bv{Q}}
\newcommand{\vR}{\bv{R}}
\newcommand{\vS}{\bv{S}}
\newcommand{\vT}{\bv{T}}

\newcommand{\vV}{\bv{V}}

\newcommand{\vX}{\bv{X}}

\newcommand{\Ab}{\bv{A}}
\newcommand{\Bb}{\bv{B}}

\newcommand{\Hb}{\bv{H}}

\newcommand{\Kb}{\bv{K}}
\newcommand{\Lb}{\bv{L}}

\newcommand{\Qb}{\bv{Q}}
\newcommand{\Rb}{\bv{R}}

\newcommand{\Tb}{\bv{T}}

\newcommand{\Wb}{\bv{W}}

\newcommand{\Yb}{\bv{Y}}

\newcommand{\vAhat}{\hat{\bv{A}}}
\newcommand{\vBhat}{\hat{\bv{B}}}

\newcommand{\vKhat}{\hat{\bv{K}}}

\newcommand{\vThat}{\hat{\bv{T}}}

\newcommand{\vAstar}{\bv{A}^\star}

\newcommand{\vKstar}{\bv{K}^\star}
\newcommand{\vLstar}{\bv{L}^\star}

\newcommand{\vPstar}{\bv{P}^\star}

\newcommand{\vBtil}{\tilde{\bv{B}}}

\newcommand{\vItil}{\tilde{\bv{I}}}

\newcommand{\vLtil}{\tilde{\bv{L}}}

\newcommand{\vRtil}{\tilde{\bv{R}}}

\newcommand{\vGamma}{\bvgrk{\Gamma}}
\newcommand{\vdelta}{\bvgrk{\updelta}}
\newcommand{\vDelta}{\bvgrk{\Delta}}

\newcommand{\vtheta}{\bvgrk{\uptheta}}
\newcommand{\vTheta}{\bvgrk{\Theta}}

\newcommand{\vpi}{\bvgrk{\uppi}}
\newcommand{\vPi}{\bvgrk{\Pi}}

\newcommand{\vSigma}{\bvgrk{\Sigma}}

\newcommand{\vphi}{\bvgrk{\upphi}}

\newcommand{\vPitil}{\tilde{\vPi}}
\newcommand{\nn}{\nonumber}

\newcommand{\E}{\operatorname{\mathbb{E}}}

\newcommand{\Fc}{\mathcal{F}}
\newcommand{\tf}[1]{\|{#1}\|_{F}}

\newcommand{\Nc}{\mathcal{N}}

\newcommand{\bTeta}{\boldsymbol{\Theta}}
\newcommand{\bTetas}{\boldsymbol{\Theta}^\star}
\renewcommand{\P}{\operatorname{\mathbb{P}}}
\newcommand{\leqsym}[1]{\stackrel{\text{(#1)}}{\leq}}
\newcommand{\geqsym}[1]{\stackrel{\text{(#1)}}{\geq}}
\newcommand{\eqsym}[1]{\stackrel{\text{(#1)}}{=}}

\newcommand{\Abs}{\bv{A}}
\newcommand{\Bbs}{\bv{B}}

\newcommand{\Pbs}{\bv{P}^\star}

\newcommand{\norm}[1]{\|{#1} \|}

\newcommand{\curlybracketsbig}[1]{\left\{ #1 \right\}}
\newcommand{\curlybrackets}[1]{\{ #1 \}}

\newcommand{\squarebrackets}[1]{[ #1 ]}
\newcommand{\parenthesesbig}[1]{\left( #1 \right)}
\newcommand{\parentheses}[1]{( #1 )}

\newcommand{\indicator}[1]{\mathbf{1}_{\curlybrackets{#1}}}

\newcommand{\R}{\mathbb{R}}				

\newcommand{\dm}[2]
{
	\IfStrEq{#2}{1}{\R^{#1}}{\R^{#1 \x #2}}
}

\newcommand{\N}{\mathcal{N}}			\newcommand{\onevec}{\mathbf{1}} 		

 			\newcommand{\tr}{\textup{\textbf{tr}}} 			 		 		\newcommand{\vek}{\textup{\textbf{vec}}}			\newcommand{\expctn}{\mathbb{E}} 	 	 \newcommand{\inv}{{-1}} 				 				 								\newcommand{\fro}{\textup{F}}			\newcommand{\distas}{\overset{\text{i.i.d.}}{\sim}}

\makeatletter
\renewcommand*{\T}{{\mathpalette\@transpose{}}}
\newcommand*{\@transpose}[2]{\raisebox{\depth}{$\m@th#1\intercal$}}
\makeatother

\newcommand*{\x}{\times \mskip1mu}

\newcommand{\splitatcommas}[1]{\begingroup
	\begingroup\lccode`~=`, \lowercase{\endgroup
		\edef~{\mathchar\the\mathcode`, \penalty0 \noexpand\hspace{0pt plus .1em}}}\mathcode`,="8000 #1\endgroup
}

\usepackage{bbm}
\newcommand{\Item}[1]{\ifx\relax#1\relax  \item \else \item[#1] \fi
	\abovedisplayskip=0pt\abovedisplayshortskip=0pt~\vspace*{-\baselineskip}
}

\newcommand{\subItvl}{L}

\usepackage{booktabs}
\usepackage{multirow}

\usepackage{mathtools}
\makeatletter
\newcommand{\raisemath}[1]{\mathpalette{\raisem@th{#1}}}\newcommand{\raisem@th}[3]{\raisebox{#1}{$#2#3$}}
\makeatother
\newcommand{\barbarepsilon}{\raisemath{0.7pt}{{\mathchar'26 \mkern-9mu}}\raisemath{-0.7pt}{{\mathchar'26 \mkern-9mu}}\epsilon}

\providecommand{\customgenericname}{}
\newtheorem{innercustomgeneric}{\customgenericname}

\newcommand{\newcustomtheorem}[2]{\newenvironment{#1}[1]
  {\renewcommand\customgenericname{#2}\renewcommand\theinnercustomgeneric{##1}\innercustomgeneric
  }
  {\endinnercustomgeneric}
}
\newcustomtheorem{customtheorem}{Theorem}
\newcustomtheorem{customlemma}{Lemma}
\newcustomtheorem{customproposition}{Proposition}

\newcommand{\vertiii}[1]{{\vert\kern-0.25ex\vert\kern-0.25ex\vert #1 \vert\kern-0.25ex\vert\kern-0.25ex\vert}}

\newcommand{\beq}{\begin{equation}}

\newcommand{\eeq}{\end{equation}}

\newcommand{\li}{\left<}
\newcommand{\ri}{\right>}

\newcommand{\bgl}{{~\big |~}}

\definecolor{emmanuel}{RGB}{255,127,0}

\newcommand{\Z}{\mathbb{Z}}

\renewcommand{\P}{\operatorname{\mathbb{P}}}

\numberwithin{equation}{section} 

\def \endprf{\hfill {\vrule height6pt width6pt depth0pt}\medskip}
\newenvironment{proofsk}{\noindent {\bf Proof sketch} }{\endprf\par}

\usepackage{stackengine}
\usepackage{scalerel}
\newcommand\mysqrt[2][0pt]{\stretchrel{\sqrt{}}{\addstackgap [#1]{$\displaystyle\overline{#2}$}}}

\newcommand{\numSys}{s}
\newcommand{\dimSt}{n}
\newcommand{\dimInput}{p}
\newcommand{\epkidx}{q}
\newcommand{\ofst}{\ell}

\let\debugmode\undefined \def\debugmode{}

\ifdefined\debugmode
            \newcommand{\blue}[1]{{\color{blue} #1}}
    \newcommand{\red}[1]{{\color{red} #1}}
    
    \newcommand{\zheinline}[1]{{\color{orange} (Zhe: #1)}}
    \newcommand{\zhe}[1]{\marginpar{\hspace{-10pt}\color{orange}\tiny\ttfamily Zhe: #1}}
\newcommand{\necmiye}[1]{\marginpar{\hspace{-10pt}\color{green}\tiny\ttfamily Nec: #1}}
  
    \newcommand{\laura}[1]{{\color{cyan} #1}}
    \newcommand{\laurac}[1]{\marginpar{\hspace{-10pt}\color{cyan}\tiny\ttfamily Laura: #1}}
\newcommand{\SO}[1]{\marginpar{\hspace{-10pt}\color{darkgreen}\tiny\ttfamily SO: #1}}
\newcommand{\parentchild}[4]{}
\else
    \newcommand{\red}[1]{}
    \newcommand{\blue}[1]{}
    
    \newcommand{\zheinline}[1]{}
    \newcommand{\zhe}[1]{}
    \newcommand{\ysatt}[1]{}
    \newcommand{\necmiye}[1]{}
    \newcommand{\laura}[1]{}
    \newcommand{\laurac}[1]{}
    
    \newcommand{\SO}{}
    \newcommand{\parentchild}[4]{}
\fi 
\begin{document}

\maketitle

\begin{abstract}                          
	Learning how to effectively control unknown dynamical systems is crucial for intelligent autonomous systems. This task becomes a significant challenge when the underlying dynamics are changing with time. Motivated by this challenge, this paper considers the problem of controlling an unknown Markov jump linear system (MJS) to optimize a quadratic objective. By taking a model-based perspective, we consider identification-based adaptive control of MJSs. We first provide a system identification algorithm for MJS to learn the dynamics in each mode as well as the Markov transition matrix, underlying the evolution of the mode switches, from a single trajectory of the system states, inputs, and modes. Through martingale-based arguments, sample complexity of this algorithm is shown to be $\mathcal{O}(1/\sqrt{T})$. We then propose an adaptive control scheme that performs system identification together with certainty equivalent control to adapt the controllers in an episodic fashion. Combining our sample complexity results with recent perturbation results for certainty equivalent control, we prove that when the episode lengths are appropriately chosen, the proposed adaptive control scheme achieves $\mathcal{O}(\sqrt{T})$ regret, which can be improved to $\mathcal{O}(\polylog(T))$ with partial knowledge of the system. Our proof strategy introduces innovations to handle Markovian jumps and a weaker notion of stability common in MJSs. Our analysis provides insights into system theoretic quantities that affect learning accuracy and control performance. Numerical simulations are presented to further reinforce these insights.
\end{abstract}

\section{Introduction}\label{sec:intro}
A canonical problem at the intersection of machine learning and control is that of adaptive control of an unknown dynamical system. An intelligent autonomous system is likely to encounter such a task; from an observation of the inputs and outputs, it needs to both learn and effectively control the dynamics. A commonly used control paradigm is the Linear Quadratic Regulator (LQR), which is theoretically well understood when system dynamics are linear and known. LQR also provides an interesting benchmark, when system dynamics are unknown, for reinforcement learning (RL) with continuous state and action spaces and for adaptive control~\citep{campi1998adaptive,abbasi2011regret,dean2019sample,mania2019certainty,lale2020explore,abeille2020efficient}. A generalization of linear dynamical systems called Markov jump linear systems (MJSs) models dynamics that switch between multiple linear systems, called modes, according to an underlying finite Markov chain. MJS allows for modeling a richer set of problems where the underlying dynamics can abruptly change over time. One can, similarly, generalize the LQR paradigm to MJS by using mode-dependent cost matrices, which allow different control goals under different modes. 
For instance, a Mars rover optimally exploring an unknown heterogeneous terrain, optimal solar power generation on a cloudy day, or controlling investments in financial markets may be modeled as MJS-LQR problems with unknown system dynamics~\citep{loparo1990probabilistic,cajueiro2002stochastic,ugrinovskii2005decentralized,blackmore2005combining,svensson2008optimal}.

While the MJS-LQR problem is well understood when one has perfect knowledge of the system dynamics \citep{chizeck1986discrete,costa2006discrete}, in practice, such knowledge is not always possible, and one may have to resort to adaptive control.
Earlier works have aimed at analyzing the asymptotic properties (i.e., stability) of adaptive controllers for unknown MJSs both in continuous-time \citep{caines1995adaptive} and discrete-time \citep{xue2001necessary} settings. However, despite the practical importance of MJSs, non-asymptotic sample complexity results and regret analysis for MJSs are lacking. When the Markovian modes switch in an i.i.d. fashion, and the Markov matrix is the only unknown, recent works study data-driven stability verification \citep{gatsis2021statistical} and stabilization \citep{schuurmans2019safe} with non-asymptotic guarantees. However, it is difficult to extend these works to more general MJSs with completely unknown dynamics.
One major challenge brought by MJSs is that one needs to consider both the state/input in the continuous space and the Markovian mode switching sequence in the discrete space. Furthermore, the state data generated by the same mode are temporally separated with the mode switching, thus having time-varying statistical properties and posing difficulties to sample complexity analysis.

One advantage of MJSs is that, when solving control problems, stabi(lizabi)lity is only required in the \emph{mean-square sense}, which relaxes the deterministic counterpart that is commonly needed for non-switched systems. This, however, brings new challenges to the analysis since unstable realization is possible with mean-square stability.
Figure~\ref{fig:syntraj} shows an example (adapted from \citet{costa2006discrete}) of an MJS that is stable in the mean-square sense despite having an unstable mode. Clearly, under an unfavorable mode switching sequence, the system trajectory can still blow up. Therefore, statistical tools such as high probability light-tail bounds are not applicable without strong assumptions on the joint spectral radius of the system (cf.~\citet{sarkar2019nonparametric}). Perhaps more surprisingly, there are examples of MJS with all modes individually stable, however due to switching, the system exhibits an unstable behavior on average, and the MJS is not mean-square stable~\citep[Example 3.17]{costa2006discrete}. Therefore, finding controllers to individually stabilize the mode dynamics does not guarantee that the overall system will be stable when mode switches over time.
\begin{figure}
	\centering
	\includegraphics[width=0.4\linewidth]{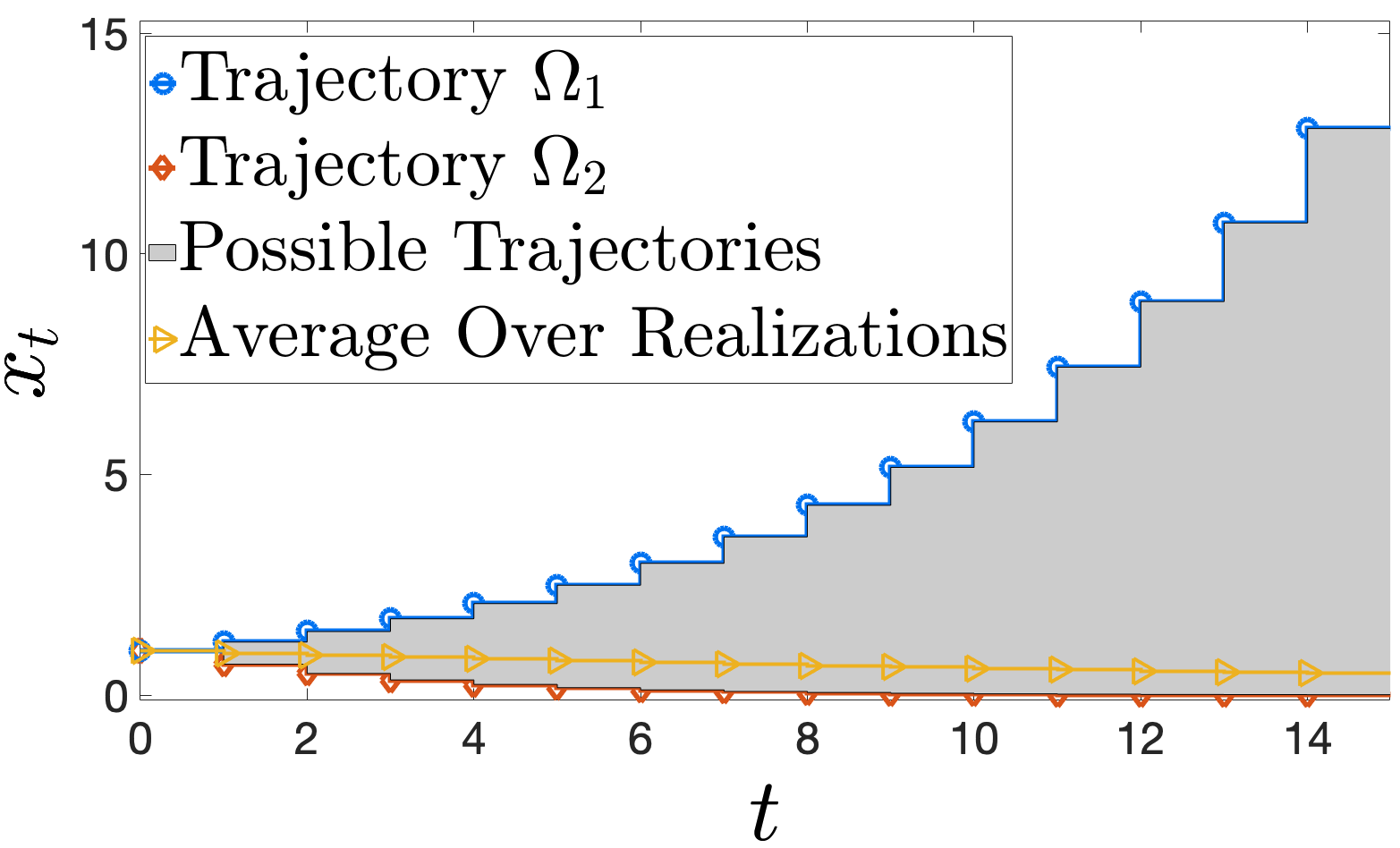}
	\vspace{-2mm}
	\caption{\footnotesize 
		State trajectories for a two-modes MJS: Mode 1: $x_{t+1} = 1.2 x_{t}$, Mode 2: $x_{t+1} = 0.7 x_{t}$, Markov matrix $[[0.6, 0.4]^\top, [0.3, 0.7]^\top]^\top$, and $x_0=1$.
Blue and red curves: mode switching sequences $\Omega_1=\curlybrackets{1,1,\dots}$ and $\Omega_2=\curlybrackets{2,2,\dots}$. Yellow curve: average over all realizations. Gray area: region for all possible trajectories.}\label{fig:syntraj}
\end{figure}
This more relaxed notion of \emph{mean-square stability} presents major challenges in learning, controlling, and statistical analysis.

\noindent\textbf{Contributions:} In this paper, we provide the first comprehensive system identification and regret guarantees for learning and controlling Markov jump linear systems using a single trajectory while assuming only mean-square stability~(see Definition~\ref{def_mss}). 
Specifically, our contributions are as follows\footnote{orders of magnitude here are up to polylogarithmic factors}:
\begin{itemize}
	\item \textbf{System identification:} 
We provide an algorithm~(Algorithm.~\ref{Alg_MJS-SYSID}) to estimate the MJS dynamics with an error rate of $\mathcal{O}(\sqrt{(n+p)/T})$, where $n$ and $p$ are the state and input dimensions respectively, and $T$ is the trajectory length. Our error rate is optimal in terms of the trajectory length $T$ and the dimensions~($n$ and $p$) of the unknown matrices.
	
	\item $\mathcal{O}(\sqrt{T})$\textbf{-regret bound:} We employ our system identification results to solve the adaptive MJS-LQR with unknown dynamics. The proposed certainty-equivalent adaptive MJS-LQR algorithm (Algorithm.~\ref{Alg_adaptiveMJSLQR}) achieves a regret bound of $\mathcal{O}(\sqrt{T})$ under multiple notions of MJS stability. Remarkably, this coincides with the optimal regret bound for the standard LQR problem obtained via certainty equivalence \citep{mania2019certainty}. 
	
	\item \textbf{$\mathcal{O}(\polylog(T))$-regret with partial knowledge:} We also consider the practically relevant setting where the state matrices are unknown but the input matrices are known. We show that the regret bound can be significantly improved to $\mathcal{O}(\polylog(T))$. This bound also coincides with the polylogarithmic regret bound for the standard LQR with the knowledge of the input matrix~\citep{cassel2020logarithmic}.
\end{itemize}

\section{Related Work} \label{sec:related}
Our work is related to several topics in model-based reinforcement learning, system identification, and adaptive control. A comparison with the related works, in the LQR setting, is provided in Table~\ref{tab:sum:result}. 

$\bullet$ {\bf{System Identification:}} Learning dynamical models has a long history in the control community, with major theoretical results being related to asymptotic properties under strong assumptions on persistence of excitation \citep{ljung1999system}. The problem becomes harder for hybrid and switched systems where the initial focus was on computational complexity as opposed to sample complexity of learning \citep{ozay2011sparsification, lauer2018hybrid}.
There are some recent results on asymptotic consistency \citep{hespanhol2019statistically} in the stochastic jump systems, a special case of MJSs where the modes switch in an i.i.d. manner. Similarly, \citet{sayedana2022consistency} provides strong consistency result for learning MJSs using switched least squares. Identification of MJSs with hidden mode sequence has also attracted significant attention \citep{tugnait1982adaptive, fox2010bayesian}.  

$\bullet$ {\bf{Sample Complexity of System Identification:}} There is a recent surge of interest toward understanding the sample complexity of learning linear dynamical systems from a single trajectory under mild assumptions \citep{oymak2021revisiting}, using statistical tools like martingales \citep{ simchowitz2018learning,sarkar2019near,tsiamis2019finite} or mixing time arguments \citep{mohri2008stability,kuznetsov2017generalization}. Recently, \citet{jedra2020finite} provides precise rates for the finite-time identification of LTI (linear time-invariant) systems using a single trajectory. The literature gets scarcer for switched systems. In \citet{lale2020stability}, a novel approach based on Lyapunov equation is proposed for systems with stochastic switches, yet theoretical guarantees are lacking. \citet{sarkar2019nonparametric} is one of the early works to provide finite sample analysis for learning systems with stochastic switches, yet with additional strong assumptions like independent switches and small joint spectral radius. The proof techniques developed within our work aim to obviate such assumptions. Closer to our work, \citet{sayedana2024strong} studies the problem of system identification for autonomous~(no control inputs) MJSs with perfect state observations. In contrast to our paper, the error bounds in \citet{sayedana2024strong} are asymptotic and hold in the limit,  that is, it does not provide finite time guarantees. Our paper tackles the open problem of learning MJS from finite samples, obtained from a single trajectory, with theoretical guarantees under mild assumptions. 
The problem of learning mixture of linear dynamical systems or piecewise affine systems has recently attracted significant attention~\citep{chen2022learning,bakshi2023tensor,block2024smoothed,shi2022finite}. Besides learning, \citet{shi2022finite} also studies the effect of switching strategies~(arbitrary or subject to an average dwell time constraint) on the estimation error. 

$\bullet$ {\bf{Learning-based Control and Regret Analysis:}} As a direct application of single-trajectory system identification results, one can provide more sophisticated adaptive control guarantees from regret perspective \citep{abbasi2011regret,abbasi2019model,dean2019sample,mania2019certainty,faradonbeh2020optimism,hazan2020nonstochastic}. Specifically, \citet{simchowitz2020naive} achieves $\mathcal{O}(\sqrt{T})$ regret lower bound for adaptive LQR control, while \citet{cassel2020logarithmic} and \citet{lale2020logarithmic} achieve logarithmic regret upper bound, with partial knowledge of the system. However, in the MJS setting, due to the lack of well established identification analysis, prior works \citep{caines1995adaptive, xue2001necessary} provide guarantees from the stability aspect. The case of input design without system state dynamics is considered in~\citet{baltaoglu2016online}, which can be thought of as a generalization of linear bandits to have a Markovian structure in the reward function without any continuous dynamic structure. However, only a regret lower bound is provided in \citet{baltaoglu2016online}. 
More recently, \citet{sayedana2023relative} proposes a certainty equivalence-based adaptive control algorithm for MJSs, and shows that it achieves a regret of $\Ocal(\sqrt{T})$ relative to a certain subset of the sample space. Finally, we refer the reader to the survey papers \citet{gaudio2019connections,matni2019self,recht2019tour} for a broad overview of the recent developments on non-asymptotic system identification, adaptive control and reinforcement learning from the perspective of optimization and control.

$\bullet$ {\bf{Model-free Approaches:}} Somehow orthogonal to the above developments, but still highly relevant, are approaches that sidestep system identification and try to learn an optimal controller (policy) directly (among many others, see e.g., \citet{fazel2018global,mohammadi2020linear,zhang2020policy,zheng2021analysis}). These works analyze the optimization landscape of LQR and related optimal control problems and provide polynomial-time algorithms that lead to a globally convergent search in the space of controllers. Importantly, these optimization algorithms do not require the knowledge of the system parameters as long as relevant quantities like gradients can be approximated from simulated system trajectories. More recently, this line of work is extended to MJSs in \citet{jansch2020policy}, significantly expanding their utility. However, these works require multiple trajectories to estimate the gradients as opposed to a controller that adapts at run-time, therefore, they provide a complementary perspective to the single trajectory adaptive control and regret analysis in our work.

A preliminary version of this work has been published at the American Control Conference 2022~\citet{du2022data}, where we provide preliminary guarantees for the data-driven adaptive control of MJS. In contrast to the current paper, Algorithm~1 in~\citet{du2022data} performs a sophisticated double sub-sampling to estimate the unknown MJS dynamics $(\Ab_i,\Bb_i)_{i=1}^s$ and $\Tb$. The reason of this double sub-sampling in~\citet{du2022data} is to facilitate learning of $(\Ab_i,\Bb_i)_{i=1}^s$ using mixing-time arguments. Algorithm~\ref{Alg_MJS-SYSID} in the current paper does not require any sub-sampling, because of the martingale-based arguments to estimate the unknown MJS dynamics from a single trajectory. 
Moreover, our new error bounds do not degrade with the decrease in stability, and capture the optimal dependence on the dimensions of the MJS $n, p$ and $s$. 
	Meanwhile, in terms of the adaptive control, we show that tighter regret bounds are attainable in two cases:
	(i) When the MJS is equipped with the uniform stability, a stability notion stronger than the mean-square stability, in the regret bound, the dependency on the failure probability $\delta$ can be improved from $1/ \delta$ to $\polylog(1/ \delta)$ (Section~\ref{subsubsec_UniformStab}).
	(ii) When the input matrices are known, the dependency on the planning horizon $T$ can be improved from $\Ocal(\sqrt{T})$ to $\Ocal(\polylog(T))$ (Section~\ref{subsubsec_knownB}).

\renewcommand{\arraystretch}{1}
\begin{table}[!ht]
	\centering
	\caption{Comparison with prior works in the LQR setting.}
	\begin{adjustbox}{max width=\columnwidth}
		\begin{tabular}{||c|l|l|c|c|c||} 
			\hline
			\textbf{Model} &\textbf{Reference} & \textbf{Regret} & \textbf{Computational} & \textbf{Cost} & \textbf{Stabilizability/} \\
			& &  & \textbf{Complexity} & & \textbf{Controllability} \\
			\hline\hline
			\multirow{8}{*}{LTI} &\citet{abbasi2011regret} & $\sqrt{T}$ & Exponential & Strongly Convex & Controllable\\
			&\citet{ibrahimi2012efficient}  &  $\sqrt{T}$& Exponential & Convex  &  Controllable \\
			&\citet{abeille2018improved}~(one dim. systems) & $\sqrt{T}$  & Polynomial & Strongly Convex & Stabilizable \\
			&\citet{dean2018regret} & $T^{2/3}$ & Polynomial & Convex & Stabilizable\\ 
			&\citet{mania2019certainty} &  $\sqrt{T}$ & Polynomial & Strongly Convex & Controllable\\
			&\citet{cohen2019learning} & $\sqrt{T}$& Polynomial & Strongly Convex & Strongly Stabilizable \\
			&\citet{faradonbeh2020adaptive, simchowitz2020naive}& $\sqrt{T}$ &Polynomial & Strongly Convex &  Stabilizable \\
			&\citet{cassel2020logarithmic} (known $\vA$ or $\vB$)& $\polylog(T)$ & Polynomial  & Strongly Convex & Strongly Stabilizable \\ \hline
			\multirow{3}{*}{MJS}& \citet{sayedana2023relative} & $s\sqrt{T}$ & Polynomial &  Strongly Convex & MSS \\
			& \textbf{Ours} & $s\sqrt{T}$ & Polynomial &  Strongly Convex & MSS \\
			&\textbf{Ours (known $\vB_{1:s}$)}& $s\,\polylog(T)$ & Polynomial & Strongly Convex & MSS\\
			\hline
		\end{tabular}
	\end{adjustbox}
	\label{tab:sum:result}
\end{table}
\section{Preliminaries and Problem Setup}\label{sec:setup}
\noindent \textbf{Notations:} We use boldface uppercase (lowercase) letters to denote matrices (vectors). 
For a matrix $\vV$, $\rho(\vV)$ denotes its spectral radius. We use $\|\cdot\|$ to denote the Euclidean norm of vectors as well as the spectral norm of matrices. Similarly, we use $\|\cdot\|_1$ to denote the $\ell_1$-norm of a matrix/vector. The Kronecker product of two matrices $\vM$ and $\vN$ is denoted as $\vM \otimes \vN$. 
$\vV_{1:\numSys}$ denotes a set of $s$ matrices $\curlybrackets{\vV_i}_{i=1}^\numSys$ of same dimensions.
We define $[\numSys]:= \curlybrackets{1,2, \dots, \numSys}$ and $\norm{\vV_{1:\numSys}} := \max_{i \in [\numSys]} \norm{\vV_i}$. 
The $i$-th row or column of a matrix $\vM$ is denoted by $[\vM]_{i,:}$  or $[\vM]_{:,i}$ respectively. Orders of magnitude notation $\hat{\Ocal}(\cdot)$ hides $\log(1/\delta)$ or $\log^2(1/\delta)$ terms.
\subsection{Markov Jump Linear Systems}\label{subsec_IntroMJS}
In this paper we consider the identification and adaptive control of MJSs which are governed by the following state equation,
\begin{equation}\label{switched LDS}
	\begin{aligned}
		\vx_{t+1} &= \Ab_{\omega(t)}\vx_t + \Bb_{\omega(t)}\ub_t + \vw_t \quad
		 \text{s.t.}
		\quad  \omega(t) &\sim \text{Markov Chain}(\vT),
	\end{aligned}
\end{equation}
where $\vx_t \in \R^n$, $\vu_t \in \R^p$ and $\vw_t\in \R^n$ are the state, input, and process noise of the MJS at time $t$ with $\{\vw_t\}_{t=0}^{\infty} \distas \Nc(0,\sigma_\vw^2\vI_n)$. There are $\numSys$ modes in total, and the dynamics of mode $i$ is given by the state matrix $\vA_i$ and input matrix $\vB_i$. The active mode at time $t$ is indexed by $\omega(t) \in [s]$.  Throughout, we assume the state $\vx_t$ and the mode $\omega(t)$ can be observed at time $t$. The mode switching sequence $\{\omega(t)\}_{t = 0}^{\infty}$ follows a Markov chain with transition matrix $\Tb \in \R_+^{\numSys \times \numSys}$ such that for all $t \geq 0$, the $ij$-th element of $\Tb$ denotes the conditional probability
$[\Tb]_{ij} := \P\big(\omega(t+1) = j \mid \omega(t) = i \big)$ for all $i,j \in [\numSys]$.
Throughout, we assume the initial state $\vx_0$, the mode switching sequence $\curlybrackets{\omega(t)}_{t=0}^\infty$, and the noise $\curlybrackets{\vw_t}_{t=0}^\infty$ are mutually independent. We use MJS$(\vA_{1:\numSys}, \vB_{1:\numSys}, \vT)$ to refer to an MJS with state equation~\eqref{switched LDS}, parameterized by the matrix tuple $(\vA_{1:\numSys}, \vB_{1:\numSys}, \vT)$. We call a sequence of controllers $\vK_{1:\numSys}:= \{ \vK_1, \dots, \vK_\numSys \}$ a mode-dependent state-feedback controller for the MJS if the input is given by $\vu_t = \vK_{\omega(t)} \vx_t$. Under $\vK_{1:\numSys}$, the MJS becomes closed-loop with the state matrices $\vL_{1:\numSys}$ where $\vL_i := \vA_i + \vB_i \vK_i$.

Due to the randomness in the mode sequence $\{\omega(t)\}_{t=0}^{\infty}$, it is common to consider the stability of MJS in the mean-square sense which is defined as follows.
\begin{definition}[Mean-square stability~\citep{costa2006discrete}]\label{def_mss}
	We say the MJS in \eqref{switched LDS} is mean-square stable (MSS) if, setting $\vu_t = 0$, there exists $\vx_\infty, \vSigma_\infty$ such that for any initial state $\vx_0$ and mode $\omega(0)$, as $t \rightarrow \infty$, we have
	\begin{align}
		\norm{\expctn[\vx_t] - \vx_\infty} \to 0, \quad	\norm{\expctn[\vx_t \vx_t^\top] - \vSigma_\infty} \to 0,
	\end{align}
	where the expectation is over the Markovian mode switching sequence $\{\omega(t)\}_{t=0}^{\infty}$, the noise $\{\vw_t\}_{t=0}^{\infty}$ and the initial state $\vx_0$. In the noise-free case (i.e., $\vw_t=0$), we have $\vx_\infty=0$, $\vSigma_\infty=0$. We say the MJS in \eqref{switched LDS} is (mean-square) stabilizable if there exists mode-dependent controller $\Kb_{1:s}$ such that the closed-loop MJS $\vx_{t+1} = (\vA_{\omega(t)} + \vB_{\omega(t)} \vK_{\omega(t)}) \vx_t$ is MSS. We call such $\Kb_{1:s}$ a stabilizing controller.
\end{definition}

Similar to the Lyapunov stability of LTI systems, MJSs also have the spectral radius criterion to determine the MSS. For notation brevity, let $\vL_{1:\numSys}$ denote the MJS state matrices, where $\vL_i = \vA_i + \vB_i \vK_i$ for the closed-loop case and $\vL_i = \vA_i$ otherwise. Define the augmented state matrix $\vLtil \in \R^{\numSys \dimSt^2 \times \numSys \dimSt^2}$ with the $ij$-th $\dimSt^2 {\times} \dimSt^2$ block given by $\squarebrackets{\vLtil}_{ij} := [\vT]_{ji} \vL_j \otimes \vL_j$. Then, $\rho(\vLtil)<1$ if and only if the MJS is MSS \citep[Theorem 3.9]{costa2006discrete}. This follows from the fact that the matrix $\vLtil$ maps $\expctn[\vx_t \vx_t^\top]$ to $\expctn[\vx_{t+1} \vx_{t+1}^\top]$ (see \eqref{eq_covEvolution} in Appendix~\ref{appendix_Preliminaries}). The notions of stabilizability and stabilizing controller follow similarly. 

In this work, we consider two major problems under the MJS setting; system identification and adaptive quadratic control, with identification being the core part of adaptive control.

\subsection{System Identification}
System identification seeks to estimate unknown system dynamics from a single~(or multiple) trajectory(ies) of the system's states, inputs and mode observations. In the MJS setting, our goal is to estimate the state/input matrices $\vA_{1:\numSys}, \vB_{1:\numSys}$ and the Markov transition matrix $\vT$ from a single trajectory of the system's states, inputs and mode observations $\curlybrackets{\vx_t, \vu_t, \omega(t)}_{t=0}^T$, and provide finite sample estimation guarantees. In this work, the main assumption for the MJS to be identified is as follows. \begin{assumption} \label{asmp_ergodicity_MarginMSS}
		The MJS has ergodic Markov chain and is stabilizable.
	\end{assumption}Ergodicity guarantees that the distribution of the mode sequence $\omega(t)$ converges to a unique strictly positive stationary distribution \citep[Theorem 4.3.5]{gallager2013stochastic}. Throughout, we let $\vpi_{\infty} \in \R_+^{\numSys}$ denote the stationary distribution of $\vT$ such that $\vpi_{\infty}^\top = \vpi_{\infty}^\top \vT$, and define $\pi_{\min} := \min_{i \in [s]} \vpi_\infty(i)$, $\pi_{\max} := \max_{i \in [s]} \vpi_\infty(i)$. 
	Ergodicity ensures that the MJS could have enough ``visits'' to every mode $i \in [s]$, thus providing enough number of samples to learn $[\vT]_{i,:}$, $\vA_i$ and $\vB_i$ for all $i \in [s]$.
	We further define the mixing time \citep{levin2017markov} that describes how fast a Markov chain converges to its stationary distribution.
\begin{definition}[Markov chain mixing]\label{def:mixing_time_main}
		Consider an ergodic Markov matrix $ \Tb \in \R_+^{\numSys \times \numSys}$ with stationary distribution $\vpi_{\infty} \in \R_+^{\numSys}$. For $\epsilon \geq 0$, define the mixing time as
		\begin{equation*}
			t_{\rm MC}(\epsilon) := \min \big\{t \in \mathbb{N} : \max_{i \in [s]} \frac{1}{2} \norm{([\vT^t]_{i,:})^\top - \vpi_{\infty} }_1 \leq \epsilon\big\}.
		\end{equation*}
		Particularly, when the parameter $\epsilon$ is omitted, $t_{\rm MC} := t_{\rm MC}(\frac{1}{4})$.
	\end{definition}

As mentioned earlier, MJS presents unique statistical analysis challenges due to Markovian jumps and MSS. 
In the following, Section \ref{sec_sysid} presents our system identification procedures together with theoretical guarantees overcoming these challenges, which are further integrated into model-based adaptive control for MJS-LQR in Section~\ref{sec:adaptcontrol}.

\subsection{Adaptive Quadratic Control}
\label{subse_ProblemFormulation_AdaptiveLQR}
In addition to system identification, in this work we also consider the following finite-horizon Markov jump system linear quadratic regulator (MJS-LQR) problem: given positive semi-definite cost matrices $\vQ_{1:\numSys}$ and $\vR_{1:\numSys}$,\begin{equation}\label{LQR optimization}
	\begin{aligned}
		\inf_{\ub_{0:T}} \quad &J(\vu_{0:T}) := \sum_{t=0}^{T} \expctn \big[ \vx_t^\top\Qb_{\omega(t)}\vx_t {+} \ub_t^\top\Rb_{\omega(t)}\ub_t \big],\\
		\text{s.t.} \quad & \vx_t, \omega(t) \sim \text{MJS}(\vA_{1:\numSys}, \vB_{1:\numSys}, \vT),
	\end{aligned}
\end{equation}
where the goal is to design inputs to minimize the quadratic cost constructed with mode-dependent cost matrices $\Qb_{\omega(t)}$ and $\Rb_{\omega(t)}$ under the MJS dynamics. The flexibility of having mode-dependent cost matrices allows one to design different control requirements or trade-offs under different circumstances. 
MJS-LQR problems have seen many real world applications, including networked control with random packet losses \citep{vargas2010linear} or delays \citep{chan1995optimal}, single-link robot arm with time-varying payloads and inertia \citep{palm1998fuzzy, wu2006mode, zhong2014optimal}, optimal control for a solar thermal receiver \citep{costa2006discrete}, and public expenditure policy-making \citep{costa2006discrete}. In the remaining of the paper, we use MJS-LQR$(\vA_{1:\numSys}, \vB_{1:\numSys}, \vT, \vQ_{1:\numSys}, \vR_{1:\numSys})$ to denote the MJS-LQR problem \eqref{LQR optimization} with MJS($\vA_{1:\numSys}, \vB_{1:\numSys}, \vT$) and cost matrices $\vQ_{1:\numSys}, \vR_{1:\numSys}$.

Throughout, it is assumed that in the MJS-LQR problem, the state $\vx_t$ and the mode $\omega(t)$ can be observed at time $t$. 
To guarantee its solvability in this case, we make the following assumptions.
\begin{assumption}
	\label{asmp_ergodicity_MSS}
	The MJS-LQR problem~\eqref{LQR optimization} satisfies
	\begin{enumerate}[label=(\alph*)]
		\item The MJS has ergodic Markov chain and is stabilizable.
		\item For all $i \in [s]$, $\vQ_{i} \succ 0, \vR_{i} \succ 0$.
	\end{enumerate}    
\end{assumption}
Under Assumption \ref{asmp_ergodicity_MSS}, knowing the MJS dynamics, the optimal solution to the MJS-LQR problem is given by a state-feedback mode-dependent controller, which can be solved via the coupled discrete-time algebraic Riccati equations \citep{costa2006discrete}[Corollary A.21, Theorem 4.6].

In this work, we assume the MJS dynamics are unknown, and only the design parameters $\vQ_{1:\numSys}$ and $\vR_{1:\numSys}$ are known. One typical control scheme in this scenario is known as adaptive control, which involves real time adaption of the controller according to the latest data generated by the system. To make sure the data sufficiently reflects the underlying dynamics, the control input may need to contain additional excitation signal. Such excitation yields high-quality data to design better controller but also incurs additional cost in the LQR objective function --- a manifestation of the exploration-exploitation trade-off.

To evaluate the performance of an adaptive control scheme, we look into the notion of regret --- how much more cost it would incur if one could have applied the optimal controllers? In our setting, we compare the resulting cost against the optimal cost $T J^\star$, where $J^*$ is the optimal infinite-horizon average cost\begin{equation} \label{eq_optinfhrzncost}
	J^\star:= \limsup_{T \rightarrow \infty} \frac{1}{T} \inf_{\ub_{0:T}} J(\vu_{0:T}),
\end{equation}
i.e., if one applies the optimal controller for infinitely long, the amount of cost one would pay on average for each individual time step.

In the following, based on the system identification procedures in Section \ref{sec_sysid}, we propose the adaptive control algorithm for MJS-LQR problems in Section \ref{sec:adaptcontrol}, which is followed by its regret analysis that guarantees $\Ocal(\sqrt{T})$ regret performance.

\section{System Identification for MJS}\label{sec_sysid}
\iffalse
\begin{algorithm}
	\begin{algorithmic}[1]
		\Require{MJS trajectory $\curlybrackets{\vx_t, \vz_t, \omega(t)}_{t=0}^T$, generated using inputs $\vu_t = \vK_{\omega(t)} \vx_t + \vz_t$; exploration niose $\{\vz_t\}_{t=0}^{T} \distas \N(0, \sigma_{\vz}^2 \vI_{\dimInput})$, mean-square stabilizing controller $\vK_{1:\numSys}$, variances $\sigma_\vw^2$ and $\sigma_{\vz}^2$}
		\Procedure{Estimate}{$\vA_{1:\numSys}, \vB_{1:\numSys}, \vT$}
		\For{$i \gets 1 \textrm{ to } \numSys$} 
		\State $S_i=\{t~\big|~\omega(t)=i\}$
		\State $\hat{\vTheta}_{1,i}, \hat{\vTheta}_{2,i}  = \underset{\vTheta_1 \in \R^{n \times n}, \vTheta_2 \in \R^{n \times p}}{\arg \min}$  
		 \Statex\hspace{30pt} $\frac{1}{2 |S_i|}\sum_{t \in S_i}
		\norm{\vx_{t + 1}- \vTheta_1 \vx_{t} / \sigma_{\vw} - \vTheta_2 \vz_{t} / \sigma_{\vz}}^2$ 
		\State $\vBhat_i = \hat{\vTheta}_{2,i}/\sigma_{\vz}, \quad \vAhat_i = \hat{\vTheta}_{1,i} / \sigma_\vw - \vBhat_i \vK_i$
		\For{$j \gets 1 \textrm{ to } \numSys$}
		\State $[\vThat]_{ij} = \frac{\sum_{t=1}^{T} \indicator{\omega(t-1)=i, \omega(t)=j}}{\sum_{t=1}^{T} \indicator{\omega(t-1)=i}}$
		\EndFor 
		\EndFor
		\State \Return{$\vAhat_{1:\numSys}, \vBhat_{1:\numSys}, \vThat$}
		\EndProcedure
	\end{algorithmic}\caption{MJS-SYSID} \label{Alg_MJS-SYSID}
\end{algorithm}
\fi
{\SetAlgoNoLine
\begin{algorithm}[t]
	\KwIn{MJS trajectory $\curlybrackets{\vx_t, \vz_t, \omega(t)}_{t=0}^T$, generated using inputs $\vu_t = \vK_{\omega(t)} \vx_t + \vz_t$; exploration niose $\{\vz_t\}_{t=0}^{T} \distas \N(0, \sigma_{\vz}^2 \vI_{\dimInput})$, mean-square stabilizing controller $\vK_{1:\numSys}$, variances $\sigma_\vw^2$ and $\sigma_{\vz}^2$.}
	\textbf{Estimate} $\vA_{1:\numSys}, \vB_{1:\numSys}$: \textbf{for all} modes $i \in [\numSys]$ {\textbf{do}}\\
\quad~$S_i=\{t~\big|~\omega(t)=i\}$\quad //sample data related to mode $i$. \\
\quad~$\hat{\vTheta}_{1,i}, \hat{\vTheta}_{2,i}  = \underset{\vTheta_1 \in \R^{n \times n}, \vTheta_2 \in \R^{n \times p}}{\arg \min} \frac{1}{2 |S_i|}\sum_{t \in S_i}.
		\norm{\vx_{t + 1}- \vTheta_1 \vx_{t} / \sigma_{\vw} - \vTheta_2 \vz_{t} / \sigma_{\vz}}^2$. \\
\quad~$\vBhat_i = \hat{\vTheta}_{2,i}/\sigma_{\vz}$, $\vAhat_i = \hat{\vTheta}_{1,i}/\sigma_\vw - \vBhat_i \vK_i$.\\
	{\textbf{Estimate $\vT$:}} $[\vThat]_{ji} = \frac{\sum_{t=1}^{T} \indicator{\omega(t)=i, \omega(t-1)=j}}{\sum_{t=1}^{T} \indicator{\omega(t-1)=j}}$\quad //empirical frequency of transitions.
	\\
    \KwOut{$\vAhat_{1:\numSys}, \vBhat_{1:\numSys}, \vThat$.}
	\caption{MJS-SYSID} \label{Alg_MJS-SYSID}
\end{algorithm}
}

Our MJS identification procedure is given in Algorithm \ref{Alg_MJS-SYSID}. We assume one has access to any stabilizing controller $\vK_{1:\numSys}$ to start the identification, which has been a standard assumption in data-driven control of LTI systems \citep{abeille2018improved,cohen2019learning,dean2018regret,ibrahimi2012efficient,simchowitz2020naive}. More discussions on this assumption for MJSs can be found in Section \ref{subsec_initStabilizingController}.
Note that, if the open-loop MJS is already MSS, then one can simply set $\vK_{1:\numSys}=0$ and carry out MJS identification. Given an MJS trajectory $\curlybrackets{\vx_t, \vz_t, \omega(t)}_{t=0}^T$, generated using the input $\vu_t = \vK_{\omega(t)} \vx_t + \vz_t$~(where $\{\vz_t\}_{t=0}^T \distas \N(0, \sigma_\vz^2 \vI_p)$ is the exploration noise), we solve $s$ least-squares regression problems to estimate $\vA_{1:s}, \vB_{1:s}$. Moreover, using the empirical frequency of observed modes, we  estimate $\vT$.

	The following theorem gives our main results on learning the dynamics of an unknown MJS from finite samples obtained from a single trajectory. One can refer to Theorems \ref{lemma_ConcentrationofMC} and  \ref{thrm:learn_dynamics_complete} in Appendix \ref{appendix_SysIDAnalysis} for the detailed theorem statements and proofs.
\begin{theorem}[Identification of MJS]~\label{thrm:learn_dynamics}
		Suppose we run Algorithm~\ref{Alg_MJS-SYSID} with the trajectory length $T \geq \max \big\{2 T_0, \hat{\Ocal}\big( \frac{(n+p)\log(T)}{\pi_{\min}(1 - \varrho)}\big)\big\}$, where $T_0 := t_{\rm MC}(\pi_{\min}/2)$ and $\varrho:= \hat{\Ocal}\big(\frac{1}{\pi_{\min}}\sqrt{\frac{\pi_{\max}T_0}{T}\big)}$. Suppose, $\{\zb_t\}_{t=0}^T \distas \Nc(0,\sigma_\vz^2\vI_p)$ and $\{\wb_t\}_{t=0}^T \distas \Nc(0,\sigma_\vw^2\vI_n)$. Then, under Assumption~\ref{asmp_ergodicity_MarginMSS},  with probability at least $1-\delta$, for all $i \in [s]$, we have
		\begin{equation}\label{sysid bound}
			\begin{split} 
				\max\biggl\{\begin{matrix}
				\frac{\sigma_{\vz}}{\sigma_{\vw}+\sigma_{\vz}}\norm{\vAhat_{i} - \vA_{i}}, \frac{\sigma_{\vz}}{\sigma_{\vw}}\norm{\vBhat_{i} - \vB_{i}}\end{matrix}\biggr\} &\leq \hat{\Ocal}\bigg(\mysqrt[1pt]{\frac{(n+p)\log(T)}{\pi_{\min}(1 - \varrho)T}}\bigg), \\
				\text{and} \quad \norm{\vThat - \vT} &\leq \hat{\Ocal}  \bigg(\frac{1}{\pi_{\min}} \mysqrt[1pt]{\frac{\log(T)}{T}} \bigg).
			\end{split}
		\end{equation}
	\end{theorem}
\iffalse
\begin{comment}
	\subsection{Special Case --- $\vB$ is Known}
	When the input matrices $\vB_{1:\numSys}$ are known, no exploration noise is needed. The sample complexity for Algorithm \ref{Alg_MJS-SYSID} is as follows.
	\begin{corollary}[\red{Placeholder}] \label{corollary_pureSysIDResult_Bknown_shortVersion} When $\vB_{1:\numSys}$ is known, if $T$, $C_{sub}$, and $c_\vx$ are large enough, then with probability no less than $1-\delta$,
		\begin{align}
			\norm{\vThat - \vT}_\infty &\leq \hat{\Ocal} (\sqrt{\frac{\log(T)}{{T}}}) \\
			\norm{\vAhat_{1:\numSys} - \vAstar_{1:\numSys}} &\leq \hat{\Ocal} \left(\sigma_\vw \sqrt{\frac{n\log(T)}{{T}}} \right).
		\end{align}
\end{corollary}
	This result later proves useful and shows that with the knowledge of $\vB_{1:\numSys}$, one can improve the regret from $O(\log(T) \sqrt{T})$ to $O(\log(T))$.
\end{comment}
\fi

\begin{proofsk}[Theorem~\ref{thrm:learn_dynamics}]
	Let $\hb_t := [\xb_t^\top/\sigma_\vw~\zb_t^\top/\sigma_\vz]^\top$ and $\bTetas_i:= [\sigma_\vw(\Abs_i + \Bbs_i\Kb_i)~\sigma_\vz\Bbs_i]$ for all $i \in [\numSys]$. Then the output of each sample in $\{(\xb_{t+1}, \xb_{t}, \zb_{t}, \omega(t))\}_{t \in S_i}$ can be related to the inputs as follows,
		\begin{align}
			\xb_{t_k+1} = \bTetas_{i}\hb_{t_k} + \wb_{t_k} \quad \text{for} \quad k = 1,2,\dots, |S_i|, \label{eqn:sys update compact_psk} 
		\end{align}
where we set $S_i := \{t \bgl \omega(t) = i\} \equiv \{t_1, t_2, \cdots, t_{|S_i|}\}$.  This shows that, for each $i \in [s]$, the problem of estimating $(\vA_i, \vB_i)$ is equivalent to the problem of estimating $\bTetas_i$ from the sequence of covariate-response pairs $(\vh_{t_k}, \vx_{t_k + 1})_{k \geq 1}$. Specifically, following Algorithm~\ref{Alg_MJS-SYSID}, we solve a regression problem. For this purpose, we define the following concatenated matrices: $\Yb_i$ has $\{\xb_{t+1}^\top\}_{t \in S_i}$ on its rows, $\Hb_i$ has $\{\hb_t^\top\}_{t \in S_i}$ on its rows and $\Wb_i$ has $\{\vw_t^\top\}_{t \in S_i}$ on its rows. Observe that, we have $\Yb_i = \Hb_{i} \bTeta_i^{\star \top} + \Wb_{i}$ and the regression problem in Algorithm~\ref{Alg_MJS-SYSID} becomes,\begin{align}
			\hat{\bTeta}_i = \underset{\bTeta_i \in \R^{n \times (n+p)}}{\arg\min}\frac{1}{2|S_i|}\tf{\Yb_i - \Hb_i\bTeta_i^\top}^2.\label{eqn:least_squares_psk}
		\end{align}
		When the problem is over-determined, the solution to the least-squares problem~\eqref{eqn:least_squares_psk} is given by $\hat{\bTeta}_i^\top = \Hb_{i}^\dagger\Yb_i = (\Hb_{i}^\top\Hb_{i})^{-1}\Hb_{i}^\top \Yb_i$ and the associated estimation error is given by, $\hat{\bTeta}_i - \bTetas_i = \big((\Hb_{i}^\top\Hb_{i})^{-1}\Hb_{i}^\top \Wb_i\big)^\top$. This implies that the estimation error can be upper-bounded as follows,\begin{align}
			\norm{\hat{\bTeta}_i - \bTetas_i}  = \norm{(\Hb_i^\top\Hb_i)^{-1}\Hb_i^\top \Wb_i} \leq \frac{\norm{\Hb_i^\top \Wb_i}}{\lambda_{\min}\big(\Hb_i^\top\Hb_i\big)},\label{eqn:est_err_psk}
		\end{align}
		We upper bound the estimation error in \eqref{eqn:est_err_psk} as follows: (a) First, we prove that the covariates process $\{\vh_{t_k}\}_{k = 1}^{|S_i|}$ satisfies $(k,\vI_{n+p},q)$-block Martingale small-ball condition~(Definition~\ref{def:BMSB}), with the constants $k = 1$ and $q=3/10$. (b) Next, we use Assumption~\ref{asmp_ergodicity_MarginMSS} and Markov inequality to show that $\P\big(\Hb_i^\top\Hb_i \npreceq (|S_i|\bar{\Gamma}/\delta)\vI_{n+p}\big) \leq \delta$, for some $\bar{\Gamma} = \Ocal(T)$. (c) Next, we use Assumption~\ref{asmp_ergodicity_MarginMSS}~(ergodicity) and Freedman's inequality to show that, using $T \geq 2T_0$, where $T_0 := t_{\rm MC}(\pi_{\min}/2)$, we have $\P \big(\bigcap_{i=1}^{s} \{|S_i| \geq \pi_{\min}(1 - \varrho)T \}\big)  \geq 1 - \delta$, where $\varrho:= \hat{\Ocal}\big(\frac{1}{\pi_{\min}}\sqrt{\frac{\pi_{\max}T_0}{T}\big)}$. Finally, we combine (a), (b) and (c) with Theorem 2.4 from \citet{simchowitz2018learning} to obtain our main result on single trajectory learning of $\vA_{1:s}, \vB_{1:s}$.
\end{proofsk}

Our system identification result achieves near-optimal ($\hat{\Ocal}(\sqrt{(n+p)/T})$) dependence on the trajectory length $T$. Note that the overall sample complexity grows as $T\gtrsim (n+p)/\pi_{\min}$. A degrees-of-freedom counting argument would show that the dependency of $T\gtrsim (n+p)/\pi_{\min}$ is optimal. The reason is that, each vector state equation we fit has $n$ scalar equations. The total degrees of freedom for each dynamics pair $(\Ab_i,\Bb_i)$ is $n\times (n+p)$. Additionally, for the least-frequent mode, in steady-state, we should observe $\pi_{\min}T$ equations. Putting these together, we would minimally need $n \times \pi_{\min}T\geq n\times (n+p)$, which means we need $T \geq (n+p)/\pi_{\min}$ samples to estimate the MJS dynamics $(\Ab_{1:s}, \Bb_{1:s})$. Note that, our sample complexity is not effected directly by the number of MJS modes $s$. However, $s$ indirectly effects sample complexity via $\pi_{\min}$, which is the probability of least-frequent mode in the steady state.

It is well known that the least squares problem has a unique solution when the regressor matrix has full rank.
For the least squares problem in Algorithm \ref{Alg_MJS-SYSID}, the unknown input matrix $\vB_i$ has regressor given by the exploration noise $\{\vz_t\}_{t \in S_i}$, which can be guaranteed to be full-rank when $\vz_t$ has non-degenerate covariance. This ensures that one can uniquely recover $\vB_i$, thus is the reason we apply the additional $\vz_t$ to the input $\vu_t$. On the other hand, the regressor $\{\vx_t\}_{t \in S_i}$ associated with the state matrix $\vA_i$ is guaranteed to be full-rank due to the presence of the additive process noise $\vw_t$ in the MJS dynamics \eqref{switched LDS}.
This implies that, when $\vB_{1:s}$ are known a priori, the exploration noise $\vz_t$ is no longer needed, and one is still able to learn the state matrices $\vA_{1:s}$. The sample complexity guarantee for this case is provided in Corollary \ref{cor:learn_dynamics_KnownB} below. 
One advantage of not exerting $\vz_t$ when $\vB_{1:s}$ are known, as we shall see in Section \ref{sec:adaptcontrol}, is that the adaptive quadratic control regret can be significantly improved.

\begin{corollary}[Identification with known $\vB_{1:s}$]\label{cor:learn_dynamics_KnownB}
		Consider the same setting as Theorem~\ref{thrm:learn_dynamics}. Additionally, suppose $\vB_{1:\numSys}$ are known. Then, setting $\sigma_{\vz}=0$ and solving only for the state matrices, with probability at least $1-\delta$, for all $i \in [s]$, we have $\norm{\vAhat_{i}-\vA_{i}}\leq \hat{\Ocal}\big(\sqrt{\frac{n\log(T)}{\pi_{\min}(1 - \varrho)T}}\big)$.
	\end{corollary}

\section{Adaptive Control for MJS-LQR}
\label{sec:adaptcontrol}
To solve the MJS-LQR problem with unknown MJS dynamics in Section \ref{subse_ProblemFormulation_AdaptiveLQR}, we propose an adaptive control scheme in Algorithm \ref{Alg_adaptiveMJSLQR}. It is performed on an epoch-by-epoch basis; a fixed controller is used for each epoch, and from epoch to epoch, the controller is updated using the trajectory generated in the most recent epoch. Note that a new epoch is just a continuation of previous epochs instead of restarting the MJS.
Similar to the discussion in Section \ref{sec_sysid}, we assume, at the beginning of epoch $0$, that one has access to a stabilizing controller $\vK_{1:\numSys}^{(0)}$. During epoch $\epkidx$, the controller $\vK_{1:\numSys}^{(\epkidx)}$ is used together with additive exploration noise $\vz_t^{(\epkidx)} \distas \N(0, \sigma_{\vz,\epkidx}^2 \vI_{\dimInput})$ to boost learning. At the end of epoch $\epkidx$, the trajectory during that epoch is used to obtain a new MJS dynamics estimate $\vA_{1:\numSys}^{(\epkidx)}, \vB_{1:\numSys}^{(\epkidx)}, \vT^{(\epkidx)}$ using Algorithm \ref{Alg_MJS-SYSID}. Then, we set the controller $\vK_{1:\numSys}^{(\epkidx+1)}$ for epoch $\epkidx+1$ to be the optimal controller for the infinite-horizon MJS-LQR($\vA^{(\epkidx)}_{1:\numSys},  \vB^{(\epkidx)}_{1:\numSys}, \vT^{(\epkidx)},  \vQ_{1:\numSys}, \vR_{1:\numSys}$), which can be computed as follows:
For a generic infinite-horizon MJS-LQR($\vA_{1:\numSys}, \vB_{1:\numSys}, \vT, \vQ_{1:\numSys}, \vR_{1:\numSys}$), its optimal controller is given by $\vK_{1:s}$ such that for all $j \in [\numSys]$,\begin{equation}\label{eq_optfeedback}
	\vK_j := -\big(\vR_j + \vB_j^\top \varphi_j(\vP_{1:\numSys}) \vB_j\big)^{-1} \vB_j^\top \varphi_j(\vP_{1:\numSys}) \vA_j,
\end{equation}
where $\varphi_j(\vP_{1:\numSys}) := \sum \nolimits_{k=1}^\numSys [\vT]_{jk} \vP_k$, and $\vP_{1:\numSys}$ is the solution to the following coupled discrete-time algebraic Riccati equations (cDARE):\begin{equation}\label{eq_CARE}
	\begin{aligned}
	\vP_j &= \vA_j^\top \varphi_j(\vP_{1:\numSys}) \vA_j + \Qb_j 	- \vA_j^\top \varphi_j(\vP_{1:\numSys}) \vB_j 
\big(\Rb_j + \vB_j^\top \varphi_j(\vP_{1:\numSys}) \vB_j\big)^{-1} \vB_j^\top \varphi_j(\vP_{1:\numSys}) \vA_j,
		\end{aligned}
\end{equation}
for all $j \in [\numSys]$. In practice, cDARE can be solved efficiently via value iteration or LMIs \citep{costa2006discrete}. Note that cDARE may not be solvable for arbitrary parameters, but our theory guarantees that when epoch lengths are appropriately chosen, cDARE parameterized by $\vA^{(\epkidx)}_{1:\numSys},  \vB^{(\epkidx)}_{1:\numSys}, \vT^{(\epkidx)}, \vQ_{1:\numSys}, \vR_{1:\numSys}$ is solvable for every epoch $\epkidx$. This control design based on the estimated dynamics is also referred to as certainty equivalent control.

To achieve theoretically guaranteed performance, i.e., sub-linear regret, the key is to have a subtle scheduling of epoch lengths $T_\epkidx$ and exploration noise variance $\sigma_{\vz,\epkidx}^2$. We choose $T_\epkidx$ to increase exponentially with rate $\gamma>1$, and set $\sigma_{\vz,\epkidx}^2 = \sigma_\vw^2 / \sqrt{T_\epkidx}$, which collectively guarantee $\hat{\Ocal}(\sqrt{T})$ regret when combined with the system identification result from Theorem \ref{thrm:learn_dynamics}. Intuitively, this scheduling can be interpreted as follows: (i) the increase of epoch lengths guarantees that we have more accurate MJS estimates thus more optimal controllers; (ii) as the controller becomes more optimal we can gradually decrease the exploration noise and deploy (exploit) the controller for a longer time. Note that the scheduling rate $\gamma$ has a similar role to the discount factor in reinforcement learning: smaller $\gamma$ aims to reduce short-term cost while larger $\gamma$ aims to reduce long-term cost.
\iffalse
\begin{algorithm}
	\begin{algorithmic}[1]
		\Require{Initial epoch length $T_0$; initial stabilizing controller $\vK_{1:\numSys}^{(0)}$; epoch incremental ratio $\gamma>1$; and noise variance $\sigma_{\vw}^2$}
		
\For{$q = 0,1,2, \dots$} 
		\State Set epoch length $T_\epkidx = \lfloor T_0 \gamma^\epkidx \rfloor$
		\State Set exploration noise variance $\sigma_{\vz, \epkidx}^2 = \frac{\sigma_{\vw}^2}{\sqrt{T_\epkidx}}$ 
		\State Run the MJS for $T_\epkidx$ time-steps with  
		\Statex \quad\; $\vu_t^{(\epkidx)} = \vK_{\omega^{(\epkidx)}(t)}^{(\epkidx)} \vx_t^{(\epkidx)} + \vz_t^{(\epkidx)}$,  $\vz_t^{(\epkidx)} \distas \N(0, \sigma_{\vz,\epkidx}^2 \vI_{\dimInput})$
		\State Record the trajectory $\curlybrackets{\vx_t^{(\epkidx)}, \vz_t^{(\epkidx)}, \omega^{(\epkidx)}(t) }_{t=0}^{T_\epkidx}$
		\State $\vA_{1:\numSys}^{(\epkidx)}, \vB_{1:\numSys}^{(\epkidx)}, \vT^{(\epkidx)} = \text{MJS-SYSID}(\curlybrackets{\vx_t^{(\epkidx)}, \vz_t^{(\epkidx)},\omega^{(\epkidx)}(t) }_{t=0}^{T_\epkidx},\vK_{1:\numSys}^{(\epkidx)}, \sigma_\vw^2, \sigma_{\vz,\epkidx}^2)$
		\State Set the controller $\vK^{(\epkidx+1)}_{1:\numSys}$ for the next epoch to be 
		\Statex  \quad\; the optimal controller for the infinite-horizon 
		\Statex  \quad\; MJS-LQR($\vA^{(\epkidx)}_{1:\numSys},  \vB^{(\epkidx)}_{1:\numSys}, \vT^{(\epkidx)}, \vQ_{1:\numSys}, \vR_{1:\numSys}$)
		\EndFor
\end{algorithmic}\caption{Adaptive MJS-LQR} \label{Alg_adaptiveMJSLQR}
	\end{algorithm}
\fi
{\SetAlgoNoLine \begin{algorithm}[t]
	\KwIn{Initial epoch length $T_0$; initial stabilizing controller $\vK_{1:\numSys}^{(0)}$; epoch incremental ratio $\gamma>1$; and noise variance $\sigma_{\vw}^2$.}
	\For{$\epkidx = 0, 1, 2, \dots$}{
	Set epoch length $T_\epkidx = \lfloor T_0 \gamma^\epkidx \rfloor$. \\
	Set exploration noise variance $\sigma_{\vz, \epkidx}^2 = \frac{\sigma_{\vw}^2}{\sqrt{T_\epkidx}}$.\\
	Run the MJS for $T_\epkidx$ time-steps with $\vu_t^{(\epkidx)} = \vK_{\omega^{(\epkidx)}(t)}^{(\epkidx)} \vx_t^{(\epkidx)} + \vz_t^{(\epkidx)}$, where $\vz_t^{(\epkidx)} \distas \N(0, \sigma_{\vz,\epkidx}^2 \vI_{\dimInput})$ and record the trajectory $\curlybrackets{\vx_t^{(\epkidx)}, \vz_t^{(\epkidx)}, \omega^{(\epkidx)}(t) }_{t=0}^{T_\epkidx}$.\\
	$\vA_{1:\numSys}^{(\epkidx)}, \vB_{1:\numSys}^{(\epkidx)}, \vT^{(\epkidx)} = \text{MJS-SYSID}(\curlybrackets{\vx_t^{(\epkidx)}, \vz_t^{(\epkidx)},\omega^{(\epkidx)}(t) }_{t=0}^{T_\epkidx},\vK_{1:\numSys}^{(\epkidx)}, \sigma_\vw^2, \sigma_{\vz,\epkidx}^2)$.
    \\    
	Set the controller $\vK^{(\epkidx+1)}_{1:\numSys}$ for the next epoch to be the optimal controller for the infinite-horizon MJS-LQR($\vA^{(\epkidx)}_{1:\numSys},  \vB^{(\epkidx)}_{1:\numSys}, \vT^{(\epkidx)}, \vQ_{1:\numSys}, \vR_{1:\numSys}$). \label{algline_1}
	}
	\caption{Adaptive MJS-LQR} \label{Alg_adaptiveMJSLQR}
\end{algorithm}
}
\subsection{Regret Analysis}
We define filtration $\Fcal_{-1}, \Fcal_0, \Fcal_1, \dots$ such that $\Fcal_{-1} := \sigma(\vx_0, \omega(0))$ is the sigma-algebra generated by the initial state and initial mode, and $\splitatcommas{ \Fcal_\epkidx := \sigma(\vx_0, \omega(0), \curlybrackets{\curlybrackets{\omega^{(j)}(t)}_{t=1}^{T_j}}_{j=0}^\epkidx,   \curlybrackets{\curlybrackets{\vw_t^{(j)}}_{t=1}^{T_j}}_{j=0}^\epkidx,  \curlybrackets{\curlybrackets{\vz_t^{(j)}}_{t=1}^{T_j}}_{j=0}^\epkidx, \vw_0,\vz_0,) }$ is the sigma-algebra generated by the randomness up to epoch $\epkidx$. Note that the initial state $\vx_0^{(\epkidx)}$ of epoch $\epkidx$ is also the final state $\vx_{T_{\epkidx-1}}^{(\epkidx-1)}$ of epoch $\epkidx-1$, therefore, $\vx_0^{(\epkidx)}$ is $\Fcal_{\epkidx-1}$-measurable, and so is $\omega^{(\epkidx)}(0)$. Suppose time step $t$ belongs to epoch $\epkidx$, then we define the conditional expected cost at time $t$ as,\begin{equation}\label{eq_conditionalCost}
	c_t = \expctn [\vx_t^\top \vQ_{\omega(t)} \vx_t + \vu_t^\top \vR_{\omega(t)} \vu_t \mid \Fcal_{\epkidx-1}].
\end{equation}
The cost for epoch $\epkidx$ is defined as $J_{(\epkidx)} := \sum_{t \in  \text{epoch}-  \epkidx} c_t$, and the cumulative cost is defined as $J_T := \sum_\epkidx J_{(\epkidx)}$.
We define the total regret and epoch-$\epkidx$ regret as, \begin{equation}\label{eq_regretSingleEpoch}
	\textnormal{Regret}(T) := J_T - T J^\star, \quad \textnormal{Regret}^{(\epkidx)} := J_{(\epkidx)} - T_\epkidx J^\star.
\end{equation}
Then,  $\textnormal{Regret}(T) = \mathcal{O}(\sum_{\epkidx=1}^{\mathcal{O}(\log_{\gamma} (T/T_0))} \textnormal{Regret}^{(\epkidx)})$, where regret of epoch $0$ is ignored as it does not scale with time $T$. 
In the definition of the regret, we evaluate the expected cost conditioning on the randomness up to the previous epoch. This is the middle ground between the expected cost $\expctn[\sum_t \vx_t^\top \vQ_{\omega(t)} \vx_t + \vu_t^\top \vR_{\omega(t)} \vu_t]$ \citep{cassel2020logarithmic} and random cost $\sum_t \vx_t^\top \vQ_{\omega(t)} \vx_t + \vu_t^\top \vR_{\omega(t)} \vu_t$ \citep{lale2020logarithmic} typically considered in previous online learning works. In the next subsection, we show that under certain stronger stability of the MJS, regret based on the random cost can also be bounded. Let $\vK^{\star}_{1:\numSys}$ denotes the optimal controller for the infinite-horizon MJS-LQR($\vA_{1:\numSys}, \vB_{1:\numSys}, \vT, \vQ_{1:\numSys}, \vR_{1:\numSys}$) problem. $\vLtil^{(0)}$ and $\vLtil^\star$ denote the closed-loop augmented state matrices under the initial controller $\vK^{(0)}_{1:\numSys}$ and the optimal controller $\vK^{\star}_{1:\numSys}$ respectively, and we let $\bar{\rho}: = \max\curlybrackets{\rho(\vLtil^{(0)}), \rho(\vLtil^\star)}$, $\rho^\star := \rho(\vLtil^\star)$. With these definitions, we have the following sub-linear regret guarantee. Please refer to Theorem \ref{thrm_mainthrm_complete} in Appendix~\ref{appendix_MJSRegretAnalysis} for the complete version and proof.
\begin{theorem}[Sub-linear regret]\label{thrm_mainthrm}
	Assume that the initial state $\vx_0 = 0$, and Assumption \ref{asmp_ergodicity_MSS} holds. Suppose $T_0 \geq
		\hat{\Ocal}\big(\frac{t_{\rm MC}\log(T_0)}{\pi_{\min}(1-\varrho \lor \rho^\star)}(\dimSt+\dimInput)\big)$. Then, with probability at least $1-\delta$, Algorithm \ref{Alg_adaptiveMJSLQR} achieves
		\begin{equation}\label{eq_mainRegretUpperBd}
			\begin{aligned}
			\textnormal{Regret}(T) &\leq \hat{\Ocal} \parenthesesbig{\frac{\numSys \dimInput (\dimSt + \dimInput) \sigma_\vw^2 }{\pi_{\min} (1-\varrho \lor \rho^\star)} \log(T) \sqrt{T} } 
+ \Ocal\parenthesesbig{\frac{\sqrt{\dimSt \numSys}\log^3(T)}{\delta}}.
			\end{aligned}
	\end{equation}
\end{theorem}

\begin{proofsk}[Theorem~
	\ref{thrm_mainthrm}]
	For simplicity, we only show the dominant $\hat{\Ocal}(\cdot)$ term and leave the complete proof to appendix. We begin by defining the estimation error after epoch $\epkidx$ as follows:
	$
	\epsilon_{\vA, \vB}^{(\epkidx)} {:=} \max_{j \in [\numSys]} \max \curlybrackets{\norm{\vA^{(\epkidx)}_{j} {-} \vA_{j}}, \norm{\vB^{(\epkidx)}_{j}  {-} \vB_{j}}}$ and 
	$
	\epsilon_{\vT}^{(\epkidx)} := \norm{\vT^{(\epkidx)}  - \vT}_{\infty}
	$.
	Analyzing the finite-horizon cost and combining the infinite-horizon perturbation results in \citet{du2021certainty}, we can bound epoch-$\epkidx$ regret as
	$
	\textnormal{Regret}^{(\epkidx)} \leq \mathcal{O}\left( T_\epkidx \sigma_{\vz,\epkidx}^2 +  T_\epkidx \sigma_\vw^2 \parenthesesbig{ \epsilon_{\vA,\vB}^{(\epkidx-1)}  +  \epsilon_{\vT}^{(\epkidx-1)}}^2 \right)$. Plugging in the exploration noise variance
		$\sigma_{\vz,\epkidx}^2 = \frac{\sigma_\vw^2}{\sqrt{T_\epkidx}}$, the upper bounds on the estimation errors $\epsilon_{\vA, \vB}^{(\epkidx)} \leq \hat{\Ocal} \parenthesesbig{\frac{\sigma_{\vz,\epkidx}+\sigma_\vw}{\sigma_{\vz,\epkidx} } \mysqrt[1pt]{\frac{(n+p)\log(T_q)}{\pi_{\min}(1 - \varrho)T_q}}}$ and $\epsilon_{\vT}^{(\epkidx)} \leq \hat{\Ocal} \parenthesesbig{\frac{\sqrt{\log(T_\epkidx)}}{\sqrt{T_\epkidx}}}$ from Theorem \ref{thrm:learn_dynamics}, we have $\textnormal{Regret}^{(\epkidx)} \leq \hat{\Ocal} \parenthesesbig{\frac{\numSys \dimInput (\dimSt + \dimInput) \sigma_\vw^2}{\pi_{\min} (1-\varrho \lor \rho^\star)} \gamma \sqrt{T_\epkidx} \log(T_\epkidx) }$. Finally, since $T_\epkidx = \Ocal(T_0 \gamma^\epkidx)$ from Algorithm~\ref{Alg_adaptiveMJSLQR}, we have
		\begin{align}
			\textnormal{Regret}(T) &= \sum_{\epkidx=1}^{\mathcal{O}(\log_\gamma (\frac{T}{T_0}))} \textnormal{Regret}^{(\epkidx)}, \nn \\
			&\leq \hat{\Ocal} \bigg(\frac{\numSys \dimInput (\dimSt + \dimInput) \sigma_\vw^2}{\pi_{\min} (1-\varrho \lor \rho^\star)} \sqrt{T}  \bigg(\log(T_0) \frac{\gamma\sqrt{\gamma}}{\sqrt{\gamma}- 1}
			+ \log(T/T_0)\frac{\gamma^{2}}{(\sqrt{\gamma}- 1)^2}\bigg)\bigg),\nn \\
			& \leq \hat{\Ocal} \parenthesesbig{\frac{\numSys \dimInput (\dimSt + \dimInput) \sigma_\vw^2}{\pi_{\min} (1-\varrho \lor \rho^\star)} \log(T) \sqrt{T}}.
	\end{align}
\end{proofsk}

Note that the state dimension $\dimSt$, the input dimension $\dimInput$, number of modes $\numSys$, Markov chain mixing time $t_{MC}$ affect the regret bound in Theorem \ref{thrm_mainthrm} and the identification error bound in Theorem \ref{thrm:learn_dynamics} in a similar way. The factors that exclusively affect the regret bound are the spectral radius $\rho^\star = \rho(\vLtil^\star)$ and the epoch incremental ratio $\gamma$. 
In our regret upper bound, there is a heavy-tailed probability term $1/\delta$. In the next subsection, we discuss how this term is unavoidable under MSS, but can be improved to sub-exponential tail term $\log(1/\delta)$ when stronger stability exists.

\subsection{Two Special Cases}\label{sec:logreg}
\subsubsection{Regret under uniform stability}\label{subsubsec_UniformStab}
Note that the second term in the regret upper bound \eqref{eq_mainRegretUpperBd} in Theorem \ref{thrm_mainthrm} depends on the failure probability $\delta$ through $1/\delta$. Though this term has a much milder dependency on the time horizon $T$, when setting $\delta$ to be small, it can still easily outweigh the $\hat{\Ocal}(\cdot)$ term in \eqref{eq_mainRegretUpperBd}, which only has sub-exponential tail $\log(1/\delta)$ dependency, and can result in overly pessimistic regret bounds. The main cause of this $1/\delta$ term is that, in the regret analysis, one needs to factor in the cumulative impact of initial state of every epoch, i.e. $\sum_\epkidx \norm{\vx_0^{(\epkidx)}}^2 $. Since MSS guarantees the stability and state convergence only in the mean-square sense, we can, at best, only bound $\expctn[\norm{\vx_0^{(\epkidx)}}^2]$ and then use the Markov inequality: with probability at least $1-\delta$, $\norm{\vx_0^{(\epkidx)}}^2 \leq \expctn[\norm{\vx_0^{(\epkidx)}}^2]/\delta$.
Furthermore, in Appendix \ref{appendix_regretUniformStability}, we construct an MJS example that is MSS, but dependency no better than $1/\delta$ is possible.
Fortunately, there exists an easy workaround to get rid of this $1/\delta$ dependency if the MJS is uniformly stable \citep{liberzon2003switching, LEE2006205}, which enforces stability under arbitrary switching sequences, thus is stronger than MSS. It allows us to bound $\vx_0^{(\epkidx)}$ using tail inequalities much tighter than the Markov inequality and obtain $\norm{\vx_0^{(\epkidx)}}^2 \leq \Ocal(\log(1/\delta))$ with probability at least $1- \delta$. As a result, the $1/\delta$ dependency in the regret bound can be improved to $\log(1/\delta)$.

One type of uniform stability assumption that can help us in this case is related to the closed-loop MJS under the optimal controllers. We let $\vK_{1:\numSys}^\star$ denote the optimal controller for the infinite-horizon MJS-LQR($\vA_{1:\numSys},  \vB_{1:\numSys}, \vT, \vQ_{1:\numSys}, \vR_{1:\numSys}$) and define closed-loop state matrices $\vL_i^\star = \vA_i + \vB_i \vK_i^\star$ for all $i \in [s]$. 
We let $\theta^\star$ denote the joint spectral radius of $\vL_{1:\numSys}^\star$, i.e. $\theta^\star := \lim_{l \rightarrow \infty} \max_{\omega_{1:l} \in [\numSys]^{l}}\norm{\vL_{\omega_1}^\star \cdots \vL_{\omega_l}^\star}^{\frac{1}{l}}$, and we say $\vL_{1:\numSys}^\star$ is uniformly stable if and only if $\theta^\star < 1$.
Let $\bar{\theta}:= \frac{1+\theta^\star}{2}$.
The resulting regret bound is outlined in the following theorem, with its complete version and proof provided in Theorem \ref{thrm_mainthrm_completeUniformStability} of Appendix \ref{appendix_regretUniformStability}.
\begin{theorem}[Regret under uniform stability]\label{thrm_regretUniformStability}
	Assume that the initial state $\vx_0 = 0$, and Assumption~\ref{asmp_ergodicity_MSS} holds,
	and $\vL_{1:\numSys}^\star$ is \emph{uniformly stable}. If hyper-parameters $T_0$,  $c_\vx$, and $c_\vz$ are chosen as sufficiently large, with probability at least $1-\delta$, Algorithm \ref{Alg_adaptiveMJSLQR} achieves
\begin{equation}
		\textnormal{Regret}(T) \leq \hat{\Ocal} \parenthesesbig{\frac{\numSys \dimInput (\dimSt + \dimInput) \sigma_\vw^2 }{\pi_{\min} (1-\varrho \lor \rho^\star)} \log(T) \sqrt{T} }.
	\end{equation}
\end{theorem}

Another benefit of assuming uniform stability is that we can establish a sub-linear bound for the regret defined using the random cost. Denote the random cost at time $t$ as
	$c_t^\circ$, the random cost for epoch $\epkidx$ as 
	$J^\circ_{(\epkidx)}$, and random regret as $\textnormal{Regret}^\circ(T)$, defined as follows:
	\begin{align*}
		c_t^\circ &:= \vx_t^\top \vQ_{\omega(t)} \vx_t + \vu_t^\top \vR_{\omega(t)} \vu_t, \quad 
		J^\circ_{(\epkidx)} := \sum_{t~\in~\text{epoch}~\epkidx} c^\circ_t, \quad \textnormal{Regret}^\circ(T) := \sum_q J^\circ_{(\epkidx)} - T J^\star.
	\end{align*}
	Since we already have an upper bound for the $\textnormal{Regret}(T) = \sum_q J_{(\epkidx)} - T J^\star$ in Theorem \ref{thrm_regretUniformStability}, it suffices to upper bound $\sum_q J^\circ_{(\epkidx)} - J^{\phantom{o}}_{(\epkidx)}$ to establish an upper bound for the $\textnormal{Regret}^\circ(T)$. 
In each summand $J^\circ_{(\epkidx)} - J^{\phantom{o}}_{(\epkidx)}$, we see $J_{(\epkidx)} = \expctn[J^\circ_{(\epkidx)} \mid \Fcal_{\epkidx-1}]$ where $\Fcal_{\epkidx-1}$ affects the expectation only through the initial state $\vx_0^{(\epkidx)}$, initial mode $\omega^{(\epkidx)}(0)$, and the controller $\vK^{(\epkidx)}_{1:\numSys}$. Thus, the summand $J^\circ_{(\epkidx)} - J^{\phantom{o}}_{(\epkidx)}$ measures the deviation of the epoch's random cost $J^\circ_{(\epkidx)}$ from its conditional expectation with given initial conditions and controllers. Under the uniform stability assumption, we can show that $J^\circ_{(\epkidx)}$ is sub-exponential, which allows us to obtain $J^\circ_{(\epkidx)} - J^{\phantom{o}}_{(\epkidx)} \leq \Ocal(\sqrt{T_{\epkidx}} \log(1/\delta))$ and $\textnormal{Regret}^\circ(T) \leq \Ocal(\sqrt{T} \log(1/\delta))$. On the other hand, in the case of MSS, for similar reasons we discussed above, $J^\circ_{(\epkidx)}$ can be heavy-tailed, 
	and the dependency on $\delta$ can at best be $1/\delta$. The formal result is provided below and the proof is provided in Appendix~\ref{proof_thrm_randomRegretUniformStability}
\begin{theorem}[Random regret]\label{thrm_randomRegretUniformStability}
		Under the same setup of Theorem \ref{thrm_regretUniformStability}, with probability at least $1-\delta$, Algorithm \ref{Alg_adaptiveMJSLQR} achieves: $\textnormal{Regret}^\circ(T) \leq $
\begin{equation}\label{eq_regretUniformStability}
		\hat{\Ocal} \parenthesesbig{\big(\frac{\numSys \dimInput (\dimSt + \dimInput) \sigma_\vw^2 }{\pi_{\min} (1-\varrho \lor \rho^\star)} \log(T) + \frac{(\dimSt \dimInput)^{1.5} \sigma_\vw^2}{(1-\bar{\theta})^2}\big) \sqrt{T}}
		\end{equation}
	\end{theorem}
\subsubsection{Partial knowledge of dynamics}
\label{subsubsec_knownB}
In practice, the input matrices  $\vB_{1:\numSys}$ correspond to the actuators. One may have their knowledge either from the manufacturers or through various estimation techniques designed for non-dynamical models. From Corollary \ref{cor:learn_dynamics_KnownB}, we know that when $\vB_{1:\numSys}$ are known, no further exploration noise is needed to identify the state matrices $\vA_{1:\numSys}$ or Markov transition matrix $\vT$. 
This can also be applied to the adaptive MJS-LQR setting, and the resulting regret bound can improve (from $\hat{\mathcal{O}}(\log(T) \sqrt{T})$ to $\hat{\mathcal{O}}(\log^3(T))$), since exploration noise incurs additional costs. The result is given in the following corollary, and we omit the proof due to its similarity to the proofs of Theorems \ref{thrm_mainthrm} and \ref{thrm_regretUniformStability}.
\begin{corollary}[Poly-logarithmic regret]\label{corollary_regretKnownB}
	When $\vB_{1:\numSys}$ are known, it suffices to set the exploration noise to be $\sigma_{\vz,\epkidx}=0$ for all $\epkidx$ in Algorithm \ref{Alg_adaptiveMJSLQR}. Then, the regret bound in Theorem \ref{thrm_mainthrm} becomes,
	$\textnormal{Regret}(T) \leq \hat{\Ocal} \parenthesesbig{\frac{\numSys \dimInput (\dimSt + \dimInput) \sigma_\vw^2 }{\pi_{\min}(1-\varrho)} \log^3(T)}
	+ \Ocal\parenthesesbig{\frac{\sqrt{\dimSt \numSys}\log^3(T)}{\delta}}
	$. Additionally, the regret bound in Theorem \ref{thrm_regretUniformStability} becomes, $\textnormal{Regret}(T) \leq \hat{\Ocal} \parenthesesbig{\frac{\numSys \dimInput (\dimSt + \dimInput) \sigma_\vw }{\pi_{\min}(1-\varrho)} \log^3(T)}$.
\end{corollary}

As for the other special case when $\vA_{1:\numSys}$ is known but $\vB_{1:\numSys}$ is unknown, the exploration noise is still needed. One can analyze it as a special case of the general case when neither of them is known. For LTI systems, under certain strong assumptions, e.g. controller non-degeneracy, it is shown that poly-logarithmic regret is attainable for this case \citep{cassel2020logarithmic}. We speculate similar assumptions can lead to poly-logarithmic regret for MJS as well and leave this to the future work.

\section{Discussion}\label{sec_discussion}
In this section, we discuss how one may obtain the initial stabilizing controller for MJS as required in the input to Algorithms \ref{Alg_MJS-SYSID} and \ref{Alg_adaptiveMJSLQR} and the application of our results to offline data-driven control.
\subsection{Initial Stabilizing Controllers}\label{subsec_initStabilizingController}
Having access to an initial stabilizing controller has become a very common assumption in system identification (see for instance \citet{lee2020non} and references therein) and adaptive control \citep{abeille2018improved,cohen2019learning,dean2018regret,ibrahimi2012efficient,simchowitz2020naive} for LTI systems. 
On the other hand, for works where no initial stabilizing controller is required, there is usually a separate warm-up phase at the beginning, where coarse dynamics is learned, upon which a stabilizing controller is computed. Recent non-asymptotic system identification results \citep{faradonbeh2018finite,sarkar2019near} on potentially unstable LTI systems can be used to obtain coarse dynamics without stabilizing controller. One can use random linear feedback to construct a confidence set of the dynamics such that any point in this set can produce a stabilizing controller by solving Riccati equations~\citep{faradonbeh2018finite_tac}. In the model-free setting, \citet{lamperski2020computing} provides asymptotic results and relies on persistent excitation assumption. \citet{chen2021black} designs subtle scaled one-hot vector input and collects the trajectory to estimate the dynamics, then a stabilizing controller can be solved via semi-definite programming.
For MJS or general switched systems, to the best of our knowledge, there is no work on stabilizing unknown dynamics using single trajectory with guarantees. One challenge is, as we discussed in Section \ref{sec:intro}, the individual mode stability and overall mean-square stability does not imply each other due to mode switching. 
However, as outlined below, we can approach this problem leveraging what is recently done for the LTI case in the aforementioned literature (modulo some additional assumptions).

Similar to the LTI case, suppose we could obtain some coarse dynamics estimate $\vAhat_{1:\numSys}, \vBhat_{1:\numSys}, \vThat$, then we can solve for the optimal controller $\vKhat_{1:\numSys}$ for the infinite-horizon MJS-LQR$(\vAhat_{1:\numSys}, \vBhat_{1:\numSys}, \vThat, \vQ_{1:\numSys}, \vR_{1:\numSys})$ via coupled discrete-time algebraic Riccati equations. To investigate when $\vKhat_{1:\numSys}$ can stabilize the MJS($\vA_{1:\numSys}, \vB_{1:\numSys}, \vT$), the key is to obtain sample complexity guarantees for this coarse dynamics, i.e. dependence of estimation error $\norm{\vAhat_i - \vA_i}$, $\norm{\vBhat_i - \vB}$, and $\norm{\vThat - \vT}$ on sample size. Fortunately \citet{du2021certainty} provides the required estimation accuracy under which $\vKhat_{1:\numSys}$ is guaranteed to be stabilizing. Thus, combining \citet{du2021certainty} with the estimation error bounds (in terms of sample size), the required accuracy can be translated to the required number of samples. Note that learning $\vT$ is the same as learning a Markov chain, thus using the mode transition pair frequencies in an arbitrary single MJS trajectory, we can obtain an estimate $\vThat$ as in Algorithm \ref{Alg_MJS-SYSID}, and its sample complexity is given in Lemma \ref{lemma_ConcentrationofMC} in Appendix \ref{appendix_SysIDAnalysis}. The more challenging part is the identification scheme and corresponding sample complexity for $\vAhat_{1:\numSys}$ and $\vBhat_{1:\numSys}$. Here, we outline two potential schemes.
\begin{itemize}
	\item Suppose we can generate $N$ i.i.d. MJS rollout trajectories, each with length $T$ (small $T$, e.g. $T=1$, is preferred to avoid potential unstable behavior and for the ease of the implementation). We can obtain least squares estimates $\vAhat_{1:\numSys}, \vBhat_{1:\numSys}$ using only $\curlybrackets{\vx_T, \vx_{T-1}, \vu_{T-1}, \omega(T-1)}$ from each trajectory, which is similar to the scheme in \citet{dean2019sample} for LTI systems. Since only i.i.d. data is used in the computation, one can easily obtain the sample complexity in terms of $N$.
	\item If each mode in the MJS can run in isolation (i.e. for any $i \in [\numSys]$, $\omega(t)=i$ for all $t$) so that it acts as an LTI system, we can use recent advances on single-trajectory open-loop LTI system identification~\citep{faradonbeh2018finite,sarkar2019near} to obtain coarse estimates together with sample complexity for $\vAhat_i$ and $\vBhat_i$ for every mode $i$. 
\end{itemize}

We also note that while finding an initial stabilizing controller is theoretically very interesting and challenging, most results we know of are limited to simulated or numerical examples (see for instance \citet{lee2020non} and references therein). This is because, from a practical standpoint, an initial stabilizing controller is almost always required in model-based approaches since running experiments with open-loop unstable plants can be very dangerous as the state could explode quickly.

\subsection{Offline Data-Driven Control}\label{subsec_offlineControl}
In many scenarios, we may not be able to perform learning and control in real time due to limited onboard computing resources or measurement sensors. In this case, the dynamics is usually learned in a one-shot way at the beginning, and the resulting controller will be deployed forever without any further update. The controller suboptimality in this non-adaptive setting does not improve over time, thus the regret will increase linearly over time rather than sub-linearly as in our work. The natural performance metric in this case is the time-averaged regret, which can also be viewed as the slope of the cumulative regret with respect to time. The system identification scheme and corresponding sample complexity developed in this paper can also help address this problem.

Suppose we obtain MJS estimate $\vAhat_{1:\numSys}, \vBhat_{1:\numSys}, \vThat$ from a single trajectory of length $T_0$ using Algorithm \ref{Alg_MJS-SYSID} and solve for the controller $\vKhat_{1:\numSys}$ that is optimal for infinite-horizon MJS-LQR$(\vAhat_{1:\numSys}, \vBhat_{1:\numSys}, \vThat, \vQ_{1:\numSys}, \vR_{1:\numSys})$ via coupled discrete-time algebraic Riccati equations. Let $\Jhat:= \limsup_{T \rightarrow \infty} \frac{1}{T} J(\{\vKhat_{\omega(t)} \vx_t\}_{t=0}^T)$ denotes the infinite-horizon average cost incurred when we deploy $\vKhat_{1:\numSys}$ indefinitely. Combining our identification sample complexity result in Theorem \ref{thrm:learn_dynamics} with the infinite-horizon MJS-LQR perturbation result in \citet{du2021certainty}, we can easily obtain an upper bound on the suboptimality, $\Jhat - J^\star \leq \hat{\Ocal}\left( \log^2(T_0) / T_0 \right)$, which provides the required rollout trajectory length $T_0$ if certain suboptimality is desired.

\section{Numerical Experiments}\label{sec:exp}
We provide experiments to investigate the efficiency and verify the theory of the proposed algorithms on synthetic datasets. Throughout, we show results from a synthetic experiment where entries of the true system matrices $(\vA_{1:s},\vB_{1:s})$ are generated randomly from a standard normal distribution. We further scale each $\vA_i$ to have $\|\vA_i\| \leq 0.5$. Since this guarantees the MJS itself is MSS, as we discussed in Sec \ref{sec_sysid}, we set controller $\vK_{1:\numSys}=0$ in system identification Algorithm \ref{Alg_MJS-SYSID} and initial stabilizing controller $\vK_{1:\numSys}^{(0)}=0$ in adaptive MJS-LQR Algorithm \ref{Alg_adaptiveMJSLQR}.
For the cost matrices $(\vQ_{1:s}, \vR_{1:s})$, we set $\vQ_i=\underline{\vQ}_i\underline{\vQ}_i^\top$, and $ \vR_i=\underline{\vR}_i\underline{\vR}_i^\top$ where $\underline{\vQ}_i \in \R^{\dimSt \times \dimSt}$ and $\underline{\vR}_i\in \R^{\dimInput \times \dimInput}$ are generated from a standard normal distribution. The Markov matrix $\vT \in \R_{+}^{\numSys \times \numSys}$ is sampled from a Dirichlet distribution  $\texttt{Dir}((s-1)\cdot \vI_s + 1)$, where $\vI_s $ denotes the identity matrix. We assume that we have equal probability of starting in any initial mode. 

Since for system identification, our main contribution is estimating $\vA_{1:\numSys}$ and $\vB_{1:\numSys}$ of the MJS, we omit the plots for estimating $\vT$. Let $ \hat{\Psi}_{i}=[ \vAhat_{i}, \vBhat_{i}] $ and $\Psi_{i}=[ \vA_{i}, \vB_{i}] $.  We use $\|\hat{\Psi}-\Psi\|/\|\Psi\| := \max_{i \in [\numSys]} \|\hat{\Psi}_{i}-\Psi_{i}\|/\|\Psi_{i}\|$ to investigate the convergence behavior of MJS-SYSID Algorithm \ref{Alg_MJS-SYSID}. The clipping constants in this algorithm, $c_\vx$, and $c_\vz$ are chosen based on their lower bounds provided in Theorem~\ref{thrm_mainthrm}.
The depicted results are averaged over 10 independent Monte Carlo runs. 
\subsection{Performance of MJS-SYSID}\label{sec mjs-sysid perf}

\begin{figure*}
	\centering
\begin{subfigure}[t]{0.23\textwidth}
		\includegraphics[width=\linewidth]{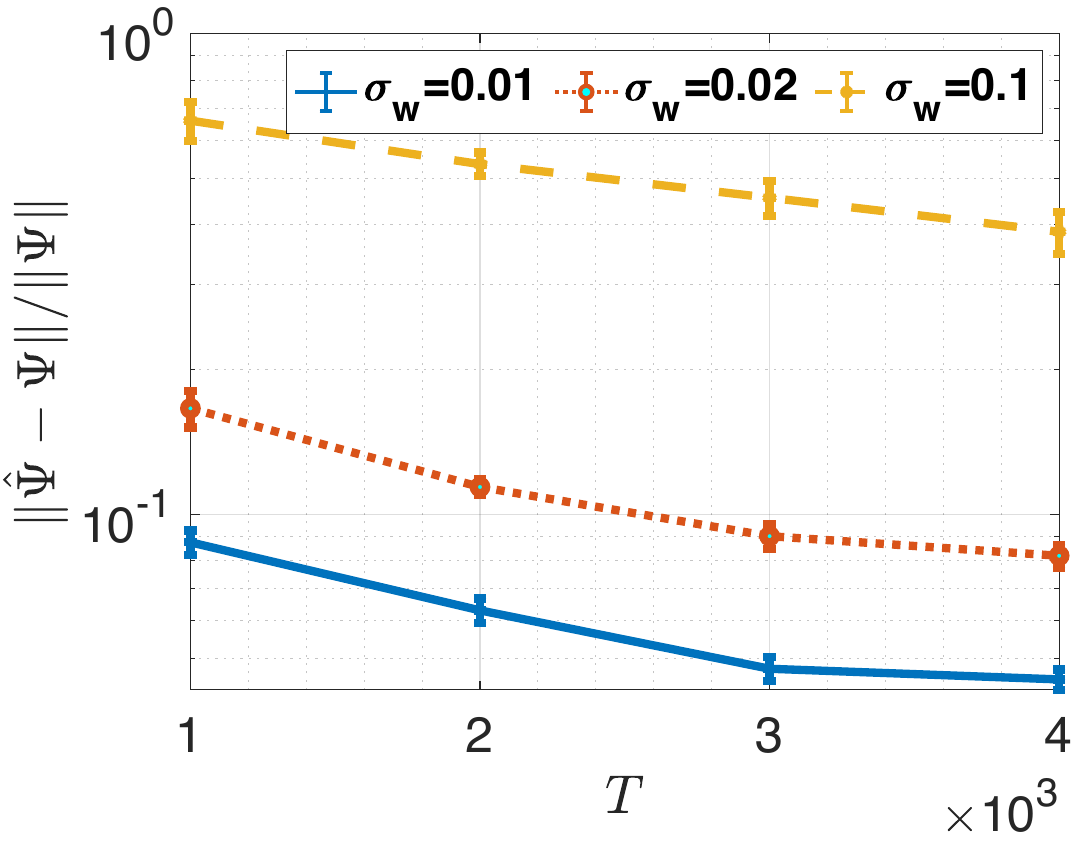}\vspace{-6pt}
		\caption{\mbox{process~noise~$\sigma_\vw$}}\label{fig2a}
\end{subfigure}
	~
\begin{subfigure}[t]{0.23\textwidth}
		\includegraphics[width=\linewidth]{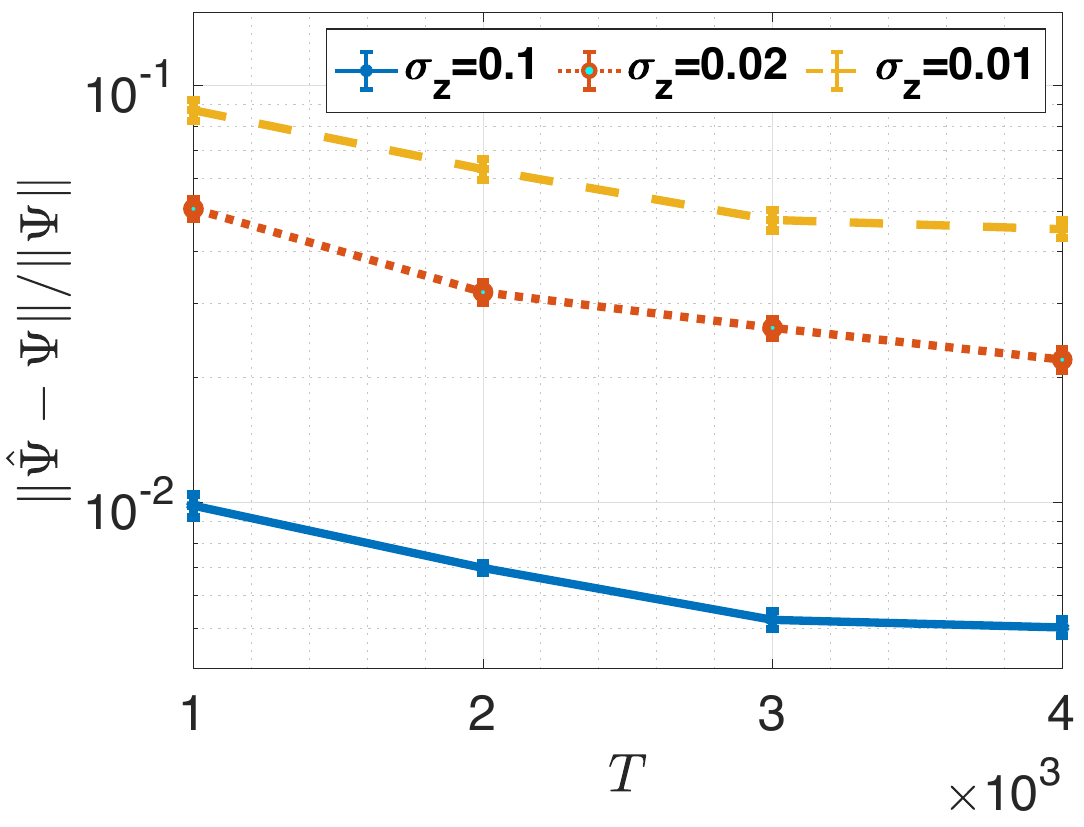}\vspace{-6pt}
		\caption{\mbox{exploration~noise~$\sigma_\vz$}}\label{fig2b}
\end{subfigure}
	~
	\begin{subfigure}[t]{0.23\textwidth}
		\includegraphics[width=\linewidth]{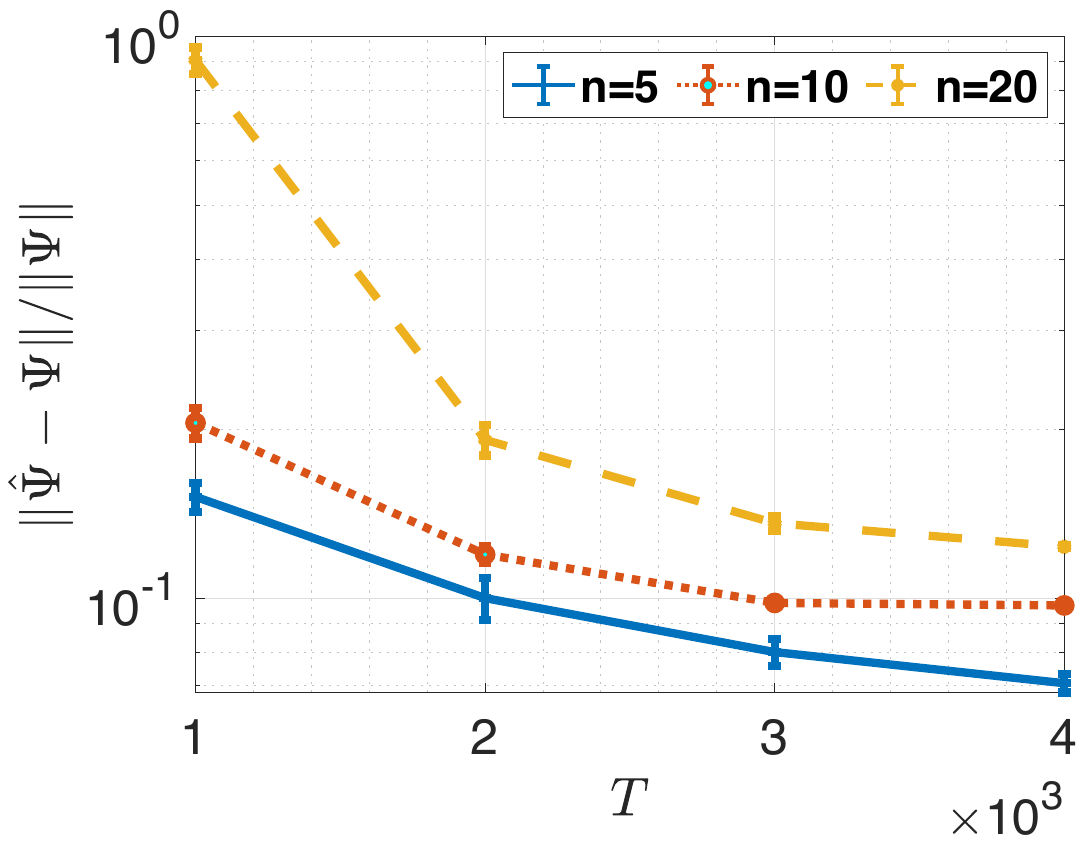}\vspace{-6pt}
		\caption{\mbox{state~dimension~$n$}}\label{fig2c}
\end{subfigure}
	~
	\begin{subfigure}[t]{0.23\textwidth}
		\includegraphics[width=\linewidth]{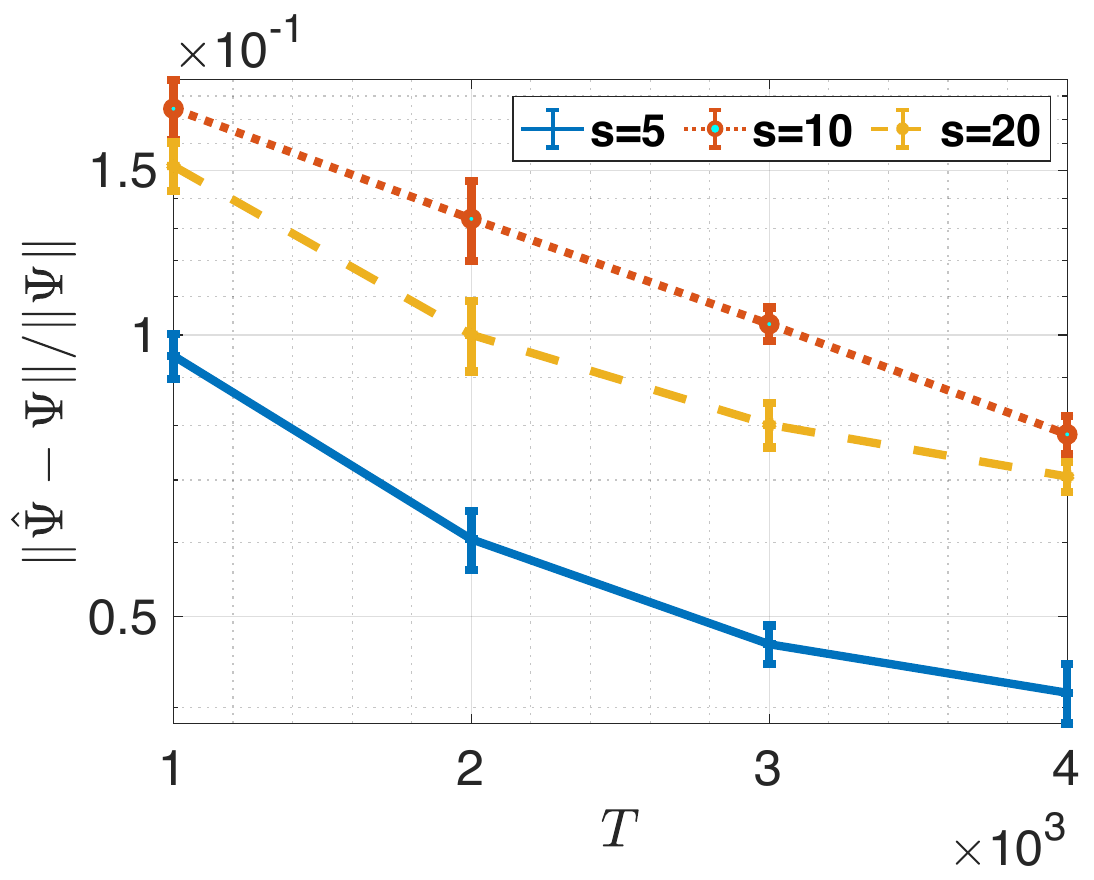}\vspace{-6pt}
		\caption{\mbox{number~of~modes~$s$}}\label{fig2d}
\end{subfigure}
	\caption{Performance profiles of MJS-SYSID with varying: (a) process noise $\sigma_\vw$, (b) exploration noise $\sigma_\vz$, (c) state dimension $n$, and (d) number of modes $s$.}
	\label{fig:sysidsigmawz}
\end{figure*}
\iffalse \begin{figure}[t]
	\centering
	\begin{tabular}{ c @{\hspace{0.2cm}} c }
		\includegraphics[width=0.47\columnwidth]{figs/figsjmlr/figsensigmw.pdf} 	& \includegraphics[width=0.47\columnwidth]{figs/figsjmlr/figsensigmz.pdf}\\
		\small \quad (a) &      \small \quad (b) \\
		\includegraphics[width=0.47\columnwidth]{figs/figsjmlr/figsenpn.pdf}  &
		\includegraphics[width=0.48\columnwidth]{figs/figsjmlr/figsenmode.pdf} \\
		\small \quad (c) &      \small \quad (d) \\
	\end{tabular}
	\caption{Performance profiles of MJS-SYSID with varying: (a) process noise $\sigma_\vw$, (b) exploration noise $\sigma_\vz$, (c) state dimension $n$, (d) number of modes $s$}
\end{figure}
\fi In this section, we investigate the performance of our MJS-SYSID method, i.e., Algorithm~\ref{Alg_MJS-SYSID}. We first empirically evaluate the effect of the noise variances $\sigma_{\vw}$ and $\sigma_{\vz}$. In particular, we study how the estimation errors vary with (i) $\sigma_{\vw}=0.01, \sigma_{\vz} \in \{0.01, 0.02, 0.1\} $ and (ii) $\sigma_{\vz}=0.01, \sigma_{\vw} \in \{0.01, 0.02, 0.1\}$. The number of states, inputs, and modes are set to $n=5$, $p=3$, and $s=5$, respectively. 
Fig.~\ref{fig:sysidsigmawz} (a) and (b) demonstrate how the relative estimation error $\|\hat{\Psi}-\Psi\|/\|\Psi\|$ changes as $T$ increases. 
Each curve on the plot represents a fixed $\sigma_\vw$ and $\sigma_\vz$. 
These empirical results are all consistent with the theoretical bound of MJS-SYSID given in \eqref{sysid bound}. In particular, the estimation errors degrade with increasing $\sigma_\vw$ and decreasing $\sigma_\vz$, respectively. 

Now, we fix $\sigma_\vw = \sigma_\vz = 0.01$ and investigate the performance of the MJS-SYSID with varying number of states, inputs, and modes. Fig.~\ref{fig:sysidsigmawz} (c) and (d) show how the estimation error $\|\hat{\Psi}-\Psi\|/\|\Psi\|$ changes with (left) $s=5$, $n \in \{5, 10, 20\}$, $p=n-2$ and (right) $n=5$,  $p=n-2$, $s\in\{5, 10, 20\}$. As we can see, the MJS-SYSID has better performance with small $n$, $p$ and $s$ which is consistent with \eqref{sysid bound}. 

\begin{comment}
	Our final experiments investigate the system identification bounds in the practically relevant setting where state matrices $\vA_{1:s}$ are unknown but input matrices $\vB_{1:s}$ are known. In particular, we set $\sigma_\vz =0$ and investigate the influence of the number of states $n$, number of modes $s$, and $\sigma_\vw$ on estimation performance. The relative estimation errors with respect to these parameters are illustrated in Figure~\ref{fig:sysidpsw}. Comparing to Figure~\ref{fig:sysidps}, one can see that the relative estimation errors improve with partial knowledge. Further, similar to Figures~\ref{fig:sysidsigmawz} and \ref{fig:sysidps}, the estimation error $\|\hat{\Psi}-\Psi\|/\|\Psi\|$ increases with $n$, $s$, and $\sigma_\vw$ which is consistent with Corollary~\ref{cor:learn_dynamics_KnownB}.
\end{comment}

\subsection{Performance of Adaptive MJS-LQR}\label{sec numer adaptive}
In our next series of experiments, we explore the sensitivity of the regret bounds to the system parameters. In these experiments, we set the initial epoch length $T_0=2000$ and incremental ratio $\gamma=2$. We select five epochs to run Algorithm~\ref{Alg_adaptiveMJSLQR}. As an intermediate step for computing controller $\vK_{1:\numSys}^{(\epkidx+1)}$ in Algorithm \ref{Alg_adaptiveMJSLQR}, the coupled Riccati equations \eqref{eq_CARE} are solved via value iteration, and the iteration stops when the parameter variation between two iterations falls below $10^{-6}$, or iteration number reaches $10^4$.

Fig.~\ref{fig:regunknownb} demonstrates how regret bounds vary with (a) $\sigma_{\vw} \in \{0.001, 0.002, 0.01, 0.02\}$,  $n=10$, $p=s=5$; (b) $\sigma_{\vw}=0.01$, $n=10$, $p=5$, $s \in \{4, 6, 8, 10\}$,  and (c) $\sigma_{\vw}=0.01$, $s=10$, $p=5$, $n \in  \{4, 6, 8, 10\}$. We see that the regret degrades as $\sigma_\vw, n $, and $s$ increase. We also see that 
when $\sigma_\vw$ is large ($T$ is small), the regret becomes worse quickly as $n$ and $s$ grow larger. These results are consistent with the theoretical bounds in Theorem~\ref{thrm_mainthrm}.   

\begin{figure}
	\centering
	\begin{subfigure}[t]{0.28\textwidth}
		\includegraphics[width=\linewidth]{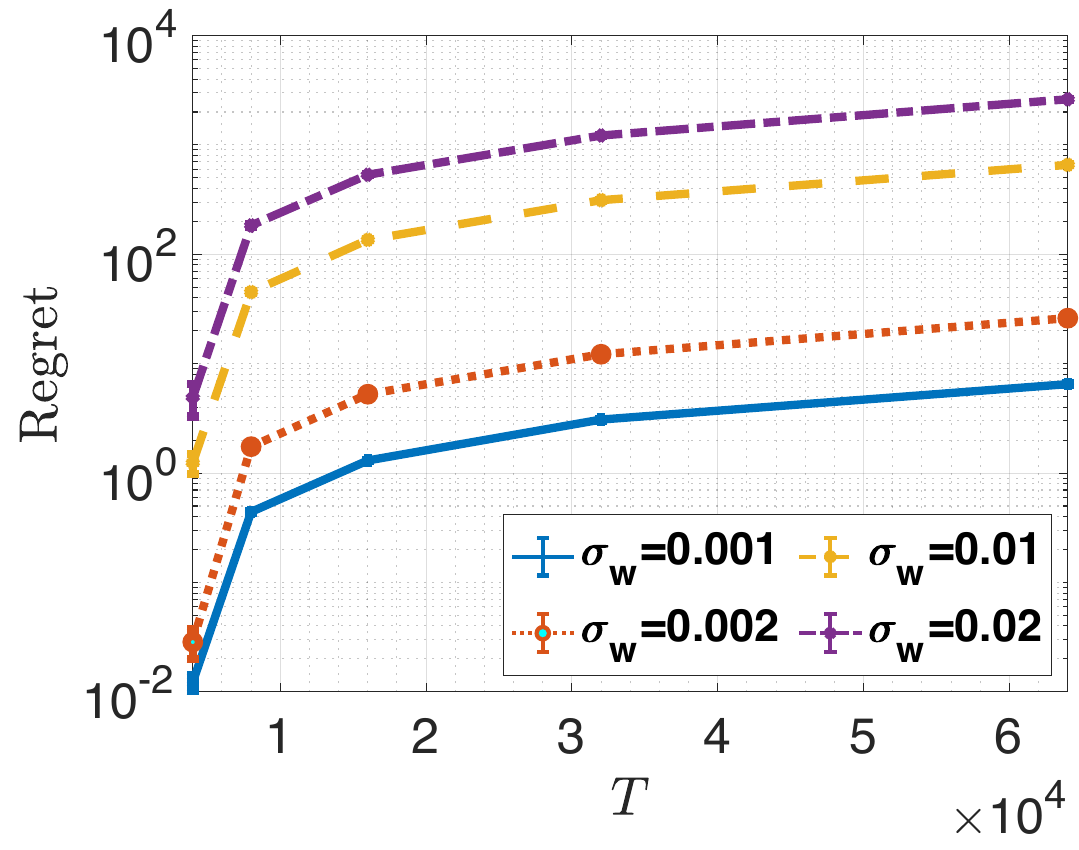}\vspace{-6pt}
		\caption{\mbox{process~noise~$\sigma_\vw$}}\label{fig:rega}
\end{subfigure}
	~	
\begin{subfigure}[t]{0.28\textwidth}
		\includegraphics[width=\linewidth]{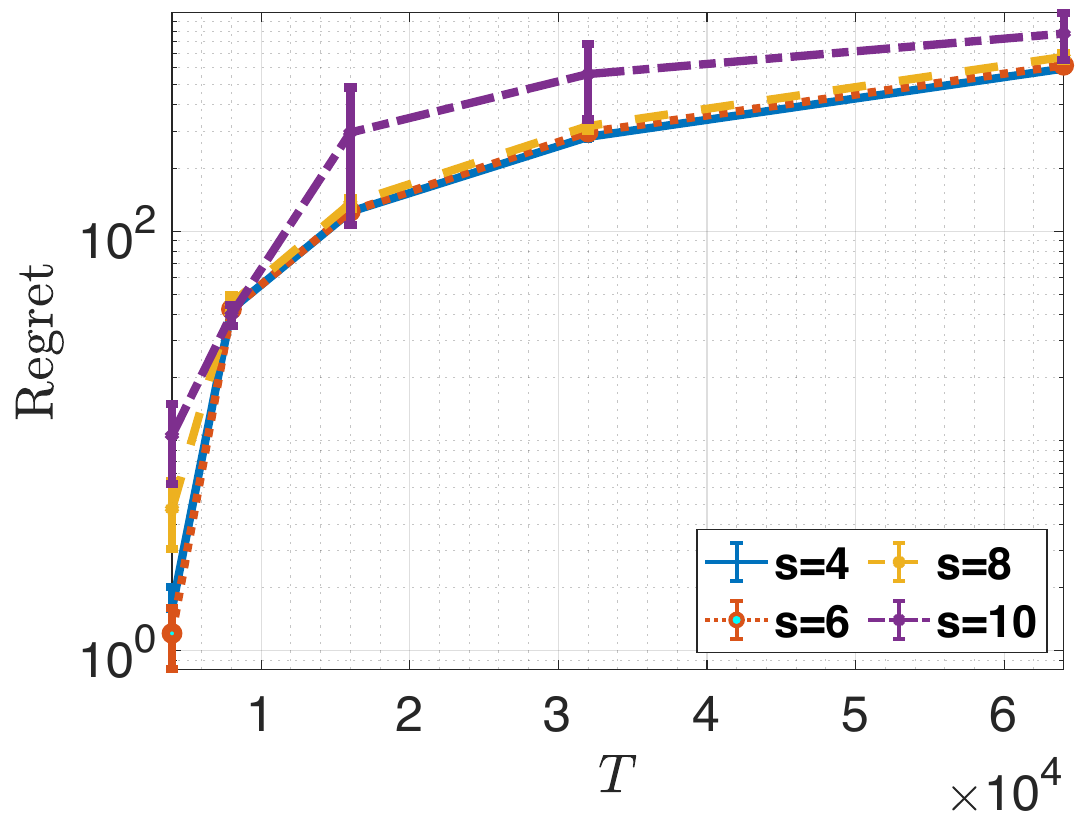}\vspace{-6pt}
		\caption{\mbox{number~of~modes~$s$}}\label{fig:regb}
\end{subfigure}
	~	
\begin{subfigure}[t]{0.28\textwidth}
		\includegraphics[width=\linewidth]{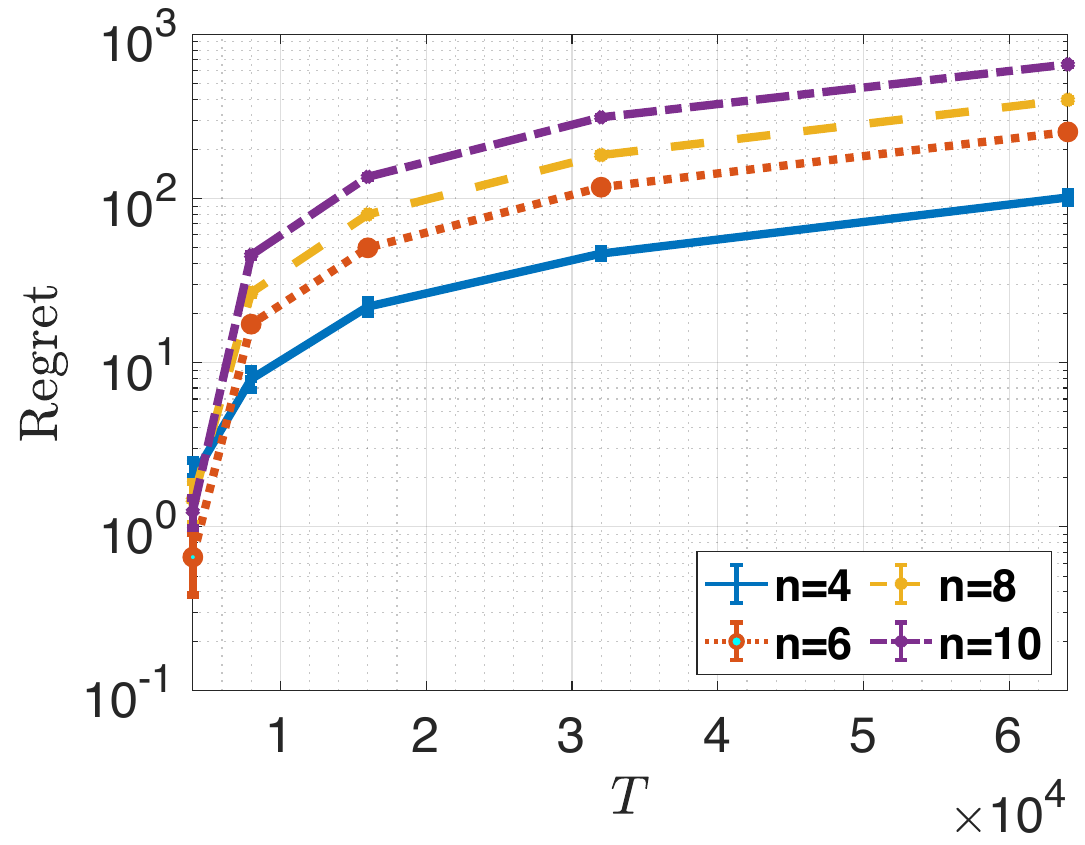}\vspace{-6pt}
		\caption{\mbox{state~dimension~$n$}}\label{fig:regc}
\end{subfigure}
	\caption{Performance profiles of Adaptive MJS-LQR with varying: (a) process noise $\sigma_\vw$, (b) number of modes $s$, (c) state dimension $n$.}
	\label{fig:regunknownb}
\end{figure}

\section{Conclusions and Discussion}\label{sec:conc}

Markov jump systems are fundamental to a rich class of control problems where the underlying dynamics are changing with time. Despite its importance, statistical understanding (system identification and regret bounds) of MJS have been lacking due to the technicalities such as Markovian transitions and weaker notion of mean-square stability. At a high-level, this work overcomes (much of) these challenges to provide finite sample system identification and model-based adaptive control guarantees for MJS. Notably, resulting estimation error and regret bounds are optimal in the trajectory length and coincide with the standard LQR up to polylogarithmic factors. As a future work, it would be interesting and of practical importance to investigate the case when mode is not observed, which makes both system identification and adaptive quadratic control problems non-trivial.

We want to mention possible negative societal impacts. While our work is theoretical and has many potential positive impacts in reinforcement learning, robotics, and autonomous systems, there are also potential negative applications in the military (e.g.~with drone control) and for malicious actors (e.g.~computer network hackers), among others. Additionally, all our work was built on stochastic noise assumptions, whereas in reality intelligent autonomous systems may instead encounter adversarial behavior. There is potential here for future work to extend our approach to non-stochastic noise or even non-Markovian / non-random switching among states.

\section*{Acknowledgements}
Y. Sattar and S. Oymak were supported in part by NSF grant CNS-1932254 and S. Oymak was supported in part by NSF CAREER award CCF-2046816 and ARO MURI grant W911NF-21-1-0312. Z. Du and N. Ozay were supported in part by ONR under grant N00014-18-1-2501 and N. Ozay was supported in part by NSF under grant CNS-1931982 and ONR under grant N00014-21-1-2431. D. Ataee Tarzanagh, Z. Du, and L. Balzano were supported in part by NSF CAREER award CCF-1845076 and AFOSR YIP award FA9550-19-1-0026.

\newpage 
\bibliographystyle{alpha}
\bibliography{AdaptiveMJSLQR}

\newpage
\tableofcontents
\appendix
\section{Preliminaries} \label{appendix_Preliminaries}
In addition to the notations defined in Section \ref{sec:intro}, we define a few more here to be used throughout the appendix. 
For a matrix $\vV$, $\underline{\sigma}(\vV)$, $\norm{\vV}_1$, and $\norm{\vV}_F$ denote its smallest singular value, $\ell_1$ norm and Frobenius norm, respectively. We use $\vek(\vV)$ to denote the vectorization of a matrix $\vV$ and define $\norm{\vV}_+:= \norm{\vV}+1$. 
We define $\underline{\sigma}(\vV_{1:\numSys}):= \min_{i \in [\numSys]} \underline{\sigma}(\vV_i)$ and $\norm{\vV_{1:\numSys}}_+ := \max_{i \in [\numSys]} \norm{\vV_i}_+$. We use $\vI_\dimSt$ to denote the identity matrix of dimension $\dimSt$. $\onevec_{\dimSt}$ denotes the all $1$ vector of dimension $\dimSt$ and $\indicator{\cdot}$ denotes the indicator function. Lastly, we use $\lesssim$ and $\gtrsim$ for inequalities that hold up to a constant factor.

To begin, we define the following quantity which will be used throughout to quantify the decay of a square matrix $\vM$.
\begin{definition} \label{def_tau}
	For a square matrix $\vM$ with $\rho(\vM) < 1$, there exists $\rho \in (\rho(\vM), 1)$ such that, we have
	\begin{equation}
		\tau(\vM) := \sup_{k\in \mathbb{N}} \curlybrackets{\norm{\vM^k}/\rho^k}.
	\end{equation}
\end{definition}
Note that $\tau(\vM)$ is finite by Gelfand's formula, and it is easy to see that $\tau(\vM) \geq 1$. This quantity measures the transient response of a non-switching system with state matrix $\vM$ and can be upper bounded by its $\mathcal{H}_\infty$ norm \citep{tu2017non}. In this work, we will mainly use this quantity to evaluate the augmented state matrix for an MJS defined in Section \ref{subsec_IntroMJS}.

For a Markov chain with transition matrix $\vT$, we let $\vpi_0 \in \R_+^{\numSys}$ denote the initial state distribution and $\vpi_t$ denote the transient state distribution, i.e. $\P(\omega(t)=i) = \vpi_t(i)$. Then, it is easy to see $\vpi_t^\top = \vpi_0^\top \vT^t $.
Note that $\vpi_t$ is essentially a convex combination of rows of matrix $\vT^t$, then by triangle inequality, we have $\norm{\vpi_t - \vpi_{\infty}}_1 \leq \max_{i \in [\numSys]} \norm{([\vT^t]_{i,:})^\top - \vpi_{\infty}}_1$. Thus, for an ergodic Markov matrix $\vT$, we define the following to quantify the convergence of $\norm{\vpi_t - \vpi_{\infty}}_1$.
\begin{definition}\label{def:mixing time}
	For an ergodic Markov matrix $ \Tb \in \R_+^{\numSys \times \numSys}$, let $\tau_{\rm MC}>0$ and $\rho_{\rm MC} \in [0,1)$ be two constants \citep[Theorem 4.9]{levin2017markov} such that 
	\begin{equation}
		\max_{i \in [\numSys]} \norm{([\vT^t]_{i,:})^\top - \vpi_{\infty}}_1 \leq \tau_{\rm MC} \rho_{\rm MC}^t.
	\end{equation}
	Furthermore, we define the mixing time of $\vT$ as
	\begin{equation}
		t_{\rm MC}(\epsilon) := \min \curlybracketsbig{t \in \mathbb{N} : \max_{i \in [s]} \frac{1}{2} \norm{([\vT^t]_{i,:})^\top - \vpi_{\infty} }_1 \leq \epsilon}.
	\end{equation}
	When the parameter $\epsilon$ is omitted, it denotes $t_{\rm MC} := t_{\rm MC}(\frac{1}{4})$, i.e., the mixing time defined in Section \ref{subsec_IntroMJS}.
\end{definition}
Note that $\tau(\vM)$ and $\tau_{\rm MC}$ have similar roles except $\tau(\vM)$ is usually used to study state matrices while $\tau_{\rm MC}$ is for Markov matrices. For $\vM$, we have $\norm{\vM^k} \leq \tau(\vM) \rho(\vM)^k$, and for a Markov matrix $\norm{\vT^t -  \onevec_{\numSys} \vpi_\infty^\top}_1 \leq \tau_{\rm MC} \rho_{\rm MC}^t$.

In the following, we define a few notations to ease the exposition in the appendix. Note that, for notations under parameterized form, i.e., notations which are functions of $(\delta, \rho, \tau)$ etc., one can choose these parameters freely to get different deterministic quantities.

{ 
	\renewcommand{\arraystretch}{1.2}
	\begin{table}[!t]
		\caption{Notations --- MJS-LQR Perturbation}
		\label{table_Notation_Perturbation}
		\centering
		\begin{adjustbox}{max width=\columnwidth}
			\begin{tabular}{||l||l||}
                \hline
$\xi
				$ & \qquad $\min \curlybrackets{\norm{\Bb_{1:\numSys}}_+^{-2} \norm{\vR_{1:\numSys}^\inv}_+^{-1}
					\norm{\vLstar_{1:\numSys}}_+^{-2}, \underline{\sigma}(\Pbs_{1:s}) }$
				\\
				$\xi'
				$ & \qquad $\norm{\Ab_{1:\numSys}}_+^{2} \norm{\Bb_{1:\numSys}}_+^{4} \norm{\Pbs_{1:\numSys}}_+^{3} \norm{\vR_{1:\numSys}^\inv}_+^{2} $
				\\
				$\Gamma_\star
				$ & \qquad $\max \curlybrackets{\norm{\vA_{1:\numSys}}_+, \norm{\vB_{1:\numSys}}_+, \norm{\vP^\star_{1:\numSys}}_+, \norm{\vK^\star_{1:\numSys}}_+}$
				\\
$C_{\vA, \vB, \vT}^{\vK} 
				$ & \qquad $28 \sqrt{\dimSt \numSys} \tau(\vLtil^\star) (1-\rho^\star)^{-1}  \parenthesesbig{\underline{\sigma}(\vR_{1:\numSys})^\inv + \Gamma_\star^3 \underline{\sigma}(\vR_{1:\numSys})^{-2}} \Gamma_\star^3 \xi'$ 
				\\
				$C_\vK^J 
				$ & \qquad $2 \numSys^{1.5} \sqrt{\dimSt} \min\curlybrackets{\dimSt, \dimInput}(\norm{\Rb_{1:s}}+\Gamma_\star^3) \frac{\tau(\vLtil^\star)}{1-\rho^\star} $
				\\
				$\bar{\epsilon}_{\vK} 
				$ & \qquad $\min \curlybracketsbig{ \norm{\vKstar_{1:\numSys}},
					\frac{1-\rho^\star}{2 \sqrt{\numSys} \tau(\vLtil^\star) (1+2 \norm{\vLstar_{1:\numSys}}) \norm{\vB_{1:\numSys}}) } }$
				\\
				$\bar{\epsilon}_{\vA,\vB,\vT}^{LQR} 
				$ & \qquad $ \frac{(1-\rho^\star) \min \curlybrackets{\Gamma_\star, \underline{\sigma}(\Rb_{1:s})^2 \bar{\epsilon}_{\vK}}}{28 \sqrt{\dimSt \numSys} \tau(\vLtil^\star) \Gamma_\star^3 (\underline{\sigma}(\Rb_{1:s})+\Gamma_\star^3) } \xi'^\inv$
				\\
				$\bar{\epsilon}_{\vA,\vB,\vT} 
				$ & \qquad $\min \curlybracketsbig{
					\frac{\xi\parentheses{1-\rho^\star}^2}{204 \dimSt \numSys \tau(\vLtil^\star)^2 \xi'}, \norm{\Bb_{1:s}}, \  \underline{\sigma}(\vQ_{1:s}),\ 
					\bar{\epsilon}_{\vA,\vB,\vT}^{LQR} }
				$
				\\
				\hline
\end{tabular}
		\end{adjustbox}
	\end{table}
}

\textbf{Table \ref{table_Notation_Perturbation}} lists the notations related to infinite-horizon MJS perturbation results closely following the notations in \citet{du2021certainty}. It provides several sensitivity parameters, e.g., how the optimal controller $\vKstar_{1:\numSys}$ varies with perturbations in the MJS parameters $\vA_{1:\numSys}, \vB_{1:\numSys}$, and $\vT$ and how the MJS-LQR cost $J$ varies with the controller $\vK_{1:\numSys}$. It also provides certain upper bounds on the variations in $\vA_{1:\numSys}, \vB_{1:\numSys}, \vT$, and $\vK_{1:\numSys}$ such that the perturbation theory holds. In this table, $\vR_{1:\numSys}^\inv := \curlybrackets{\vR_i^\inv}_{i=1}^\numSys$ and recall $\norm{\cdot}_+:= \norm{\cdot} + 1$.

{ 
	\renewcommand{\arraystretch}{1.2}
	
	\begin{table}[ht]
		\caption{Notations --- Trajectory Length}
		\label{table_Notation_TrajLen}
		\centering
		\begin{adjustbox}{max width=\columnwidth}
			\begin{tabular}{||l||l||}
\hline
				$\bar{\sigma}_\vz$ (\emph{depending on context}) & \qquad $\sigma_\vz$ or $\sigma_{\vz,0}$ or $\sqrt{\norm{\vSigma_\vz}}$
				\\
				$\bar{\sigma}_\vw$ (\emph{depending on context}) & \qquad $\sigma_\vw$ or $\sqrt{\norm{\vSigma_\vw}}$
				\\    
				$C_\vz $ & \qquad $ \bar{\sigma}_\vz/\bar{\sigma}_\vw$ 
				\\
				$\bar{\sigma}^2 \quad$ & \qquad $ \norm{\vB_{1:\numSys}}^2 \bar{\sigma}_{\vz}^2 + \bar{\sigma}_\vw^2$\\
$\bar{\tau} $ & \qquad $ \max \curlybrackets{\tau(\vLtil^{(0)}), \tau(\vLtil^\star)}$
				\\
				$\bar{\rho} $ & \qquad $ \max\curlybrackets{\rho(\vLtil^{(0)}), \frac{1+\rho^\star}{2}}$ 
				\\
				$
				C_{MC}
				$ & \qquad $ t_{\rm MC} \cdot \max \{3, 3-3 \log(\pi_{\max} \log(\numSys)) \}$    
				\\
				$
				\underline{T}_{MC,1} (C_{MC}, \delta)$ & \qquad $ \parenthesesbig{68 C_{MC} \pi_{\max} \pi_{\min}^{-2}  \log(\frac{\numSys}{\delta}) }^2$
				\\
				$
				\underline{T}_{MC}(C_{MC}, \delta)
				$ & \qquad $ \parenthesesbig{612 C_{MC} \pi_{\max} \pi_{\min}^{-2} \log(\frac{2 \numSys}{\delta}) }^2 $
				\\
				$
				\underline{T}_{cl,1}(\rho, \tau)
				$ & \qquad $ \frac{(1 - \rho)^2}{4 \dimSt^{1.5} \sqrt{\numSys} \tau \bar{\sigma}^4}$
				\\
				$
				\underline{T}_{N}(C_{MC}, \delta, \rho, \tau)
				$ & \qquad $ \max \curlybrackets{\underline{T}_{MC}(C_{MC}, \frac{\delta}{2}), \underline{T}_{cl,1}(\rho, \tau)}$
				\\
				$\varrho$ &\qquad $\frac{1}{\pi_{\min}}\sqrt{\frac{17 \pi_{\max}t_{MC}(\pi_{\min}/2)\log(2st_{ MC}(\pi_{\min}/2)/\delta)}{T - 2t_{MC}(\pi_{\min}/2)}}$ \\
				$\bar{\Gamma} $ &\qquad $\sqrt{\dimSt \numSys} \tau_{\vLtil}
					\big(\expctn[\norm{\vx_0}^2]/\sigma_{\vw}^2 + (\sigma_{\vz}^2/\sigma_{\vw}^2)\sqrt{\dimSt} \norm{\vB_{1:\numSys}}^2T + \sqrt{\dimSt}T\big) + p$ \\
				$ \underline{T}_{id}(\delta)$ &\qquad $\max\bigg\{2t_{MC}(\pi_{\min}/2), \frac{(n+p) + \log(6s\bar{\Gamma}/\delta) + \log(6s/\delta)}{\pi_{\min}(1 - \varrho)} \bigg\}$ \\
$
				\underline{T}_{id, N}(\delta)
				$ & \qquad $ \max \curlybracketsbig{
						\underline{T}_{id}(\frac{\delta}{2 L}), 
						\underline{T}_{N}(\frac{L}{\log(T)}, \frac{\delta}{2 L}, \bar{\rho}, \bar{\tau})} 
				$
				\\
				\multirow{2}{*}{$\underline{T}_{rgt, \bar{\epsilon}}(\delta, T)$} & 
				\qquad $\Ocal( \log(\frac{1}{\delta}) \frac{(\dimSt + \dimInput)^2}{\pi_{\min}^2 (1-\varrho)^2\bar{\epsilon}_{\vA,\vB,\vT}^{4}}   \log^2(T))
				$ \\
				& \qquad $\Ocal( \log(\frac{1}{\delta}) \frac{(\dimSt + \dimInput)^2}{\pi_{\min}^2 (1-\varrho)^2\bar{\epsilon}_{\vA,\vB,\vT}^{2}}   \log(T))
				$ (when $\vB_{1:\numSys}$ is known)
				\\
				$
				\underline{T}_{\vx_0}(\delta)  
				$ & \qquad $\frac{1}{\gamma \log(1/\bar{\rho})} \max \curlybrackets{\frac{2}{\log(\gamma)}, \log(\frac{\pi^2 \sqrt{\dimSt \numSys} \bar{\tau}}{3 \delta})}$
				\\    
				$
				\underline{T}_{rgt} (\delta, T) $ & \qquad $ 
				\max 
				\curlybracketsbig{
					\underline{T}_{\vx_0}(\delta),
					\underline{T}_{rgt, \bar{\epsilon}}(\delta, T),
					\underline{T}_{MC,1} (\delta),
					\underline{T}_{id, N}(\delta)
				}$\\
\hline
			\end{tabular}
		\end{adjustbox}
	\end{table}
}

\textbf{Table \ref{table_Notation_TrajLen}} introduces notations and constants related to the choice of tuning parameters, and the shortest trajectory (initial epoch) length such that theoretical performance guarantees can be achieved. Recall that $\vK_{1:\numSys}^{(0)}$ is the stabilizing controller for epoch $0$ in Algorithm \ref{Alg_adaptiveMJSLQR}. We let $\vL_{i}^{(0)}:=\vA_i + \vB_i \vK_i^{(0)}$, for all $i \in [\numSys]$, denote the closed-loop state matrix, and $\vLtil^{(0)} \in \R^{\numSys \dimSt^2 \times \numSys \dimSt^2}$ denotes the augmented closed-loop state matrix with $ij$-th $\dimSt^2 {\times} \dimSt^2$ block given by $\squarebrackets{\vLtil^{(0)}}_{ij} = [\vT]_{ji} \vL_j^{(0)} \otimes \vL_j^{(0)}$. $\tau(\cdot)$ is as in Definition \ref{def_tau} and $\rho(\cdot)$ denotes the spectral radius. For the infinite-horizon MJS-LQR($\vA_{1:\numSys}, \vB_{1:\numSys}, \vT, \vQ_{1:\numSys}, \vR_{1:\numSys}$) problem, we let $\vPstar_{1:\numSys}$ denote the solution to cDARE given by \eqref{eq_CARE} and $\vKstar_{1:\numSys}$ denotes the optimal controller which can be computed via \eqref{eq_optfeedback} with $\vPstar_{1:\numSys}$. Similarly, we define $\vLstar_{1:\numSys}$ and $\vLtil^\star$ to be the corresponding closed-loop state matrix and augmented closed-loop state matrix respectively and $\rho^\star:= \rho(\vLtil^\star)$. $\pi_{\max}$ and $\pi_{\min}$ are the largest and smallest elements in the stationary distribution of the ergodic Markov matrix $\vT$. For the definition of $\underline{T}_{rgt, \bar{\epsilon}}(\delta, T) $, notation $\bar{\epsilon}_{\vA,\vB,\vT}$ is defined in Table \ref{table_Notation_Perturbation}. As a slight abuse of notation, $T$ in $\underline{T}_{rgt, \bar{\epsilon}}(\delta, T)$ and $\Ccal$ are merely arguments to be replaced with specific quantities depending on the context.

\subsection{MJS Covariance Dynamics Under MSS}\label{appendix_Preliminaries_MJSDynamicsMSS} Consider MJS($\vA_{1:\numSys}, \vB_{1:\numSys}, \vT$) with process noise $\vw_t \sim \N(0, \vSigma_\vw)$ and input $\vu_t = \vK_{\omega(t)} \vx_t + \vz_t $ under a stabilizing controller $\vK_{1:\numSys}$ and excitation for exploration $\vz_t \sim \N(0, \vSigma_\vz)$. Let $\vL_i := \vA_i + \vB_i \vK_i$ be the closed-loop state matrix. Let $\vLtil \in \R^{\numSys \dimSt^2 \times \numSys \dimSt^2}$ be the augmented closed-loop state matrix with $ij$-th $\dimSt^2 {\times} \dimSt^2$ block given by $\squarebrackets{\vLtil}_{ij} = [\vT]_{ji} \vL_j \otimes \vL_j$.
Let $\tau_{\vLtil}>0$ and $\rho_{\vLtil} \in [0,1]$ be two constants such that $\norm{\vLtil^k} \leq \tau_{\vLtil} \rho_{\vLtil}^k$. By definitions of $\tau(\vLtil)$ and $\rho(\vLtil)$, one can choose them for $\tau_{\vLtil}$ and $\rho_{\vLtil}$ respectively. Let $\vSigma_i(t) := \expctn[\vx_t \vx_t^\top \indicator{\omega(t)=i}]$, $\vSigma(t) := \expctn[\vx_t \vx_t^\top]$,
\begin{align} \label{eq_sBtilPitil}
	\vs_t := 
	\begin{bmatrix}
		\vek(\vSigma_1(t)) \\
		\vdots \\
		\vek(\vSigma_\numSys(t)) \\
	\end{bmatrix}, \quad
	\vBtil_{t}
	&:=
	\begin{bmatrix}
		\sum_{j=1}^\numSys \vpi_{t-1}(j) \vT_{j1} (\vB_j \otimes \vB_j) \\
		\vdots\\
		\sum_{j=1}^\numSys \vpi_{t-1}(j) \vT_{j \numSys} (\vB_j \otimes \vB_j) \\
	\end{bmatrix}, \quad
	\textnormal{and}~~\vPitil_{t} := \vpi_{t} \otimes \vI_{\dimSt^2}.
\end{align}
The following lemma shows how $\vs_t$ depends on $\vs_0$, $\vSigma_\vz$, and $\vSigma_\vw$, which will be used to upper bound $\expctn[\norm{\vx_t}^2]$ in Lemma \ref{lemma_statebound}.
\begin{lemma}\label{lemma_covarianceDynamics}
	The vectorized covariance $\vs_t$ has the following dynamics,
	\begin{align*}
		\vs_t = \vLtil^t \vs_0 + (\vBtil_t + \vLtil \vBtil_{t-1} + \cdots + \vLtil^{t-1} \vBtil_{1}) \vek(\vSigma_\vz) + (\vPitil_t + \vLtil \vPitil_{t-1} + \cdots + \vLtil^{t-1} \vPitil_{1}) \vek(\vSigma_\vw). \nn
	\end{align*}
\end{lemma}
\parentchild{}{Lemma \ref{lemma_statebound}}{Zhe}{Zhe}
\begin{proof}
	To begin, we evaluate $\vSigma_i(t)$, from the equivalent MJS dynamics $\vx_{t+1} = \vL_{\omega(t)} \vx_t + \vB_{\omega(t)} \vz_t + \vw_t$, as follows,
	\begin{equation}
		\begin{split}
			\expctn[\vx_{t+1}\vx_{t+1}^\top \indicator{\omega(t+1)=i}]
			&= \sum_{j=1}^\numSys \expctn[\vL_j \vx_t \vx_t^\top \vL_j \indicator{\omega(t+1)=i, \omega(t)=j}]\\
			&+ \sum_{j=1}^\numSys \expctn[\vB_j \vz_t \vz_t^\top \vB_j^\top \indicator{\omega(t+1)=i, \omega(t)=j}] + \expctn[\vw_t \vw_t^\top \indicator{\omega(t+1)=i}].
		\end{split}
	\end{equation}
	Since $\vw_t \sim \N(0, \vSigma_\vw)$ and $\vz_t \sim \N(0, \vSigma_{\vz})$, we get \begin{align*}
		\vSigma_i(t+1) &= \sum_{j=1}^\numSys \vT_{ji} \vL_j \vSigma_j(t) \vL_j^\top 
		+ \sum_{j=1}^\numSys \vpi_t(j) \vT_{ji} \vB_j \vSigma_\vz \vB_j^\top + \vpi_{t+1}(i) \vSigma_\vw.
	\end{align*}
	Vectorizing both sides of the above equation, we have
	\begin{align*}
		\vek(\vSigma_i(t+1)) &= \sum_{j=1}^\numSys \vT_{ji} (\vL_j \otimes \vL_j) \vek(\vSigma_j(t)) \\
		&+ \sum_{j=1}^\numSys \vpi_t(j) \vT_{ji} (\vB_j \otimes \vB_j) \vek(\vSigma_\vz)
		+ \vpi_{t+1}(i) \vek(\vSigma_\vw).
	\end{align*}
	Stacking this for every $i \in [s]$, we obtain
		\begin{align}
		\begin{bmatrix} \vek(\vSigma_1(t+1)) \\ \vdots \\ \vek(\vSigma_\numSys(t+1)) \end{bmatrix}
		&=
		\vLtil
		\begin{bmatrix} \vek(\vSigma_1(t)) \\ \vdots \\ \vek(\vSigma_\numSys(t)) \end{bmatrix}+ 	\vBtil_{t+1} \vek(\vSigma_\vz) + \vPitil_{t+1} \vek(\vSigma_\vw). \label{eq_covEvolution}
		\end{align}
	Propagating this dynamics from $t$ to $0$ gives the desired result.
\end{proof}

We next provide a key lemma that upper bounds $\expctn[\norm{\vx_t}^2]$ and $\norm{\vSigma(t)}_F$, which are later used extensively in system identification analysis.
\begin{lemma}\label{lemma_statebound}
	For $\expctn[\norm{\vx_t}^2]$ and $\norm{\vSigma(t)}_F$, we have
\begin{align}
			\expctn[\norm{\vx_t}^2]  
			&\leq 
			\sqrt{\dimSt \numSys} \tau_{\vLtil}
			\big(
			\rho_{\vLtil}^t \expctn[\norm{\vx_0}^2]+
			\sqrt{\dimSt} \norm{\vB_{1:\numSys}}^2 \norm{\vSigma_\vz} \sum_{t'=1}^t \rho_{\vLtil}^{t-t'} +
			\sqrt{\dimSt} \norm{\vSigma_\vw}
			\sum_{t'=1}^t \rho_{\vLtil}^{t-t'} \big), \label{eq_statebound}\\
			\norm{\vSigma(t)}_F
			&
			\leq \sqrt{\numSys} \tau_{\vLtil}
			\big(
			\rho_{\vLtil}^t \expctn[\norm{\vx_0}^2]
			+
			\sqrt{\dimSt} \norm{\vB_{1:\numSys}}^2 \norm{\vSigma_\vz}
			\sum_{t'=1}^t \rho_{\vLtil}^{t-t'} +
			\sqrt{\dimSt} \norm{\vSigma_\vw}
			\sum_{t'=1}^t \rho_{\vLtil}^{t-t'} \big). \label{eq_varbound}
	\end{align}
\end{lemma}
\parentchild{Lemma \ref{lemma_covarianceDynamics}}{Lemma \ref{lemma_EnoughSample2}, Lemma \ref{lemma_propertiesoftrMH}, Lemma \ref{lemma_regretanalysispart1}, Proposition \ref{proposition_epochInitState}}{Zhe}{Zhe}

\begin{proof}
	First we derive an upper bound for $\expctn[\norm{\vx_t}^2]$. The upper bound for $\norm{\vSigma(t)}_F$ follows similarly. For state $\vx_t$, we have
	\begin{align*}
		\expctn[\norm{\vx_t}^2] 
		&= \sum_{i=1}^{s} \expctn [\norm{\vx_t}^2 \indicator{\omega(t)=i}] = \sum_{i=1}^{s} \tr\ ( \expctn[\vx_t \vx_t^\top \indicator{\omega(t)=i}] ), \\
		&= \sum_{i=1}^{s} \tr(\vSigma_i(t)) = \sum_{i=1}^{s} \sum_{j=1}^{n} \lambda_j(\vSigma_i(t)) \leq  \mysqrt[1pt]{\dimSt \numSys\sum_{i=1}^{s} \sum_{j=1}^{n} \lambda_j^2(\vSigma_i(t))},\\
		&\leq \mysqrt[1pt]{\dimSt \numSys\sum_{i=1}^{s} \norm{\vSigma_i(t)}_F^2 }.    
	\end{align*}
	Then, by definition of $\vs_t$ in \eqref{eq_sBtilPitil}, we have
	\begin{equation}
		\expctn[\norm{\vx_t}^2]  \leq \sqrt{\dimSt \numSys} \norm{\vs_t}.
	\end{equation}
	Now, applying the dynamics of $\vs_t$ from Lemma \ref{lemma_covarianceDynamics}, we have
	\begin{align}
		\expctn[\norm{\vx_t}^2]  
		&\leq 
		\sqrt{\dimSt \numSys} 
		\big(
		\norm{\vLtil^t} \norm{\vs_0}
		+
		\sum_{t'=1}^t \norm{\vLtil^{t-t'}} \norm{\vBtil_{t'} \vek(\vSigma_\vz)} + 
		\sum_{t'=1}^t \norm{\vLtil^{t-t'}} \norm{\vPitil_{t'} \vek(\vSigma_\vw)}\big) \nn \\
		&\leq 
		\sqrt{\dimSt \numSys} \tau_{\vLtil}
		\big(
		\rho_{\vLtil}^t \norm{\vs_0}
		+
		\sum_{t'=1}^t \rho_{\vLtil}^{t-t'}  \norm{\vBtil_{t'} \vek(\vSigma_\vz)} 
        + 
		\sum_{t'=1}^t \rho_{\vLtil}^{t-t'}  \norm{\vPitil_{t'} \vek(\vSigma_\vw)}\big),\label{eq_tac_48}
	\end{align}
	where the second line follows from $\norm{\vLtil^t} \leq \tau_{\vLtil} \rho_{\vLtil}^t$. 
	
	Now, we evaluate $\norm{\vs_0}$, $\norm{\vBtil_{t'} \vek(\vSigma_\vz)}$, $\norm{\vPitil_{t'} \vek(\vSigma_\vw)}$ separately. For the first term, we have
	\begin{align}\label{eq_tac_17}
\norm{\vs_0} &= \mysqrt[1pt]{\sum_{i=1}^\numSys \norm{\vSigma_i(0)}_F^2} 
			= \mysqrt[1pt]{\sum_{i=1}^\numSys \vpi_0(i)^2 \norm{\expctn[\vx_0 \vx_0^\top]}_F^2}  \leq \norm{\expctn[\vx_0 \vx_0^\top]}_F \leq \expctn[\norm{\vx_0}^2].
\end{align}
	Let $[\vBtil_{t'}]_i$ denote the $i$th block of $\vBtil_{t'}$, i.e.,
	$[\vBtil_{t'}]_i = \sum_{j=1}^\numSys \vpi_{t-1}(j) \vT_{ji} (\vB_j \otimes \vB_j) $, then
	\begin{equation}\label{eq_tac_46}
		\begin{split}
			\norm{\vBtil_{{t'}} \vek(\vSigma_\vz)}
			&= \mysqrt[1pt]{\sum_{i=1}^\numSys \norm{[\vBtil_{{t'}} ]_i \vek(\vSigma_\vz)}^2 } \\
			&\leq \sum_{i=1}^\numSys \norm{[\vBtil_{{t'}} ]_i \vek(\vSigma_\vz)} \\
			&= \sum_{i=1}^\numSys \norm{\sum_{j=1}^\numSys \vpi_{{t'}-1}(j) \vT_{ji} (\vB_j \otimes \vB_j) \vek(\vSigma_\vz)} \\
			&= \sum_{i=1}^\numSys \norm{\sum_{j=1}^\numSys \vpi_{{t'}-1}(j) \vT_{ji} (\vB_j \vSigma_\vz \vB_j^\top)}_F \\
			&\leq \norm{\vB_{1:\numSys}}^2 \norm{\vSigma_\vz} \cdot \sum_{i=1}^\numSys \norm{\sum_{j=1}^\numSys \vpi_{{t'}-1}(j) \vT_{ji} \vI_\dimSt}_F \\
			&= \norm{\vB_{1:\numSys}}^2 \norm{\vSigma_\vz} \cdot \sum_{i=1}^\numSys \norm{\vpi_{t'}(i) \vI_\dimSt}_F \\
			&\leq \sqrt{\dimSt} \norm{\vB_{1:\numSys}}^2 \norm{\vSigma_\vz}.
		\end{split}
	\end{equation}
	Lastly, we have 
	\begin{equation}\label{eq_tac_47}
		\begin{split}
			&\norm{\vPitil_{{t'}}  \vek(\vSigma_\vw)} 
			= \mysqrt[1pt]{\sum_{i=1}^\numSys \norm{\vpi_{t'}(i) \vek(\vSigma_\vw)}^2} 
			 \leq \norm{\vek(\vSigma_\vw)} = \norm{\vSigma_\vw}_F = \sqrt{\dimSt} \norm{\vSigma_\vw}.
		\end{split}
	\end{equation}
	Plugging \eqref{eq_tac_17}--\eqref{eq_tac_47} into \eqref{eq_tac_48}, we obtain
	\begin{align*}
		\expctn[\norm{\vx_t}^2]  
		&\leq 
		\sqrt{\dimSt \numSys} \tau_{\vLtil}
		\big(
		\rho_{\vLtil}^t \expctn[\norm{\vx_0}^2]
		+
		\sqrt{\dimSt} \norm{\vB_{1:\numSys}}^2 \norm{\vSigma_\vz} 
        \sum_{t'=1}^t \rho_{\vLtil}^{t-t'}
		+
		\sqrt{\dimSt} \norm{\vSigma_\vw}
		\sum_{t'=1}^t \rho_{\vLtil}^{t-t'} \big), 
	\end{align*}
	which gives the bound for $\expctn[\norm{\vx_t}^2]$ in \eqref{eq_statebound}. To obtain the bound for $\norm{\vSigma(t)}_F$ in \eqref{eq_varbound}, note that $\norm{\vSigma(t)}_F = \norm{\sum_{i=1}^\numSys \vSigma_i(t)}_F \leq \sqrt{\numSys \sum_{i=1}^{s} \norm{\vSigma_i(t)}_F^2 } \leq \sqrt{s}\norm{\vs_t}$. We then follow a similar line of reasoning as above to get the statement of the lemma. This completes the proof.
\end{proof}

\subsection{Supporting Lemmas}
\label{sec:helper_lemmas}
In this section, we provide a list of lemmas that will be useful for the subsequent proofs.
\begin{lemma}\label{lemma_GaussianTailBound}
	Suppose $\vz \sim \N(0, \vSigma_\vz)$ with $\vSigma_\vz \in \R^{\dimInput \times \dimInput}$. For any $t \geq (3+ 2\sqrt{2}) \dimInput$, we have
	\begin{equation*}
		\P(\norm{\vz}^2 \geq 3 \norm{\vSigma_\vz} t) \leq e^{-t}.
	\end{equation*}
\end{lemma}
\parentchild{}{Lemma \ref{lemma_EnoughSample2}}{Zhe}{}
\begin{proof}
	From \citet[Proposition 1]{hsu2012tail}, we have for any $t>0$,
	\begin{equation*}
		\P(\norm{\vz}^2 \geq \tr(\vSigma_\vz) + 2 \sqrt{\tr(\vSigma_\vz^2) t} + 2 \norm{\vSigma_\vz} t) \leq e^{-t},
	\end{equation*}
	which implies
	\begin{equation*}
		\P(\norm{\vz}^2 \geq \dimInput \norm{\vSigma_\vz} + 2 \sqrt{\dimInput} \norm{\vSigma_\vz} \sqrt{t} + 2 \norm{\vSigma_\vz} t) \leq e^{-t}.
	\end{equation*}
	We can see that when $t \geq (3+2\sqrt{2}) \dimInput$, we have $\dimInput + 2 \sqrt{\dimInput} \sqrt{t} \leq t$, which implies $\dimInput \norm{\vSigma_\vz} + 2 \sqrt{\dimInput} \norm{\vSigma_\vz} \sqrt{t} \leq \norm{\vSigma_\vz} t$. Therefore, we have
	$
	\P(\norm{\vz}^2 \geq 3 \norm{\vSigma_\vz} t) \leq e^{-t}.
	$
\end{proof}

\section{Proofs of The Results on System Identification}\label{appendix_SysIDAnalysis}
In this Appendix, we discuss in detail the estimation of MJS dynamics $\vA_{1:\numSys}, \vB_{1:\numSys}$, as well as the Markov transition matrix $\vT$ from finite samples obtained from a single trajectory of~\eqref{switched LDS}.

\subsection{Identification of \texorpdfstring{$\vT$}{T} (Proof of Theorem~\ref{thrm:learn_dynamics})}\label{sec estimate T}
The following theorem adapted from \citet[Lemma 7]{zhang2018state} provides the sample complexity result for estimating Markov matrix $\vT$, which is a corresponds to the sample complexity on $\norm{\vThat - \vT}$ in Theorem \ref{thrm:learn_dynamics}.
\begin{theorem}\label{lemma_ConcentrationofMC}
	Suppose we have an ergodic Markov chain $\vT \in \R^{\numSys\times \numSys}$ with mixing time $t_{MC}$ and stationary distribution $\vpi_{\infty} \in \R^{\numSys}$. Let $\pi_{\max}:= \max_{i \in [\numSys]} \vpi_{\infty}(i)$ and $\pi_{\min}:= \min_{i \in [\numSys]} \vpi_{\infty}(i)$. Given a state sequence $\omega(0), \omega(1), \dots, \omega(T)$ of the Markov chain, define the empirical estimator $\vThat$ of the Markov matrix as follows,
	\begin{align*}
		[\hat{\vT}]_{ij} &= \frac{\sum_{t=1}^{ T - 1 }\indicator{\omega(t)=i,\omega(t+1)=j }}{\sum_{t=1}^{ T-1 }\indicator{\omega(t)=i}},
	\end{align*}
	Assume for some $\delta > 0$, $T \geq \underline{T}_{MC,1} (C_{MC}, \frac{\delta}{4}):=\parenthesesbig{68 C_{MC} \pi_{\max} \pi_{\min}^{-2}  \log(\frac{4\numSys}{\delta}) }^2$, where $C_{MC}$ is defined in Table \ref{table_Notation_TrajLen}.
Then, we have with probability at least  $1-\delta$,
	\begin{align}\label{eq_MCEstimation}
		&\norm{\hat{\vT} - \vT} \leq \frac{4 \norm{\vT}}{\pi_{\min}} 
        \mysqrt[1pt]{\frac{17 \pi_{\max} C_{MC} \log(T) \log({4 \numSys C_{MC}  \log(T)}/{\delta})}{T}}.
	\end{align}
\end{theorem}
\begin{proof}
	We first consider estimators computed using a sub-trajectory of $\omega(0), \omega(1), \dots, \omega(T)$, then combine them together to show the error bound for $\vThat$ in the claim. For $C_{MC}$ defined in Table \ref{table_Notation_TrajLen}, let $L = C_{MC} \log(T)$. Then, for $\ofst=0, 1, \dots, L-1$, define $\vThat^{(\ofst)} \in \R_+^{\numSys \times \numSys}$ such that $[\hat{\vT}^{(\ofst)}]_{ij} = \frac{\sum_{k=1}^{ \lfloor T/\subItvl \rfloor }\indicator{\omega(k\subItvl+\ofst)=i,\omega(k\subItvl+1+\ofst)=j }}{\sum_{k=1}^{ \lfloor T/\subItvl \rfloor }\indicator{\omega(k\subItvl+\ofst)=i}}$. In other words, $\hat{\vT}^{(\ofst)}$ is the estimator computed using data with sub-sampling period $L$. Following the proof of \citet[Lemma 7]{zhang2018state}, we know for any $\epsilon< \pi_{\min}/2$, suppose $\subItvl \geq 6t_{MC} \log(\epsilon^\inv)$.
	\begin{equation}
		\P \parenthesesbig{\norm{\hat{\vT}^{(\ofst)} - \vT} \leq 4 \pi_{\min}^{-1} \norm{\vT} \epsilon} 
		\geq 1-4 \numSys \exp\parenthesesbig{-\frac{T \epsilon^2}{17 \pi_{\max} \subItvl}}.
	\end{equation}
	By setting $\delta = 4 \numSys \exp\parenthesesbig{-\frac{T \epsilon^2}{17 \pi_{\max} \subItvl}}$, one can also interpret the above result as: for all $\delta > 0$, suppose 	
	\begin{equation}\label{eq_tac_28}
		\subItvl \geq 3 t_{MC} \log\left(\frac{T}{17 \pi_{\max} \subItvl \log(\frac{4 \numSys}{\delta})}\right),
	\end{equation}
	then when
	\begin{equation}\label{eq_tac_29}
		T \geq 68 \subItvl \pi_{\max} \pi_{\min}^{-2} \log(\frac{4\numSys}{\delta}),
	\end{equation}
	we have with probability at least  $1-\delta$
	\begin{equation}\label{eq_tac_27}
		\norm{\hat{\vT}^{(\ofst)} - \vT} \leq \frac{4  \norm{\vT}}{\pi_{\min}} \mysqrt[1pt]{\frac{17 \pi_{\max} C_{MC} \log(T) \log({4 \numSys}/{\delta})}{T}}.
	\end{equation}    
	One can verify \eqref{eq_tac_28} holds by plugging in $L = C_{MC} \log(T)$ and using definition $C_{MC}:= t_{\rm MC} \cdot \max \{3, 3-3 \log(\pi_{\max} \log(\numSys)) \}$; \eqref{eq_tac_29} holds under the condition $T \geq \underline{T}_{MC,1} (C_{MC}, \frac{\delta}{4}) := \parenthesesbig{68 C_{MC} \pi_{\max} \pi_{\min}^{-2}  \log(\frac{4\numSys}{\delta}) }^2$.
	
	Note that by definition, $\vThat$ can be viewed as a convex combination of $\vThat^{(\ofst)}$ for all $\ofst = 0, 1, \dots, L$, thus by triangle inequality and union bound, we have with probability $1 - L \delta$,
	\begin{equation}
		\norm{\hat{\vT} - \vT} \leq \frac{4 \norm{\vT}}{\pi_{\min}} \mysqrt[1pt]{\frac{17 \pi_{\max} C_{MC} \log(T) \log({4 \numSys}/{\delta})}{T}}.
	\end{equation}
	Finally, by replacing $L\delta$ with $\delta$, we could show \eqref{eq_MCEstimation} and conclude the proof.
\end{proof}

\subsection{Identification of \texorpdfstring{$\vA_{1:s}$}{A} and \texorpdfstring{$\vB_{1:s}$}{B} (Proof of Theorem~\ref{thrm:learn_dynamics})}
\label{appendix_SysIDAnalysis_AandB}

In this section, we estimate the unknown MJS dynamics $\Abs_{1:s}$ and $\Bbs_{1:s}$ from finite samples obtained from a single trajectory of~\eqref{switched LDS}. Given a stabilizing controller $\Kb_{1:s}$, under the input $\ub_t = \Kb_{\omega(t)}\xb_t + \zb_t$, the MJS state equation~\eqref{switched LDS} becomes,
	\begin{equation}\label{eqn:sys update 2}
		\begin{aligned}
			\xb_{t+1} &= \Lb_{\omega(t)} \xb_t + \Bbs_{\omega(t)}\zb_t + \wb_t, 
			\quad \text{s.t.}
			\quad  \omega(t)& \sim \text{Markov Chain}(\vT),
		\end{aligned}
	\end{equation}
	where $\Lb_{\omega(t)}:=\Abs_{\omega(t)} + \Bbs_{\omega(t)}\Kb_{\omega(t)}$ denotes the closed-loop state matrix, and $\{\zb_t\}_{t=0}^T \distas \Nc(0,\sigma_\vz^2\vI_p)$ is the i.i.d. excitation for exploration. To estimate the unknown system dynamics $(\vA_{1:s}, \vB_{1:s})$, we run the closed-loop MJS~\eqref{eqn:sys update 2} for $T$ time-steps and collect the trajectory $(\xb_t,\zb_t,\omega(t))_{t=0}^{T}$. Then, we run Algorithm~\ref{Alg_MJS-SYSID} on the collected trajectory to obtain the estimates $(\vAhat_{1:s},\vBhat_{1:s})$. To proceed, let $\hb_t := [\xb_t^\top/\sigma_\vw~\zb_t^\top/\sigma_\vz]^\top$ and $\bTetas_i:= [\sigma_\vw\Lb_{i}~\sigma_\vz\Bbs_i]$ for all $i \in [\numSys]$. Then the output of each sample in $\{(\xb_{t+1}, \xb_{t}, \zb_{t}, \omega(t))\}_{t \in S_i}$ can be related to the inputs as follows,
	\begin{align}
		\xb_{t_k+1} = \bTetas_{i}\hb_{t_k} + \wb_{t_k} \quad \text{for} \quad k = 1,2,\dots, |S_i|, \label{eqn:sys update compact} 
	\end{align}
where we set $S_i := \{t \bgl \omega(t) = i\} \equiv \{t_1, t_2, \cdots, t_{|S_i|}\}$.  This shows that, for each $i \in [s]$, the problem of estimating $(\vA_i, \vB_i)$ is equivalent to the problem of estimating $\bTetas_i$ from the sequence of covariate-response pairs $(\vh_{t_k}, \vx_{t_k + 1})_{k \geq 1}$. Specifically, following Algorithm~\ref{Alg_MJS-SYSID}, we solve a regression problem. For this purpose, we define the following concatenated matrices,
	\begin{align}
		\Yb_i = \begin{bmatrix}
			\xb_{t_1+1}^\top \\
			\xb_{t_2+1}^\top \\
			\vdots \\
			\xb_{t_{|S_i|}+1}^\top
		\end{bmatrix}, \quad
		\Hb_i = \begin{bmatrix}
			\vh_{t_1}^\top \\
			\vh_{t_2}^\top \\
			\vdots \\
			\vh_{t_{|S_i|}}^\top
		\end{bmatrix}, \quad
		\Wb_{i} = \begin{bmatrix}
			\vw_{t_1}^\top \\
			\vw_{t_2}^\top \\
			\vdots \\
			\vw_{t_{|S_i|}}^\top
		\end{bmatrix},\label{eqn:Xb_Hb}
	\end{align}
	that is, $\Yb_i$ has $\{\xb_{t+1}^\top\}_{t \in S_i}$ on its rows, $\Hb_i$ has $\{\hb_t^\top\}_{t \in S_i}$ on its rows and $\Wb_i$ has $\{\vw_t^\top\}_{t \in S_i}$ on its rows. Observe that, we have $\Yb_i = \Hb_{i} \bTeta_i^{\star \top} + \Wb_{i}$ and the regression problem in Algorithm~\ref{Alg_MJS-SYSID} becomes,
	\begin{align}
		\hat{\bTeta}_i^\top = \underset{\bTeta_i \in \R^{n \times (n+p)}}{\arg\min}\frac{1}{2|S_i|}\tf{\Yb_i - \Hb_i\bTeta_i^\top}^2.\label{eqn:least_squares}
	\end{align}
	When the problem is over-determined, the solution to the least-squares problem~\eqref{eqn:least_squares} is given by $\hat{\bTeta}_i^\top = \Hb_{i}^\dagger\Yb_i = (\Hb_{i}^\top\Hb_{i})^{-1}\Hb_{i}^\top \Yb_i$ and the associated estimation error is given by, $\hat{\bTeta}_i - \bTetas_i = \big((\Hb_{i}^\top\Hb_{i})^{-1}\Hb_{i}^\top \Wb_i\big)^\top$. This implies that the estimation error can be upper-bounded as follows,
	\begin{align}
		\norm{\hat{\bTeta}_i - \bTetas_i}  = \norm{(\Hb_i^\top\Hb_i)^{-1}\Hb_i^\top \Wb_i} \leq \frac{\norm{\Hb_i^\top \Wb_i}}{\lambda_{\min}\big(\Hb_i^\top\Hb_i\big)},\label{eqn:est_err}
	\end{align}
	To make the problem~\eqref{eqn:least_squares} well-conditioned, we also need a stability guarantee on the closed-loop MJS~\eqref{eqn:sys update 2}. This will make sure that the design matrix $\Hb_i$ has smaller condition number to help better estimation. Specifically, we will use the notion of mean-square stability introduced by Definition~\ref{def_mss} to achieve this.

At the core of our analysis is showing that the random process $\{\hb_t := [\xb_t^\top/\sigma_\vw~\zb_t^\top/\sigma_\vz]^\top\}_{t \in S_i}$ satisfies the martingale small-ball condition~(for each $i \in [s]$), which is defined as follows.
\begin{definition}[Martingale small-ball~\citep{simchowitz2018learning}]\label{def:BMSB} 
	Let $\{~\Fcal_t~\}_{t \geq 1}$ denotes a filtration and $\{Z_t\}_{t \geq 1}$ be an $\{\Fcal_t\}_{t \geq 1}$-adapted random process taking values in $\R$. We say $\{Z_t\}_{t \geq 1}$ satisfies the $(k,\nu,q)$-block martingale small-ball (BMSB) condition if, for any $j \geq 0$, one has $\frac{1}{k}\sum_{i = 1}^k \P\big(|Z_{j+i}| \geq \nu \bgl \Fcal_j\big) \geq q$ almost surely. Given a process $\{\xb_t\}_{t \geq 1}$ taking values in $\R^d$, we say it satisfies the $(k,\vGamma_{\rm sb},q)$-BMSB condition for $\vGamma_{\rm sb} \succ 0$ if, for any fixed $\vv \in \Scal^{d-1}$, the process $Z_t = \li \vv, \xb_t \ri$ satisfies $(k, \sqrt{\vv^\top\vGamma_{\rm sb} \vv},q)$-BMSB.
\end{definition}
To show that the covariate process $\{\hb_t := [\xb_t^\top/\sigma_\vw~\zb_t^\top/\sigma_\vz]^\top\}_{t \in S_i}$ satisfies BMSB condition, let $\Fc_t := \sigma (\xb_0,\dots,\xb_t, \vz_0,\dots,\vz_t,\wb_0,\dots,\wb_{t-1}$, $ \omega(1), \dots, \omega(t))$ denote the filtration generated by the states, the excitation and the noise processes, and the mode switching sequence when $t \geq 1$. Furthermore, let $\Fc_0 := \sigma(\xb_0,\zb_0, \omega(0))$. Then, $\xb_t, \zb_t$ and $\omega(t)$ become $\Fcal_{t}$-measurable and $\vw_t$ is $\Fcal_{t+1}$-measurable.
\begin{theorem}[BMSB condition for $\{\vh_t\}_{t \geq 1}$]\label{thrm:BMSB_BLS}
	Consider closed-loop MJS \eqref{eqn:sys update 2}. Suppose $\{\vz_t\}_{t=0}^{\infty} \distas  \N(0, \sigma_\vz^2 \vI_p)$ and $\{\vw_t\}_{t=0}^{\infty} \distas  \N(0, \sigma_\vw^2 \vI_n)$. Then, the covariate process $\{\hb_t = [\xb_t^\top/\sigma_\vw~\zb_t^\top/\sigma_\vz]^\top\}_{t \geq 1}$ satisfies the $(k,\vI_{n+p},q)$-martingale small-ball condition, with the constants $k = 1$ and $q=3/10$.
\end{theorem}
The theorem above uses martingale small-ball with $k=1$. We remark that using $k>1$ is expected to help capture the role of additional excitation terms in the BMSB lower bound, specifically, the dependence on $\vLtil$. However, this requires bounding higher order moments that involve cross-products of the input signal and noise terms and is left as future research.

Next, under the ergodicity of Markov chain~(Assumption~\ref{asmp_ergodicity_MarginMSS}), we establish a high probability lower bound on the cardinality of the set $S_i := \{t \bgl \omega(t) = i\} \equiv \{t_1, t_2, \cdots, t_{|S_i|}\}$. Our result is stated in the following theorem, which plays a critical role in establishing finite sample learning guarantees for the unknown MJS state and input matrices $\vA_{1:s}, \vB_{1:s}$.

\begin{theorem}[Lower bound on $|S_i|$]\label{lemma_EnoughSample1} 
	Let $\{\omega(t)\}_{t=0}^{\infty}$ be an ergodic Markov chain with the transition matrix $\Tb \in \R_+^{\numSys \times \numSys}$. 
Let $t_{\rm MC}(\epsilon)$ be as in Definition~\ref{def:mixing time}, and define $T_0 := t_{\rm MC}(\pi_{\min}/2)$. Let $S_i$ be as in Algorithm~\ref{Alg_MJS-SYSID}. Fix $\delta \in (0,1)$, such that $\sqrt{\frac{17 \pi_{\max}T_0\log(sT_0/\delta)}{T - 2T_0}} \leq {\pi_{\min}}/{2}$. Then, choosing $T \geq 2T_0$, we have
		\begin{align}
			&\P \bigg(\bigcap_{i=1}^{s} \biggl\{|S_i| \geq \frac{\pi_{\min}T}{4}\bigg(1 - \frac{1}{\pi_{\min}}\mysqrt[1pt]{\frac{17 \pi_{\max}T_0\log(\frac{sT_0}{\delta})}{T - 2T_0}}\bigg)\biggr\}\bigg)\nn \\
			 & \geq 1 - \delta. \label{eqn:enough_samples}
	\end{align}
\end{theorem}
This theorem states that, choosing $T\geq 2t_{\rm MC}(\pi_{\min}/2)$, an ergodic Markov chain is guaranteed to visit each mode $i \in [s]$, at least $\hat{\Ocal}(\pi_{\min}T)$ times. We remark that, our estimate is consistent with the asymptotic case when $T \to \infty$. Note that, the term $\sqrt{\frac{17 \pi_{\max}T_0\log(sT_0/\delta)}{T - 2T_0}}$ in \eqref{eqn:enough_samples} can be made arbitrary small by choosing sufficiently large trajectory length $T$. Finally, we combine Theorems~\ref{thrm:BMSB_BLS} and \ref{lemma_EnoughSample1} with Theorem 2.4 from \citet{simchowitz2018learning} to obtain our main result on single trajectory learning of $\vA_{1:s}, \vB_{1:s}$.

\begin{theorem}[Identification of MJS]~\label{thrm:learn_dynamics_complete}
		Fix $\delta \in (0,1)$, such that, 
		\begin{align}
			\varrho := \frac{1}{\pi_{\min}}\sqrt{\frac{17 \pi_{\max}T_0\log(2sT_0/\delta)}{T - 2T_0}} \leq \frac{1}{2}.  
		\end{align}
		Suppose we run Algorithm~\ref{Alg_MJS-SYSID} with the trajectory length $T$ satisfying the following lower bound,
		\begin{align}
			T &\gtrsim  \max\bigg\{2T_0, \frac{(n+p) + \log(6s\bar{\Gamma}/\delta) + \log(6s/\delta)}{\pi_{\min}(1 - \varrho)} \bigg\},
		\end{align}
		where $T_0 := t_{\rm MC}(\pi_{\min}/2)$ and $\bar{\Gamma} :=  \sqrt{\dimSt \numSys} \tau_{\vLtil}
		\big(\expctn[\norm{\vx_0}^2]/\sigma_{\vw}^2 + (\sigma_{\vz}^2/\sigma_{\vw}^2)\sqrt{\dimSt} \norm{\vB_{1:\numSys}}^2T + \sqrt{\dimSt}T\big) + p$. Suppose $\{\zb_t\}_{t=0}^T \distas \Nc(0,\sigma_\vz^2\vI_p)$, $\{\wb_t\}_{t=0}^T \distas \Nc(0,\sigma_\vw^2\vI_n)$. Let $C_K := \max_{i \in [s]}\|\vK\|$. Then, under Assumption~\ref{asmp_ergodicity_MarginMSS}, we have
		\begin{align*}
			&\P\bigg(\bigcap_{i=1}^{s}\bigg\{\|\hat{\vA}_i - \vA_i\| \lesssim \frac{(C_K\sigma_{\vw}+ \sigma_{\vz})}{\sigma_{\vz}} 
            \mysqrt[1pt]{\frac{(n+p) + \log(6s\bar{\Gamma}/\delta) + \log(6s/\delta)}{\pi_{\min}(1 - \varrho)T}}\bigg\}\bigg) \geq 1-\delta, \nn \\
			&\P\bigg(\bigcap_{i=1}^{s}\bigg\{\|\hat{\vB}_i - \vB_i\| \lesssim \frac{\sigma_{\vw}}{\sigma_{\vz}}
			\mysqrt[1pt]{\frac{(n+p) + \log(6s\bar{\Gamma}/\delta) + \log(6s/\delta)}{\pi_{\min}(1 - \varrho)T}}\bigg\}\bigg) \geq 1-\delta. 	 
		\end{align*}
\end{theorem}
Here, a few remarks are in place. First, the result appears to be convoluted however most of the dependencies are logarithmic (specifically, the dependency on the failure probability $\delta$ and $\log(T)$ terms). Besides these, the dominant term, when estimating $\vA_{1:s}$, $\vB_{1:s}$ reduces to
	\[
	\frac{(C_K\sigma_\vw + \sigma_\vz)}{\sigma_\vz}\sqrt{\frac{n+p}{\pi_{\min}T}} \quad \text{and} \quad \frac{\sigma_\vw}{\sigma_\vz}\sqrt{\frac{n+p}{\pi_{\min}T}},
	\]
	respectively, which is identical to our statement in Theorem \ref{thrm:learn_dynamics}. Note that the overall sample complexity grows as $T\gtrsim (n+p)/\pi_{\min}$. A degrees-of-freedom counting argument would show that the dependency of $T\gtrsim (n+p)/\pi_{\min}$ is optimal. The reason is that, each vector state equation we fit has $n$ scalar equations. The total degrees of freedom for each dynamics pair $(\Ab_i,\Bb_i)$ is $n\times (n+p)$. Additionally, for the least-frequent mode, in steady-state, we should observe $\pi_{\min}T$ equations. Putting these together, we would minimally need $n \times \pi_{\min}T\geq n\times (n+p)$, which means we need $T \geq (n+p)/\pi_{\min}$ samples to estimate the MJS dynamics $(\Ab_{1:s}, \Bb_{1:s})$. Note that, our sample complexity is not effected directly by the number of MJS modes $s$. However, $s$ indirectly effects sample complexity via $\pi_{\min}$, which is the probability of least-frequent mode in the steady state.

\subsection{Proofs of Intermediate Theorems and Lemmas} \label{sec:proof of intermediate results}
\subsubsection{Proof of Theorem \ref{thrm:BMSB_BLS}}
\begin{proof}
	To begin, we show that the process $\{\hb_t = [\xb_t^\top/\sigma_\vw~\zb_t^\top/\sigma_\vz]^\top\}_{t \geq 1}$ satisfies $(1,\vI_{n+p},q)$-BMSB condition, for some constant $q >0$. For this purpose, we need to show that, for any fixed $\vv \in \Scal^{n+p-1}$, the random process $\{Z_t\}_{t \geq 1} := \{\li\vv,\vh_t\ri\}_{t \geq 1}$ satisfies $(1,\norm{\vv},q)$-BMSB condition, that is, for any $j \geq 0$, we need to show that $\P(|Z_{j+1}| \geq \norm{\vv} \bgl \Fcal_j) \geq q$ almost surely. To proceed, for any $j \geq 0$, consider the concatenated state vector,
		\begin{align}
			\vh_{j+1} &= \begin{bmatrix} \xb_{j+1}/\sigma_{\vw}  \\ \zb_{j+1}/\sigma_{\vz}\end{bmatrix} = \begin{bmatrix} (\Lb_{\omega(j)} \xb_j + \Bbs_{\omega(j)}\zb_j + \wb_j)/\sigma_{\vw} \\ \zb_{j+1}/\sigma_{\vz} \end{bmatrix}. \label{eqn:vxtil_j+1} 
		\end{align}
		For any fixed $\vv \in \Scal^{n+p-1}$, let $\vv_1 \in \R^n$ and $\vv_2 \in \R^p$ such that $\vv = [\vv_1^\top~\vv_2^\top]^\top$. Combining this with \eqref{eqn:vxtil_j+1}, we get
\begin{equation}\label{eqn:Z_j+1}
			\begin{aligned}
				&Z_{j+1} := \li\vv,\vh_{j+1}\ri 
                =  \sigma_{\vw}^{-1}\li \vv_1, \Lb_{\omega(j)} \xb_j + \Bbs_{\omega(j)}\zb_j + \wb_j\ri + \sigma_{\vz}^{-1} \li \vv_2, \zb_{j+1} \ri.
			\end{aligned}
		\end{equation}
To proceed, let $\{\Fcal_t\}_{t \geq 1}$ denotes the filtration as defined before Theorem~\ref{thrm:BMSB_BLS}. Then, it is easy to see that $\sigma_{\vw}^{-1}\li \vv_1, \Lb_{\omega(j)} \xb_j + \Bbs_{\omega(j)}\zb_j + \wb_j\ri \bgl \Fcal_j \sim~$  $\Ncal(\sigma_{\vw}^{-1}\li \vv_1, \Lb_{\omega(j)} \xb_j + \Bbs_{\omega(j)}\zb_j\ri, \norm{\vv_1}^2)$. This is because $\xb_j, \zb_j$ and $\omega(j)$ are $\Fcal_{j}$-measurable, whereas, $\vw_j$ is $\Fcal_{j+1}$-measurable. Similarly, $\sigma_{\vz}^{-1} \li \vv_2, \zb_{j+1} \ri \bgl \Fcal_{j} \sim \Ncal(0, \norm{\vv_2}^2)$. Furthermore, since $\vw_j$ and $\vz_{j+1}$ are independent, $Z_{j+1} \bgl \Fcal_{j}$ has the following distribution: $Z_{j+1}  \bgl \Fcal_{j} \sim $
		\begin{align}
			\Ncal(\sigma_{\vw}^{-1}\li \vv_1, \Lb_{\omega(j)} \xb_j + \Bbs_{\omega(j)}\zb_j\ri, \norm{\vv_1}^2+ \norm{\vv_2}^2).
		\end{align}
		Therefore, integrating the probability density function of a standard Gaussian random variable, it can be shown that,
		\begin{align}
			\P\big(|\li \vv, \vh_{j+1}\ri| \geq \norm{\vv}  \bgl \Fcal_j\big) \geq 3/10, \label{eqn:Z_lower}
		\end{align} 
		where we obtain the above result by integrating the probability density function of a Gaussian random variable as follows,
		\begin{align}
			\forall \alpha \in \R,\;\; \P_{Z \sim \Ncal(0,\sigma^2)}(|\alpha + Z| \geq \sigma) &\geq \P_{Z \sim \Ncal(0,\sigma^2)}(|Z| \geq \sigma), \nn\\ 
			&= \P_{Z' \sim \Ncal(0,1)}(|Z'| \geq 1)
			= 1- \P_{Z' \sim \Ncal(0,1)}(|Z'| \leq 1), \nn \\
			&= 1 - 2 \int_{0}^{1}\frac{1}{\sqrt{2\pi}} e^{-z'^2/2}d z', \nn \\
			&\geq 1- 2 (7/20) = 3/10.
		\end{align}
		This verifies our claim that $\{\hb_t = [\xb_t^\top/\sigma_\vw~\zb_t^\top/\sigma_\vz]^\top\}_{t \geq 1}$ satisfies $(1,\vI_{n+p}, 3/10)$-BMSB condition. This completes the proof.
\end{proof}

\subsubsection{Proof of Theorem~\ref{lemma_EnoughSample1}}
\begin{proof} 
	From Definition~\ref{def:mixing time}, $t_{\rm MC}(\epsilon) := \min \curlybracketsbig{t \in \mathbb{N} : \max_{j \in [s]} \frac{1}{2} \norm{([\vT^t]_{j,:})^\top - \vpi_{\infty} }_1 \leq \epsilon}$, and $([\vT^{\subItvl}]_{i,:}) \onevec = \vpi_{\infty}^\top \onevec = 1$, for all $i \in [s]$. Therefore, choosing $L \geq t_{\rm MC}(\pi_{\min}/2)$, we have
		\begin{equation}\label{eq_tac_44}
			\max_{j \in [s]} \norm{([\vT^{\subItvl}]_{j,:})^\top - \vpi_{\infty} }_\infty \leq \frac{\pi_{\min}}{2}.
		\end{equation}
		To proceed, let $\Z^{+}:=\{1,2,3,\ldots\}$ denotes the set of positive integers. Then, to lower bound $|S_i|$ in Algorithm~\ref{Alg_MJS-SYSID}, we split the set $S_i := \{t~\big|~\omega(t)=i\}$ into $L \geq 1$ subsets via $S_i = \bigcup_{\ell = 0}^{L-1} S_i^{(\ell)}$, such that
		\begin{align}
			S_i^{(\ell)}:=\big\{t~\big|~\omega(t)=i,~(t-\ell)/L \in \Z^+ \rfloor\big\}, 
		\end{align}
		where $0\leq\ell\leq L-1$ is a fixed offset. Let $\{\Fcal_t\}_{t \geq 1}$ denotes the filtration as defined before Theorem~\ref{thrm:BMSB_BLS}. 
To ease the notation, we let $\tilde{\omega}(k):= \omega(\ell + k L)$, and $\tilde{\Fcal}_k :=  \Fcal_{\ell + k L}$, for all $k \in \Z^+$. Then, one can see that $\tilde{\omega}(k)$ is $\tilde{\Fcal}_k$-measurable. To proceed, define $\vdelta_k, \vDelta_k \in \R^{\numSys}$ such that 
		\begin{equation}
			\begin{aligned}
				\vdelta_k(i) &:= \indicator{\tilde{\omega}(k)=i} - \expctn[ \indicator{\tilde{\omega}(k)=i} \mid \tilde{\Fcal}_{k-1}], \\
				\vDelta_k(i) &:= \sum_{j=1}^{k} \vdelta_j(i).
			\end{aligned} \label{eqn:martingale_diff}
		\end{equation}
		Note that for all $i \in [\numSys]$, the random process $\{\vDelta_k(i)\}_{k \in \Z^+}$, adapted to the filtration $\{\tilde{\Fcal}_{k}\}_{k \in \Z^+}$, forms a martingale, that is, we have
		\begin{align*}
\expctn[\vDelta_{k+1}(i) \mid \tilde{\Fcal}_k]
				&= \expctn[\sum_{j=1}^{k+1} \vdelta_j(i)  \mid \tilde{\Fcal}_k] \\
				&= \sum_{j=1}^{k} \vdelta_j(i) + \expctn[ \indicator{\tilde{\omega}(k+1)=i} - \expctn[ \indicator{\tilde{\omega}(k+1)=i} \mid \tilde{\Fcal}_{k}]   \mid \tilde{\Fcal}_k] \\
				&= \sum_{j=1}^{k} \vdelta_j(i) = \vDelta_k(i).
\end{align*}
		Therefore, $\vdelta_k(i) = \vDelta_k(i) - \vDelta_{k-1}(i)$ can be viewed as the martingale difference sequence. Since $\expctn [\vdelta_k(i) \mid \tilde{\Fcal}_{k-1}] = 0$, we have $\expctn[\vdelta_k(i)^2 \mid \tilde{\Fcal}_{k-1}] = \text{Var}(\vdelta_k(i) \mid \tilde{\Fcal}_{k-1}) = \text{Var}(\indicator{\tilde{\omega}(k)=i}\mid  \tilde{\Fcal}_{k-1}) \leq \expctn[ \indicator{\tilde{\omega}(k)=i}^2 \mid  \tilde{\Fcal}_{k-1} ] \leq \expctn[ \indicator{\tilde{\omega}(k)=i} \mid  \tilde{\Fcal}_{k-1} ] = \P(\tilde{\omega}(k)=i \mid \tilde{\omega}(k-1)) = [\vT^{\subItvl}]_{\tilde{\omega}(k-1), i}$. When $L \geq t_{\rm MC}(\pi_{\min}/2)$, using \eqref{eq_tac_44}, we get $[\vT^{\subItvl}]_{\tilde{\omega}(k-1), i} \leq \vpi_\infty(i) + \max_{j \in [s]} \norm{([\vT^{\subItvl}]_{j,:})^\top - \vpi_{\infty} }_\infty$ $\leq 2 \pi_{\max}$. Therefore,
		\begin{equation}
			\sum_{k=1}^{\Ttil} \expctn[\vdelta_k(i)^2 \mid \tilde{\Fcal}_{k-1}] \leq 2 \pi_{\max} \Ttil,
		\end{equation}
		where we use the definition $\Ttil:= \lfloor\frac{T-\ell}{L} \rfloor$. Combining this with the observation that $|\vdelta_k(i)| < 1$, we have
		\begin{align}
\P
				\bigg(\bigg|\sum_{k=1}^{\tilde{T}}\indicator{\tilde{\omega}(k)=i} - \sum_{k=1}^{\tilde{T}} \expctn[ \indicator{\tilde{\omega}(k)=i}  \mid \tilde{\Fcal}_{k-1}]\bigg| \geq  \frac{\epsilon}{2}\Ttil\bigg)
				& \eqsym{i} \P(\big|\vDelta_{\tilde{T}}(i)\big| \geq \frac{\epsilon}{2}\Ttil), \nn \\
				&\leqsym{ii} \exp(- \frac{\Ttil \epsilon^2/8}{2 \pi_{\max} + \epsilon/6}),\nn \\
				&\leqsym{iii} \exp(- \frac{ \Ttil \epsilon^2}{17 \pi_{\max}} ), \label{eq_tac_30}
\end{align}
		where (i) follows from the definition of $\vDelta_{\Ttil}(i)$, (ii) follows from Freedman's inequality \citep{freedman1975tail}, and (iii) follows from picking $\epsilon \leq \pi_{\min}/2$. Moreover, when $L \geq t_{\rm MC}(\pi_{\min}/2)$, we also have
		\begin{equation}
			\begin{aligned}
				\bigg|\sum_{k=1}^{\tilde{T}} \expctn[ \indicator{\tilde{\omega}(k)=i}  \mid \tilde{\Fcal}_{k-1}] - \vpi_\infty(i) \Ttil \bigg| 
				&= \bigg| \sum_{k = 1}^{\tilde{T}} \P(\tilde{\omega}(k)=i \mid \tilde{\omega}(k-1))- \vpi_\infty(i) \Ttil \bigg |, \\
				& \leq \sum_{k = 1}^{\tilde{T}} \big| [\vT^{\subItvl}]_{\tilde{\omega}(k-1), i} - \vpi_\infty(i)\big|, \\
				& \leq \Ttil \max_{j \in [s]} \norm{([\vT^{\subItvl}]_{j,:})^\top - \vpi_{\infty} }_\infty, \\
				& \leq \frac{\pi_{\min}}{2} \Ttil.
			\end{aligned}\label{eqn:exp_lower}
		\end{equation}
		Combining \eqref{eqn:exp_lower} with \eqref{eq_tac_30}, and union bounding over $0 \leq \ell \leq L-1$, we obtain
		\begin{align}
			\P
			\bigg(\bigcap_{\ell=0}^{L-1} \biggl\{|S_i^{(\ell)}| \geq \vpi_\infty(i) \Ttil- \frac{\pi_{\min}}{2} \Ttil - \frac{\epsilon}{2} \Ttil\biggr\}\bigg)
			\geq 1 - \sum_{\ell = 0}^{L-1} \exp(- \frac{ \Ttil \epsilon^2}{17 \pi_{\max}} ). \label{eqn:lower_bound_prob_E1}
		\end{align}
		To proceed, define the events $\Ecal_1 := \bigcap_{\ell=0}^{L-1} \big\{|S_i^{(\ell)}| \geq (\pi_{\min}/2 - \epsilon/2)\Ttil\big\}$ and $\Ecal_2 := \big\{|S_i| \geq (\pi_{\min}/2 - \epsilon/2)(T-L)\big\}$. Note that $\Ecal_1 \subset \Ecal_2$ because, $|S_i| =  \sum_{\ell = 0}^{L-1}|S_i^{(\ell)}|$ and $\sum_{\ell = 0}^{L-1} \Ttil = \sum_{\ell = 0}^{L-1} \lfloor\frac{T-\ell}{L} \rfloor = T-L$. This implies that $\P(\Ecal_2) \geq \P(\Ecal_1)$. Combing this with \eqref{eqn:lower_bound_prob_E1}, and union bounding over all $i \in [s]$, we have 
		\begin{equation}
			\begin{aligned}
				\P \bigg(\bigcap_{i=1}^{s} \biggl\{|S_i| \geq (\pi_{\min}/2 - \epsilon/2)(T-L) \biggr\}\bigg)
				  &\geq 1 - sL\exp\big(- \frac{ (T/L-2)\epsilon^2}{17 \pi_{\max}}\big), \\
				\implies \P \bigg(\bigcap_{i=1}^{s} \biggl\{|S_i| \geq (\pi_{\min}/4 - \epsilon/4)T\biggr\}\bigg)
				  &\geqsym{i} 1 - sL\exp\big(- \frac{(T/L-2)\epsilon^2}{17 \pi_{\max}}\big), \label{eqn:lower_bound_prob_E2}
			\end{aligned}
		\end{equation}
		where (i) follows from choosing $T \geq 2L$. Finally, setting $\delta = sL\exp(- \frac{(T/L-2)\epsilon^2}{17 \pi_{\max}})$ and replacing $\epsilon$ with $\sqrt{\frac{17 \pi_{\max}L\log(sL/\delta)}{T - 2L}}$, we obtain the statement of the theorem,
		\begin{align*}
			\P \bigg(\bigcap_{i=1}^{s} \biggl\{|S_i| {\geq} \frac{\pi_{\min}T}{4}\bigg(1 {-} \frac{1}{\pi_{\min}}\sqrt{\frac{17 \pi_{\max}L\log(sL/\delta)}{T - 2L}}\bigg)\biggr\}\bigg)  
			\geq 1 - \delta.
		\end{align*}
		This completes the proof.	
\end{proof}

\subsubsection{Proof of Theorem~\ref{thrm:learn_dynamics_complete}}
\begin{proof}
	For the sake of completeness, before we present the proof of Theorem~\ref{thrm:learn_dynamics_complete}, we present a meta result from~\citet{simchowitz2018learning} which will be used to prove Theorem~\ref{thrm:learn_dynamics_complete}.
	\begin{theorem}[Meta-theorem~\citep{simchowitz2018learning}]\label{thrm:Theorem2.4_Simchowitz}
		Fix $\delta \in (0,1)$, $T \in \mathbb{N}$ and $0 \prec \vGamma_{\rm sb} \prec \bar{\Gamma}$. Then if $(\xb_t, \yb_t)_{t=1}^T \in (\R^d \times \R^n)^T$ is a random sequence such that (a) $\yb_t = \Ab_\star \xb_t + \vw_t$, where $\vw_t \bgl \Fcal_{t}$ is $\sigma_\vw^2$-subgaussian and mean zero, (b) $\xb_1,\dots,\xb_T$ satisfy the $(k,\vGamma_{\rm sb},q)$-small ball condition, and (c) such that $\P\big(\sum_{t = 1}^T \xb_t \xb_t^\top \npreceq T\bar{\vGamma}\big) \leq \delta$. Then if
			\begin{align}
				T \geq \frac{10k}{q^2}\big(\log(1/\delta) + 2 d\log(10/q) + \log(\det(\bar{\vGamma}\vGamma_{\rm sb}^{-1}))\big), \nn
			\end{align}
			we have
			\begin{align}
				&\P\bigg(\|\hat{\Ab} - \Ab_\star\| \geq \frac{90\sigma_\vw}{q} \mysqrt[1pt]{\frac{n + d\log(10/q) + \log(\det(\bar{\vGamma} \vGamma_{\rm sb}^{-1})) + \log(1/\delta)}{T \lambda_{\min}(\vGamma_{\rm sb})}}\bigg) \leq 3 \delta. \nn
		\end{align}
	\end{theorem}
	Our proof strategy is to verify that the conditions (a), (b), and (c) of Theorem~\ref{thrm:Theorem2.4_Simchowitz} hold for the MJS in~\eqref{eqn:sys update 2} and then apply Theorem~\ref{thrm:Theorem2.4_Simchowitz} to estimate $(\vA_{1:s}, \vB_{1:s})$. Before that, let $S_i$ be as defined in Algorithm~\ref{Alg_MJS-SYSID}, that is, $S_i :=\{t~\big|~\omega(t)=i\}$. Then, the samples $\{(\xb_{t+1}, \xb_{t}, \zb_{t}, \omega(t))\}_{t \in S_i}$ used to estimate $(\vA_i, \vB_i)$ are related as follows,
		\begin{align}
			\xb_{t_k+1} = \bTetas_{i}\hb_{t_k} + \wb_{t_k} \quad \text{for} \quad k = 1,2,\dots, |S_i|, \label{eqn:sys update compact proof} 
		\end{align}
		where we set $S_i := \{t \bgl \omega(t) = i\} \equiv \{t_1, t_2, \cdots, t_{|S_i|}\}$,  $\hb_{t_k} := [\xb_{t_k}^\top/\sigma_\vw~\zb_{t_k}^\top/\sigma_\vz]^\top$ and $\bTetas_i:= [\sigma_\vw\Lb_{i}~\sigma_\vz\Bbs_i]$. This shows that, for each $i \in [s]$, the problem of estimating $(\vA_i, \vB_i)$ is equivalent to the problem of estimating $\bTetas_i$ from the sequence of covariate-response pairs $(\vh_{t_k}, \vx_{t_k + 1})_{k \geq 1}$.  Moreover, let $\{\Fcal_t\}_{t \geq 1}$ denotes the filtration as defined before Theorem~\ref{thrm:BMSB_BLS}.
	
	{\emph{(a) Sub-Gaussian noise:}} Following re-parameterization in \eqref{eqn:sys update compact proof}, the covariate-response pairs $(\vh_{t_k}, \vx_{t_k + 1})_{k \geq 1}$ are generated from a linear response time series $\xb_{t_k+1} = \bTetas_{i}\hb_{t_k} + \wb_{t_k}$ for $k = 1,2,\dots, |S_i|$. Moreover, under the Assumption that $\{\vw_t\}_{t=0}^{T} \distas \Nc(0,\sigma_\vw^2\vI_n)$ and that $\vw_{t_k}$ is $\Fcal_{t_k+1}$-measureable, $\vw_{t_k} \bgl \Fcal_{t_k}\sim \Nc(0,\sigma_\vw^2\vI_n)$.
	
	{\emph{(b) BMSB condition:}} Theorem~\ref{thrm:BMSB_BLS} proves that the covariates process $\{\vh_{t_k}\}_{k = 1}^{|S_i|}$ satisfies $(k,\vI_{n+p},q)$-BMSB condition, with the constants $k = 1$ and $q=3/10$.
	
	{\emph{(c) Covariates correlation bound:}} Recalling the definition of $\vh_{t_k}$ and $\vH_i$ from \eqref{eqn:Xb_Hb}, we have
\begin{align*}
				\E[\|\vH_i^\top \vH_i\|] &= \E[\|\sum_{k=1}^{|S_i|}\vh_{t_k} \vh_{t_k}^\top\|], \\
				& \leq \sum_{k=1}^{|S_i|}\E[\|\vh_{t_k} \vh_{t_k}^\top\|] \leq \sum_{k=1}^{|S_i|} \E[\norm{\vh_{t_k}}^2], \\
				&= \sum_{k=1}^{|S_i|} (\E[\norm{\vx_{t_k}}^2]/\sigma_{\vw}^2 + \E[\norm{\vz_{t_k}}^2]/\sigma_{\vz}^2), \\
				&\leqsym{i} \sum_{k=1}^{|S_i|} \sqrt{\dimSt \numSys} \tau_{\vLtil}
				\big(\expctn[\norm{\vx_0}^2] 
				+ \sigma_{\vz}^2t_k\sqrt{\dimSt} \norm{\vB_{1:\numSys}}^2 + \sigma_{\vw}^2 t_k\sqrt{\dimSt}\big)/\sigma_{\vw}^2 + \sum_{k=1}^{|S_i|} p, \\
				&\leq \sqrt{\dimSt \numSys} \tau_{\vLtil}
				\big(\expctn[\norm{\vx_0}^2]/\sigma_{\vw}^2 
				+ (\sigma_{\vz}^2/\sigma_{\vw}^2)\sqrt{\dimSt} \norm{\vB_{1:\numSys}}^2T + \sqrt{\dimSt}T\big)|S_i| + p|S_i|
			\end{align*}
where we obtain (i) from combining Lemma~\ref{lemma_statebound} with Assumption~\ref{asmp_ergodicity_MarginMSS}~(which says $\rho(\vLtil) \leq 1$). Hence, setting
		\begin{align}
			\bar{\Gamma} :=  \sqrt{\dimSt \numSys} \tau_{\vLtil}
			\big(\frac{\expctn[\norm{\vx_0}^2]}{\sigma_{\vw}^2} + \frac{\sigma_{\vz}^2}{\sigma_{\vw}^2}\sqrt{\dimSt} \norm{\vB_{1:\numSys}}^2T + \sqrt{\dimSt}T\big) + p, \label{eqn:Gamma_bar}
		\end{align}
		we have, $\E[\|\sum_{k=1}^{|S_i|}\vh_{t_k} \vh_{t_k}^\top\|] = \E[\|\vH_i^\top \vH_i\|] \leq |S_i| \bar{\Gamma}$. Next, we use Markov inequality to show that
		\begin{equation}
			\begin{aligned}
				\P\big(\sum_{k=1}^{|S_i|}\vh_{t_k} \vh_{t_k}^\top \npreceq (|S_i|\bar{\Gamma}/\delta)\vI_{n+p}\big)
				 &= \P\big(\lambda_{\max}(\sum_{k=1}^{|S_i|}\vh_{t_k} \vh_{t_k}^\top) \geq |S_i|\bar{\Gamma}/\delta\big), \\
				&\leq \E\big[\lambda_{\max}(\sum_{k=1}^{|S_i|}\vh_{t_k} \vh_{t_k}^\top)\big]\delta/(|S_i|\bar{\Gamma}) \leq \delta.
			\end{aligned}
		\end{equation}
		We are now ready to use Theorem~2.4 from~\citet{simchowitz2018learning} to obtain our final result.
	
	{\emph{(d) Finalizing the proof:}} We use Theorem~\ref{thrm:Theorem2.4_Simchowitz}, with $\bar{\vGamma} =(\bar{\Gamma}/\delta)\vI_{n+p}$, $\vGamma_{\rm sb} = \vI_{n+p}$, $k = 1$, $q = 3/10$, and $d = n+p$ to upper bound the estimation error~\eqref{eqn:est_err} with high probability. Suppose the cardinality of the set $S_i=\{t \bgl \omega(t) = i\}$ satisfies,
		\begin{align}
			|S_i| &\gtrsim   (n+p) + \log(3s\bar{\Gamma}/\delta) + \log(3s/\delta), \label{eqn:lower_bound_Si}
		\end{align}
		for each $i \in [s]$. Then, using Theorem~\ref{thrm:Theorem2.4_Simchowitz}, we have
		\begin{align}
			&\P\bigg(\bigcap_{i=1}^{s}\bigg\{\|\hat{\bTeta}_i - \bTetas_i\| \lesssim \sigma_{\vw} 
			\mysqrt[1pt]{\frac{(n+p) + \log(3s\bar{\Gamma}/\delta) + \log(3s/\delta)}{|S_i|}}\bigg\}\bigg) \geq 1-\delta. \label{eqn:estimation_err_theta}
		\end{align}
		Combining \eqref{eqn:estimation_err_theta} with Theorem~\ref{lemma_EnoughSample1}, we fix $\delta \in (0,1)$, such that $\sqrt{\frac{17 \pi_{\max}T_0\log(sT_0/\delta)}{T - 2T_0}} \leq {\pi_{\min}}/{2}$, and choose the trajectory length $T$ satisfying
		\begin{align}
			T &\gtrsim  \max\bigg\{2T_0, \frac{(n+p) + \log(3s\bar{\Gamma}/\delta) + \log(3s/\delta)}{\pi_{\min}\big(1 - \frac{1}{\pi_{\min}}\sqrt{\frac{17 \pi_{\max}T_0\log(sT_0/\delta)}{T - 2T_0}}\big)} \bigg\}, \label{eqn:lower_bound_T}
		\end{align}
then, we have
		\begin{align}
			&\P\bigg(\bigcap_{i=1}^{s}\bigg\{\|\hat{\bTeta}_i - \bTetas_i\| \lesssim \sigma_{\vw} 
			\mysqrt[1pt]{\frac{(n+p) + \log(3s\bar{\Gamma}/\delta) + \log(3s/\delta)}{T\pi_{\min}\big(1 - \frac{1}{\pi_{\min}}\sqrt{\frac{17 \pi_{\max}T_0\log(sT_0/\delta)}{T - 2T_0}}\big)}}\bigg\}\bigg) \geq 1-2\delta. \label{eqn:estimation_err_theta_v2}
		\end{align}
		To proceed, using standard result from linear algebra that the spectral norm of a sub-matrix is upper bounded by the spectral norm of the original matrix, we have
		\begin{equation}
			\begin{aligned}
				\P\bigg(\bigcap_{i=1}^{s}\bigg\{\|\hat{\vA}_i - \vA_i\| \lesssim \frac{(C_K\sigma_{\vw}+ \sigma_{\vz})}{\sigma_{\vz}} 
				 \mysqrt[1pt]{\frac{(n+p) + \log(3s\bar{\Gamma}/\delta) + \log(3s/\delta)}{T\pi_{\min}\big(1 - \frac{1}{\pi_{\min}}\sqrt{\frac{17 \pi_{\max}T_0\log(sT_0/\delta)}{T - 2T_0}}\big)}}\bigg\}\bigg) &\geq 1-2\delta, \\
				\P\bigg(\bigcap_{i=1}^{s}\bigg\{\|\hat{\vB}_i - \vB_i\| \lesssim \frac{\sigma_{\vw}}{\sigma_{\vz}} 
				\mysqrt[1pt]{\frac{(n+p) + \log(3s\bar{\Gamma}/\delta) + \log(3s/\delta)}{T\pi_{\min}\big(1 - \frac{1}{\pi_{\min}}\sqrt{\frac{17 \pi_{\max}T_0\log(sT_0/\delta)}{T - 2T_0}}\big)}}\bigg\}\bigg) &\geq 1-2\delta, 	 
			\end{aligned}\label{eqn:estimation_err_AB}
		\end{equation}
		where we used the relation $\|\hat{\vA}_i - \vA_i\| \leq \|\hat{\vL}_i - \vL_i\|+ \|\hat{\vB}_i - \vB_i\| \|\vK_i\|$ and $\|\vK_i\| \leq C_K$ to upper bound the estimation error of the state matrices $\{\vA_i\}_{i=1}^s$. Finally, replacing $\delta$ with $\delta/2$, we get the statement of the theorem. This completes the proof.
\end{proof}

\section{MJS Regret Analysis} \label{appendix_MJSRegretAnalysis}
Consider MJS-LQR($\vA_{1:\numSys}, \vB_{1:\numSys}, \vT, \vQ_{1:\numSys}, \vR_{1:\numSys}$) with dynamics noise $\vw_t \sim \N(0,\vSigma_\vw)$, some arbitrary initial state $\vx_0$ and stabilizing controller $\vK_{1:\numSys}$. The input is $\vu_t = \vK_{\omega(t)} \vx_t + \vz_t$ where exploration noise $\vz_t \sim \N(0, \vSigma_{\vz})$. 
Let $\vL_i := \vA_i + \vB_i \vK_i$. Let $\vLtil \in \R^{\numSys \dimSt^2 \times \numSys \dimSt^2}$ denote the augmented closed-loop state matrix with $ij$-th $\dimSt^2 {\times} \dimSt^2$ block given by $\squarebrackets{\vLtil}_{ij} := [\vT]_{ji} \vL_j \otimes \vL_j$.
Let $\tau_{\vLtil}>0$ and  $\rho_{\vLtil} \in [0,1)$ be two constants such that $\norm{\vLtil^k} \leq \tau_{\vLtil} \rho_{\vLtil}^k$. By definition, one available choice for $\tau_{\vLtil}$ and $\rho_{\vLtil}$ are $\tau(\vLtil)$ and $\rho(\vLtil)$, respectively.

We define the following cumulative cost conditioned on the initial state $\vx_0$, initial mode $\omega(0)$, and controller $\vK_{1:\numSys}$,
\begin{align}\label{eq_tac_49}
	J_T(\vx_0, \omega(0), \curlybrackets{\vK_{1:\numSys}, \vSigma_\vz}) 
	&:= 
	\sum_{t=1}^T \expctn[\vx_t^\top \vQ_{\omega(t)} \vx_t + \vu_t^\top \vR_{\omega(t)} \vu_t \mid \vx_0, \omega(0), \vK_{1:\numSys}]. 
\end{align}
The definition of this cumulative cost coincides with the cost $J_{(q)} = \sum_{t=1}^{T_\epkidx} c_{T_0+ \cdots + T_{\epkidx-1} + t}$ in the definition of $\text{Regret}_\epkidx$ in \eqref{eq_regretSingleEpoch} with $\vx_0, \omega(0), \vK_{1:\numSys}$ setting to $\vx_0^{(\epkidx)}, \omega^{(\epkidx)}(0), \vK_{1:\numSys}^{(\epkidx)}$ since $\text{Regret}_\epkidx$ depends on randomness in $\Fcal_{\epkidx-1}$ only through $\vx_0^{(\epkidx)}, \omega^{(\epkidx)}(0), \vK_{1:\numSys}^{(\epkidx)}$. In the remainder of this appendix, for simplicity, we will drop the conditions $\vx_0, \omega(0), \vK_{1:\numSys}$ in the expectation and simply write $\expctn[\ \cdot \mid \vx_0, \omega(0), \vK_{1:\numSys}]$ as $\expctn[\cdot]$. Hence, for any measurable function $f$, $\expctn[f(\vx_0, \omega(0), \vK_{1:\numSys})] = f(\vx_0, \omega(0), \vK_{1:\numSys})$. Note that, even though the results in this appendix are derived for conditional expectation $\expctn[\ \cdot \mid \vx_0, \omega(0), \vK_{1:\numSys}]$, most of them also hold for the total expectation $\expctn[\cdot]$.

For the infinite-horizon case, we define the following infinite-horizon average cost without exploration noise $\vz_t$ and starting from $\vx_0 = 0$.
\begin{equation}
	J(0, \omega(0), \curlybrackets{\vK_{1:\numSys}}) := \limsup_{T \rightarrow \infty} \frac{1}{T} J_T(0, \omega(0), \curlybrackets{\vK_{1:\numSys}, 0}).
\end{equation}

To proceed, let $\vPstar_{1:\numSys}$ denotes the solution to the cDARE($\vA_{1:\numSys}, \vB_{1:\numSys}, \vT, \vQ_{1:\numSys}, \vR_{1:\numSys}$) defined in \eqref{eq_CARE}. Let $\vKstar_{1:\numSys}$ denotes the resulting infinite-horizon optimal controller computed using $\vPstar_{1:\numSys}$ and following \eqref{eq_optfeedback}. Note that, the infinite-horizon optimal average cost $J^\star$ in \eqref{eq_optinfhrzncost} is achieved if the optimal controller $\vKstar_{1:\numSys}$ is used, i.e.
\begin{equation}
	J^\star = J(0, \omega(0), \curlybrackets{\vKstar_{1:\numSys}}).
\end{equation}
Moreover, if the underlying Markov chain $\vT$ is ergodic, for any initial state $\vx_0$ and mode $\omega(0)$, $J^\star = J(\vx_0, \omega(0), \curlybrackets{\vKstar_{1:\numSys}})$. Let $\vLstar_i = \vA_i + \vB_i \vKstar_i$, for all $i \in [\numSys]$, denote the closed-loop state matrices when the optimal controller $\vK^\star_{1:\numSys}$ is used. Define the augmented state matrix $\vLtil^\star$ such that its $ij$-th block is given by 
$
\squarebrackets{\vLtil^\star}_{ij} := [\vT]_{ji} {\vL^\star_j} \otimes {\vL^\star_j}.
$
From \citet{costa2006discrete}, we know $\vKstar_{1:\numSys}$ stabilizes the MJS, thus $\rho^{\star}:=\rho(\vLtil^\star)<1$. 

Since $\text{Regret}_\epkidx$ defined in \eqref{eq_regretSingleEpoch} can be written as
\begin{equation}
	\text{Regret}_\epkidx = J_T(\vx_0^{(\epkidx)}, \omega^{(\epkidx)}(0), \curlybrackets{\vK^{(\epkidx)}_{1:\numSys}, \sigma_{\vz,\epkidx}^2 \vI_{\dimInput}}) - T J^\star,
\end{equation}
to evaluate $\text{Regret}(T)$, it suffices to evaluate the sub-optimality $J_T(\vx_0, \omega(0), \curlybrackets{\vK_{1:\numSys}, \vSigma_\vz}) - T J^\star$ for generic $\vx_0$, $\omega(0)$, $\vK_{1:\numSys}$, and $\vSigma_\vz$.
The outline of this Appendix \ref{appendix_MJSRegretAnalysis} is as follows.
\begin{itemize}
	\item Appendix \ref{appendix_MJSRegretAnalysis_InfHorizonPerturb} restates perturbation results \citep{du2021certainty} on $J(0, \omega(0), \curlybrackets{\vK_{1:\numSys}}) - J^\star$.
	\item  Appendix \ref{appendix_MJSRegretAnalysis_SingleEpoch} bounds $J_T(\vx_0, \omega(0), \curlybrackets{\vK_{1:\numSys}, \vSigma_\vz}) - T J(0, \omega(0), \curlybrackets{\vK_{1:\numSys}})$. Then, applying the results in Appendix \ref{appendix_MJSRegretAnalysis_InfHorizonPerturb}, for each epoch, we can bound the single epoch regret $J_T(\vx_0, \omega(0), \curlybrackets{\vK_{1:\numSys}, \vSigma_\vz}) - T J^\star$.
	\item In Appendix \ref{appendix_MJSRegretAnalysis_Stitching}, we stitch regrets for all epochs together, and combine them with the system identification results in Appendix \ref{appendix_SysIDAnalysis} to bound $\text{Regret}(T)$.
\end{itemize}

\subsection{MJS-LQR Perturbation Results}\label{appendix_MJSRegretAnalysis_InfHorizonPerturb}

We first present a lemma on the perturbation of augmented closed-loop state matrix if we use a controller $\vK_{1:\numSys}$ that is close to the optimal $\vKstar_{1:\numSys}$.
\begin{lemma}[Lemma 9 in \citet{du2021certainty} ]\label{lemma_spectralradiusperturbation}
	For an arbitrary controller $\vK_{1:\numSys}$, let $\vL_i = \vA_i + \vB_i \vK_i$, for all $i \in [\numSys]$, and let $\vLtil$ be the augmented state matrix such that its $ij$-th $n^2 \times n^2$ block is given by 
	$
	\squarebrackets{\vLtil}_{ij} := [\vT]_{ji} {\vL_j} \otimes {\vL_j}.
	$
	Assume $\norm{\vK_{1:\numSys} - \vKstar_{1:\numSys}} \leq \bar{\epsilon}_\vK$, where $\bar{\epsilon}_\vK$ is defined in Table \ref{table_Notation_Perturbation}. Then, we have
	\begin{align}
		\norm{\vLtil^k} &\leq \tau(\vLtil^\star) \parentheses{\frac{1+\rho^\star}{2}}^k,~~ \forall~ k \in \mathbb{N} \label{eq_tac_9},\\
\rho(\vLtil) &\leq \frac{1+\rho^\star}{2}. \label{eq_tac_11}
	\end{align}
	Thus controller $\vK_{1:\numSys}$ is stabilizing.
\end{lemma}
\parentchild{}{Proposition \ref{proposition_regretanalysispart2}, Proposition \ref{proposition_epochsysid}}{Zhe}{Zhe}
The following perturbation results from~\citet[Theorem 6 \& Lemma 11]{du2021certainty} show how much the infinite-horizon average cost deviates depending on the deviations from the optimal controller, and how much  the optimal controller deviates depending on the model accuracy for the MJS-LQR problem.
\begin{lemma}[MJS-LQR Perturbation~\citep{du2021certainty}] \label{lemma_tac_1}
	The infinite-horizon MJS-LQR($\vA_{1:\numSys}, \vB_{1:\numSys}, \vT, \vQ_{1:\numSys}, \vR_{1:\numSys}$) problems have the following perturbation results. Note that, notations $\bar{\epsilon}_\vK, \bar{\epsilon}_{\vA, \vB, \vT}$, and $C_{\vA, \vB, \vT}^\vK$ are defined in Table \ref{table_Notation_Perturbation}.
	
	\begin{enumerate}
		\item Suppose we have an arbitrary controller $\vK_{1:\numSys}$ such that $\norm{\vK_{1:\numSys} - \vKstar_{1:\numSys}} \leq \bar{\epsilon}_\vK$. Then, we have
		\begin{equation}\label{eq_tac_2}
			J(0, \omega(0), \curlybrackets{\vK_{1:\numSys}})
			- J^\star \leq C_\vK^J \norm{\vSigma_\vw} \norm{\vK_{1:\numSys} - \vKstar_{1:\numSys}}^2.
		\end{equation} \item Suppose there is an arbitrary MJS($\vAhat_{1:\numSys}, \vBhat_{1:\numSys}, \vThat$) such that, we have the following upper bounds hold: $\splitatcommas{\epsilon_{\vA, \vB} {:=} \max \curlybrackets{\norm{\vAhat_{1:\numSys} - \vA_{1:\numSys}}, \norm{\vBhat_{1:\numSys} - \vB_{1:\numSys}}}}$ $ \leq \bar{\epsilon}_{\vA, \vB, \vT}$, and $\epsilon_{\vT}:= \norm{\vThat - \vT}_{\infty} \leq \bar{\epsilon}_{\vA, \vB, \vT}$. Then, there exists an optimal controller $\vK_{1:\numSys}$ to the infinite-horizon MJS-LQR($\splitatcommas{\vAhat_{1:\numSys}, \vBhat_{1:\numSys}, \vThat, \vQ_{1:\numSys}, \vR_{1:\numSys}}$) and it can be computed using \eqref{eq_optfeedback} and \eqref{eq_CARE}, and we have
\begin{equation}\label{eq_tac_1}
			\norm{\vK_{1:\numSys} - \vKstar_{1:\numSys}} \leq C_{\vA, \vB, \vT}^\vK (\epsilon_{\vA, \vB} + \epsilon_{\vT}).
		\end{equation}
		By definition of $\bar{\epsilon}_{\vA,\vB, \vT}$, we see $\norm{\vK_{1:\numSys} - \vKstar_{1:\numSys}} \leq \bar{\epsilon}_{\vK}$, thus Lemma \ref{lemma_spectralradiusperturbation} is applicable.
	\end{enumerate}
\end{lemma}

\parentchild{}{Proposition \ref{proposition_regretanalysispart2}, Proposition \ref{proposition_epochEstErrGuarantee}, Proposition \ref{proposition_epochRegret}}{Zhe}{Zhe}

\subsection{Single Epoch Regret Analysis}\label{appendix_MJSRegretAnalysis_SingleEpoch}
Recall the definitions of $\vBtil_t$ and $\vPitil_t$ in \eqref{eq_sBtilPitil} of Appendix \ref{appendix_Preliminaries_MJSDynamicsMSS}. Furthermore, we define
\begin{equation}\label{eq_tac_18}
	\vPitil_{\infty} = \vpi_{\infty} \otimes \vI_{\dimSt^2}, \quad
	\vRtil_t = \sum_{i=1}^\numSys \vpi_t(i) \vR_i.
\end{equation}
For a set of matrices $\vV_{1:\numSys}$, define the following reshaping map, 
\begin{equation}\label{eqn:Hinv}
	\Hcal\bigg(
	\begin{bmatrix}
		\vV_1 \\ \vdots \\ \vV_\numSys
	\end{bmatrix}
	\bigg) 
	=
	\begin{bmatrix}
		\vek(\vV_1) \\ \vdots \\ \vek(\vV_\numSys)
	\end{bmatrix},
\end{equation}
and let $\Hcal^\inv$ denote the inverse mapping of $\Hcal$. Furthermore, let
\begin{equation}
	\vM_i := \vQ_i + \vK_i^\top \vR_i \vK_i, \qquad \vM := [\vM_1, \dots, \vM_\numSys].
\end{equation}
We define
\begin{equation}\label{eq_tac_12}
	\begin{split}
		N_{0,t} 
		&= \tr \bigg(\vM \Hcal^\inv\bigg(\vLtil^t \begin{bmatrix} \vek(\vSigma_1(0)) \\ \vdots \\ \vek(\vSigma_\numSys(0)) \end{bmatrix}\bigg)\bigg), \\
		N_{\vz,1,t} 
		&= \tr (\vM \Hcal^\inv( (\vBtil_t + \vLtil \vBtil_{t-1} + \cdots + \vLtil^{t-1} \vBtil_1) \vv_\vz )), \\
		N_{\vw,t} 
		&= \tr (\vM \Hcal^\inv( (\vPitil_t + \vLtil \vPitil_{t-1} + \cdots + \vLtil^{t-1} \vPitil_1) \vv_\vw )), \\
		N_{\vz,2,t}
		&= \tr(\vRtil_t \vSigma_z), \\ 
\text{where,} & \quad \vv_\vz := \vek(\vSigma_\vz) \; \text{and} \;  \vv_\vw := \vek(\vSigma_\vw)		
	\end{split}
\end{equation}
and
\begin{equation}\label{eq_tac_13}
	\begin{split}
	S_{0,T} &= \sum_{t=1}^T N_{0,t}, \quad
	S_{\vz,1,T} = \sum_{t=1}^T N_{\vz,1,t}, \\
	S_{\vw,T} &= \sum_{t=1}^T N_{\vw,t}, \quad
	S_{\vz,2,T} = \sum_{t=1}^T N_{\vz,2,t}.
	\end{split}
\end{equation}
First, we provide an exact expression for the cumulative cost. It will be used later to analyze the regret.
\begin{lemma}[Cumulative cost expression]\label{lemma_singlecost}
	For cost $J_T(\vx_0, \omega(0), \curlybrackets{\vK_{1:\numSys}, \vSigma_\vz})$ defined in \eqref{eq_tac_49}, we have
	\begin{equation}
		J_T(\vx_0, \omega(0), \curlybrackets{\vK_{1:\numSys}, \vSigma_\vz})
		=
		S_{0,T} + S_{\vz,1,T} + S_{\vz,2,T} + S_{\vw,T}.
	\end{equation}
\end{lemma}
\parentchild{}{Lemma \ref{lemma_statebound}, Lemma \ref{lemma_regretanalysispart1} }{Zhe}{Zhe,~Davoud}
\begin{proof}
	For the expected cost at time $t$, we have
	\begin{equation}\label{eq_tac_8}
		\begin{split}
			\expctn[\vx_t^\top \vQ_{\omega(t)} \vx_t + \vu_t^\top \vR_{\omega(t)} \vu_t] 
			&= \sum_{i=1}^\numSys \tr \big( \expctn[\vQ_{\omega(t)} \vx_t \vx_t^\top \indicator{\omega(t)=i}] 
			  + \expctn[\vR_{\omega(t)} \vu_t \vu_t^\top \indicator{\omega(t)=i}]\big) \\
			&=  \sum_{i=1}^\numSys \tr \parenthesesbig{
				(\vQ_i + \vK_i^\top \vR_i \vK_i) \vSigma_i(t) + \vpi_t(i) \vR_i \vSigma_\vz} \\
			&= \sum_{i=1}^\numSys \tr \parenthesesbig{
				\vM_i \vSigma_i(t) } + N_{\vz, 2, t},
		\end{split}
	\end{equation}
	where the second equality follows since $\vu_t = \vK_{\omega(t)} \vx_t + \vz_t$. Now plugging in the dynamics of $\vSigma_i(t)$ from Lemma \ref{lemma_covarianceDynamics}, we can conclude the proof.
\end{proof}

Before proceeding, we provide several properties of the operator $\tr(\vM \Hcal(\cdot))$ that shows up in \eqref{eq_tac_12} and \eqref{eq_tac_13}, which will be used later to evaluate $J_T(\vx_0, \omega(0), \curlybrackets{\vK_{1:\numSys}, \vSigma_\vz}) - T J(0, \omega(0), \curlybrackets{\vK_{1:\numSys}})$.
\begin{lemma}[Properties of cost building bricks]\label{lemma_propertiesoftrMH}
	Let $C_\vartheta {: = } \dimSt \sqrt{\numSys} \norm{\vM_{1:\numSys}} \norm{\vLtil^t} \norm{\vSigma_\vw}$. For any $t, t^{'} \in \mathbb{N}$, we have 
	\begin{enumerate}[label={(\textnormal{L\arabic*)}}]
		\item \label{enum_tac_2}  $
		\tr(\vM \Hcal^\inv (\vLtil^t \vv)) 
		\leq 
		\sqrt{\dimSt \numSys} \norm{\vM_{1:\numSys}} \norm{\vLtil^t} \norm{\vv}$, where $\vv := [ \vek(\vV_1)^\top,  \ldots,\vek(\vV_\numSys)^\top ]^\top$ for some $\vV_{1:\numSys}$ such that $\vV_i \succeq 0$ for all $i\in[\numSys]$,
		\item\label{enum_tac_3} $ \tr(\vM \Hcal^\inv( \vLtil^{t} \vBtil_{t'} \vek(\vSigma_\vz) ))
		\leq
		C_\vartheta  \frac{\norm{\vSigma_{\vz}}} {\norm{\vSigma_{\vw}}}\norm{\vB_{1:\numSys}}^2$,
		\item \label{enum_tac_1} $ \tr(\vM \Hcal^\inv( \vLtil^{t} \vPitil_{{t'}} \vek(\vSigma_\vw) )) \leq C_\vartheta$,
		\item \label{enum_tac_4} $
		|\tr(\vM \Hcal^\inv( \vLtil^{t}  (\vPitil_{{t'}} - \vPitil_{\infty}) \vek(\vSigma_\vw) )) | {\leq} \tau_{MC}  C_\vartheta \rho_{MC}^{t'}$ where $\tau_{MC}$ and $\rho_{MC}$ are given in Definition \ref{def:mixing time}, and $\vPitil_{\infty}$ is given in \eqref{eq_tac_18}
	\end{enumerate}
\end{lemma}
\parentchild{Lemma \ref{lemma_statebound}}{Lemma \ref{lemma_regretanalysispart1}}{Zhe}{Zhe,Davoud}
\begin{proof}
	Let $[\cdot]_i$ denotes the $i$th sub-block of an $\numSys \times 1$ block matrix. Let $\vek^\inv$ denotes the inverse mapping of $\vek$, i.e.,  $\vek^\inv([\vv_1^\top, \cdots, \vv_r^\top]^\top) = [\vv_1, \cdots, \vv_r]$ for a set of vectors $\{\vv_i\}_{i=1}^r$. It can be easily seen that for any set of matrices $\vA, \vB, \vC$ and $\vX$, we have $\vA \vX \vB = \vC$ if and only if $(\vB^\top \otimes\vA) \vek(\vX) = \vek(\vC)$. This together with the definitions of $\vBtil_t$, $\vPitil_t$ in \eqref{eq_sBtilPitil}, $\vPitil_\infty$, $\vRtil_t$ in \eqref{eq_tac_18}, and $\Hcal(\cdot)$ in \eqref{eqn:Hinv} yields the following preliminary results
	\begin{subequations}\begin{align}&[\Hcal^\inv(\vLtil^t \vv)]_i \succeq 0, \label{eq_tac_14}\\
			&\vek^\inv([\vBtil_{t'} \vek(\vSigma_\vz)]_i) \succeq 0, \label{eq_tac_15}\\
			&\vek^\inv ([ \vPitil_{{t'}}  \vek(\vSigma_\vw)]_i) \succeq 0, \label{eq_tac_23}\\
			&\vek^\inv ([ | \vPitil_{{t'}} - \vPitil_{\infty}| \vek(\vSigma_\vw)]_i)  \succeq 0, \label{eq_tac_16}\\
			&|\tr(\vM \Hcal^\inv( \vLtil^{t}  (\vPitil_{{t'}} - \vPitil_{\infty}) \vek(\vSigma_\vw) )) | 
			\leq 
			\ \tr(\vM \Hcal^\inv( \vLtil^{t}  |\vPitil_{{t'}} - \vPitil_{\infty}| \vek(\vSigma_\vw) )),\label{eq_tac_19}
		\end{align}
	\end{subequations}
	here $|\cdot|$ denotes the element-wise absolute value of a matrix.
	Now, let us consider \ref{enum_tac_2}. We observe that 
	\begin{equation}\label{eqn:bhinv0}
		\begin{split}
			\tr(\vM \Hcal^\inv (\vLtil^t \vv))
			= \tr(\sum_{i=1}^\numSys \vM_i [\Hcal^\inv (\vLtil^t \vv)]_i) &\leq \norm{\vM_{1:s}} \cdot \tr(\sum_{i=1}^\numSys [\Hcal^\inv (\vLtil^t \vv)]_i) \\
			& \leq  \sqrt{\dimSt} \norm{\vM_{1:s}}  \norm{\sum_{i=1}^\numSys [\Hcal^\inv (\vLtil^t \vv)]_i}_F,
		\end{split}
	\end{equation}
	where the first inequality uses \eqref{eq_tac_14} and the definition that $\norm{\vM_{1:\numSys}}= \max_{i \in [\numSys]} \norm{\vM_i}$; and the last inequality follows from Cauchy-Schwarz inequality and the fact that $[\Hcal^\inv (\vLtil^t \vv)]_i \in \R^{\dimSt \times \dimSt}$. 
	Now, for the last term on the RHS of \eqref{eqn:bhinv0}, we have 
	\begin{equation}\label{eqn:bhinv1}
		\begin{split}
			\norm{\sum_{i=1}^\numSys [\Hcal^\inv (\vLtil^t \vv)]_i}_F & \leq \sum_{i=1}^\numSys \norm{ [\Hcal^\inv (\vLtil^t \vv)]_i}_F
			 \leq  \sqrt{\numSys} \mysqrt[1pt]{\sum_{i=1}^\numSys \norm{[\Hcal^\inv (\vLtil^t \vv)]_i}_F^2}, \nn \\
			&=\sqrt{\numSys} \norm{\Hcal^\inv (\vLtil^t \vv)}_F 
			= \sqrt{\numSys} \norm{\vLtil^t \vv}
			\leq \sqrt{\numSys} \norm{\vLtil^t} \norm{\vv},
		\end{split}    
	\end{equation}
	where the second equality holds since $\Hcal^\inv$ is a reshaping operator, and $\vLtil^t \vv$ is a vector. Substituting \eqref{eqn:bhinv1} into \eqref{eqn:bhinv0} gives \ref{enum_tac_2}.
	
	To show \ref{enum_tac_3}, we combine \eqref{eq_tac_15} with \ref{enum_tac_2} to get $\tr(\vM \Hcal^\inv( \vLtil^{t} \vBtil_{t'} \vek(\vSigma_\vz) )) \leq \sqrt{\dimSt \numSys} \norm{\vM_{1:\numSys}} \norm{\vLtil^t}~$ $ \norm{\vBtil_{{t'}} \vek(\vSigma_\vz)}$. Then, using the upper bound for $\norm{\vBtil_{{t'}} \vek(\vSigma_\vz)}$ derived in \eqref{eq_tac_46} completes proof of \ref{enum_tac_3}.
	
	To establish \ref{enum_tac_1}, we combine \eqref{eq_tac_23} with \ref{enum_tac_2} to obtain 
	\begin{equation}\label{eqn:l3b}
		\begin{split}
		&\tr(\vM \Hcal^\inv( \vLtil^{t} \vPitil_{{t'}} \vek(\vSigma_\vw) )) 
		\leq \sqrt{\dimSt \numSys} \norm{\vM_{1:\numSys}} \norm{\vLtil^t} \norm{\vPitil_{{t'}} \vek(\vSigma_\vw)}. 
		\end{split}   
	\end{equation}
	Then, using the upper bound for $\norm{\vPitil_{{t'}} \vek(\vSigma_\vw)}$ derived in \eqref{eq_tac_47} gives \ref{enum_tac_1}.
	
	Finally, let us consider \ref{enum_tac_4}. It follows from \eqref{eq_tac_16} and \eqref{eq_tac_19} in conjunction with \ref{enum_tac_2} that 
	\begin{equation}\label{eqn:l4b}
		\begin{split}
		&|\tr(\vM \Hcal^\inv( \vLtil^{t} | \vPitil_{{t'}} - \vPitil_{\infty}| \vek(\vSigma_\vw) ))| 
		 \leq \sqrt{\dimSt \numSys} \norm{\vM_{1:\numSys}} \norm{\vLtil^t} \norm{| \vPitil_{{t'}} - \vPitil_{\infty}| \vek(\vSigma_\vw)}.
		\end{split}
	\end{equation}
	Now, using \eqref{eq_tac_18}, we obtain 
	\begin{equation*}
		\begin{split}
			\norm{| \vPitil_{{t'}} {-} \vPitil_{\infty}| \vek(\vSigma_\vw)} 
			& = \mysqrt[1pt]{\sum_{i=1}^\numSys \norm{| [\vPitil_{{t'}}]_i - [\vPitil_{\infty}]_i| \vek(\vSigma_\vw)}^2}\\ 
			&= \mysqrt[1pt]{\sum_{i=1}^\numSys \norm{| \vpi_{t'}(i) - \vpi_{\infty}(i)| \vek(\vSigma_\vw)}^2} \\
			&= \norm{\vpi_{t'} - \vpi_\infty} \norm{\vek(\vSigma_\vw)} \\
			&\leq \norm{\vpi_{t'} - \vpi_\infty}_1 \norm{\vSigma_\vw}_F \\
			&\leq  \tau_{MC} \sqrt{\dimSt} \norm{\vSigma_\vw}\rho_{MC}^{t'},
		\end{split}
	\end{equation*}
	where the last inequality follows from Definition \ref{def:mixing time}. Substituting the above inequality in \eqref{eqn:l4b} completes the proof of \ref{enum_tac_4}.
\end{proof}

The following lemma provides a bound on $J_T(\vx_0, \omega(0), \curlybrackets{\vK_{1:\numSys}, \vSigma_\vz}) - T J(0, \omega(0), \curlybrackets{\vK_{1:\numSys}})$ using an arbitrary stabilizing controller $\vK_{1:\numSys}$. Based on this result, we will provide in Proposition \ref{proposition_regretanalysispart2} a uniform upper bound on this difference when using any controller $\vK_{1:\numSys}$ that is close to $\vK^\star_{1:\numSys}$.
\begin{lemma}\label{lemma_regretanalysispart1}
	For an arbitrary stabilizing controller $\vK_{1:\numSys}$, we have
	\begin{equation} \nn
		\begin{split}
			&J_T(\vx_0, \omega(0), \curlybrackets{\vK_{1:\numSys}, \vSigma_\vz}) - T J(0, \omega(0), \curlybrackets{\vK_{1:\numSys}}), \\
			&\leq \sqrt{\dimSt \numSys} \norm{\vM_{1:\numSys}} \cdot \norm{\vx_0}^2 + \frac{\dimSt \sqrt{\numSys} \tau_{\vLtil}}{1- \rho_{\vLtil}} \norm{\vM_{1:\numSys}} \norm{\vB_{1:\numSys}}^2 \norm{\vSigma_{\vz}}  T  \\
			& + \dimSt \norm{\vR_{1:\numSys}}  \norm{\vSigma_\vz}  T + \dimSt \sqrt{\numSys} \tau_{MC} \tau_{\vLtil} \norm{\vM_{1:\numSys}} \norm{\vSigma_\vw} \frac{\rho_{MC}}{\rho_{MC} - \rho_{\vLtil}} (\frac{\rho_{MC}}{1-\rho_{MC}} - \frac{\rho_{\vLtil}}{1-\rho_{\vLtil}}),
		\end{split}
	\end{equation}
	where $\tau_{MC}$ and $\rho_{MC}$ are given in Definition \ref{def:mixing time}, $\tau_{\vLtil}$ and $\rho_{\vLtil}$ are constants defined at the beginning of Appendix \ref{appendix_MJSRegretAnalysis}, and $\vM=[\vM_1, \dots, \vM_\numSys]$ with $\vM_i = \vQ_i + \vK_i^\top \vR_i \vK_i$.
\end{lemma}

\parentchild{Lemma \ref{lemma_singlecost}, Lemma \ref{lemma_propertiesoftrMH}, Lemma \ref{lemma_statebound}}{Proposition \ref{proposition_regretanalysispart2}}{Zhe}{Zhe}

\begin{proof}
	From Lemma \ref{lemma_singlecost}, we know that
	\begin{gather*}
		J_T(\vx_0, \omega(0), \curlybrackets{\vK_{1:\numSys}, \vSigma_\vz}) {=} S_{0,T} + S_{\vz,1,T} + S_{\vz,2,T} + S_{\vw,T}, 
		\\
		J(0, \omega(0), \curlybrackets{\vK_{1:\numSys}}) {=} \limsup_{T\rightarrow \infty} \frac{1}{T} (S_{0,T} + S_{\vw,T}) =: S_0 + S_{\vw}.
	\end{gather*}
	$S_0:=\limsup_{T\rightarrow \infty} \frac{1}{T} S_{0,T}$ and $S_{\vw}:=\limsup_{T\rightarrow \infty} \frac{1}{T} S_{\vw,T}$. Next, we will evaluate each term on the RHSs separately. 
	
	For $S_{0,T}$, letting $\vs_0 = \begin{bmatrix} \vek(\vSigma_1(0)) \\ \vdots \\ \vek(\vSigma_\numSys(0)) \end{bmatrix}$, we have \begin{equation*}
		\begin{split}
			S_{0,T}= \sum_{t=1}^T \tr (\vM \Hcal^\inv(\vLtil^t \vs_0))  \leq  \sqrt{\dimSt \numSys} \norm{\vM_{1:\numSys}} \norm{\vLtil^t} \norm{\vs_0} 
			&\leq \sqrt{\dimSt \numSys} \norm{\vM_{1:\numSys}} \cdot \expctn[\norm{\vx_0}^2] \\
			&= \sqrt{\dimSt \numSys} \norm{\vM_{1:\numSys}} \cdot \norm{\vx_0}^2,
		\end{split}
	\end{equation*}
	where the second line follows from Item \ref{enum_tac_2} in Lemma \ref{lemma_propertiesoftrMH}; the third line follows from \eqref{eq_tac_17} in Lemma \ref{lemma_statebound}. And from the discussion at the beginning of Appendix \ref{appendix_MJSRegretAnalysis}, we can get rid of $\expctn[\cdot]$. Then it is easy to see
	$
	S_0 = 0,
	$
	as long as $\norm{\vx_0}^2$ is bounded.
	
	For $S_{\vz,1,T}$, we have
	\begin{equation}
		\begin{split}
			S_{\vz,1,T}
			&= \sum_{t=1}^T \sum_{{t'}=0}^{t-1} \tr (\vM \Hcal^\inv( \vLtil^{t'} \vBtil_{t-{t'}} \vek(\vSigma_\vz) ))\\
			&\leq \dimSt \sqrt{\numSys} \norm{\vM_{1:\numSys}} \norm{\vB_{1:\numSys}}^2 \norm{\vSigma_{\vz}} (\sum_{t=1}^T \sum_{{t'}=0}^{t-1} \norm{\vLtil^{t'}})\\
			&\leq \frac{\dimSt \sqrt{\numSys} \tau_{\vLtil}}{1- \rho_{\vLtil}} \norm{\vM_{1:\numSys}} \norm{\vB_{1:\numSys}}^2 \norm{\vSigma_{\vz}}  T,
		\end{split}
	\end{equation}
	where the first inequality follows from Item \ref{enum_tac_3} in Lemma \ref{lemma_propertiesoftrMH}, and the second inequality follows from the fact $\norm{\vLtil^{t'}} \leq \tau_{\vLtil} \rho_{\vLtil}^{t'}$.
	
	For $S_{\vz,2,T}$, we have
	\begin{equation}
		S_{\vz,2,T} = \sum_{t=1}^T \tr(\sum_{i=1}^\numSys \vpi_t(i) \vR_i \vSigma_\vz) \leq \dimSt \norm{\vR_{1:\numSys}}  \norm{\vSigma_\vz}T.
	\end{equation}
	
	For $S_{\vw, T}$, we have
	\begin{equation}
		S_{\vw, T} = \sum_{t=1}^T \sum_{{t'}=0}^{t-1} \tr (\vM \Hcal^\inv( \vLtil^{t'} \vPitil_{t-{t'}} \vek(\vSigma_\vw) )).
	\end{equation}
	
	To evaluate it, we first define the following terms:
	\begin{align}
		S_{\vw, T}^{(\infty)} &:= \sum_{t=1}^T \sum_{{t'}=0}^{t-1} \tr (\vM \Hcal^\inv( \vLtil^{t'} \vPitil_\infty \vek(\vSigma_\vw) )), \\
		S_{\vw}^{(\infty)} &:= \limsup_{T \rightarrow \infty} \frac{1}{T} S_{\vw, T}^{(\infty)},
	\end{align}
	where $\vPitil_\infty$ is defined in \eqref{eq_tac_18}. Note that $S_{\vw, T}^{(\infty)}$ and $S_{\vw}^{(\infty)}$ are the counterparts of $S_{\vw, T}$ and $S_{\vw}$ except that the initial mode distribution $\vpi_0$ is the stationary distribution $\vpi_\infty$. Then, we have
	\begin{equation}\label{eq_tac_22}
		\begin{split}
			|S_{\vw, T} - S_{\vw, T}^{(\infty)}| 
			&= \left| \sum_{t=1}^T \sum_{{t'}=0}^{t-1} \tr (\vM \Hcal^\inv( \vLtil^{t'} (\vPitil_{t-{t'}}- \vPitil_\infty) \vek(\vSigma_\vw) )) \right| \\
			&\leq \tau_{MC} \dimSt \sqrt{\numSys} \norm{\vM_{1:\numSys}} \norm{\vSigma_\vw}  (\sum_{t=1}^T \sum_{{t'}=0}^{t-1} \norm{\vLtil^{t'}}  \rho_{MC}^{t-{t'}} ) \\
			&\leq \tau_{MC} \dimSt \sqrt{\numSys} \norm{\vM_{1:\numSys}} \norm{\vSigma_\vw}  (\sum_{t=1}^\infty \sum_{{t'}=0}^{t-1} \tau_{\vLtil} \rho_{\vLtil}^{t'} \rho_{MC}^{t-{t'}} ) \\
			&\leq \dimSt \sqrt{\numSys} \tau_{MC} \tau_{\vLtil}  \norm{\vM_{1:\numSys}} \norm{\vSigma_\vw}  \frac{\rho_{MC}}{\rho_{MC} - \rho_{\vLtil}} (\frac{\rho_{MC}}{1-\rho_{MC}} - \frac{\rho_{\vLtil}}{1-\rho_{\vLtil}})
		\end{split}
	\end{equation}
	where the first inequality follows from Item \ref{enum_tac_4} in Lemma \ref{lemma_propertiesoftrMH}. Thus,
	\begin{equation}\label{eq_tac_21}
		\begin{split}
			S_{\vw} =\limsup_{T\rightarrow \infty} \frac{1}{T} S_{\vw,T}&= \limsup_{T\rightarrow \infty} \frac{1}{T}  (S_{\vw,T} - S_{\vw, T}^{(\infty)})  
			+ \limsup_{T\rightarrow \infty} \frac{1}{T} S_{\vw, T}^{(\infty)} = S_{\vw}^{(\infty)}.
		\end{split}    
	\end{equation}
	Since $\sum_{t=1}^T \sum_{{t'}=0}^{t-1} \vLtil^{t'} = (\vI - \vLtil)^\inv T - (\vI - \vLtil)^{-2} \vLtil (\vI - \vLtil^T)$ and $\sum_{t'=0}^\infty \vLtil^{t'} = (\vI - \vLtil)^\inv$ we have $S_{\vw} = S_{\vw}^{(\infty)}$
	\begin{equation}\nn
		\begin{split}
			S_{\vw}^{(\infty)}
			&=\tr (\vM \Hcal^\inv( \limsup_{T\rightarrow \infty} \frac{1}{T} \sum_{t=1}^T \sum_{{t'}=0}^{t-1} \vLtil^{t'} \vPitil_\infty \vek(\vSigma_\vw) )) \\
			&= \tr (\vM \Hcal^\inv( (\vI - \vLtil)^\inv \vPitil_\infty \vek(\vSigma_\vw) )) 
			= \sum_{t'=0}^\infty \tr (\vM \Hcal^\inv( \vLtil^{t'} \vPitil_\infty \vek(\vSigma_\vw) )).
		\end{split}
	\end{equation}
	Thus,
	\begin{equation}\label{eq_tac_20}
		\begin{split}
			T S_{\vw} {=} T S_{\vw}^{(\infty)}
			&{=} \sum_{t=1}^T \sum_{t'=0}^\infty \tr (\vM \Hcal^\inv( \vLtil^{t'} \vPitil_\infty \vek(\vSigma_\vw) )) \\
			&\geq \sum_{t=1}^T \sum_{{t'}=0}^{t-1} \tr (\vM \Hcal^\inv( \vLtil^{t'} \vPitil_\infty \vek(\vSigma_\vw) )) 
			 = S_{\vw, T}^{(\infty)}
		\end{split}
	\end{equation}
	where the inequality holds since each trace summand is non-negative. Therefore,
	\begin{equation}\nn
		\begin{split}
			S_{\vw, T} 
			\leq &S_{\vw, T}^{(\infty)} + |S_{\vw, T} - S_{\vw, T}^{(\infty)}| \\
\overset{\eqref{eq_tac_21}}{\leq} &T S_{\vw} + |S_{\vw, T} - S_{\vw, T}^{(\infty)}| \\
			\overset{\eqref{eq_tac_22}}{\leq} &T S_{\vw} + \dimSt \sqrt{\numSys} \tau_{MC} \tau_{\vLtil}  \norm{\vM_{1:\numSys}} \norm{\vSigma_\vw}  
			\frac{\rho_{MC}}{\rho_{MC} - \rho_{\vLtil}} (\frac{\rho_{MC}}{1-\rho_{MC}} - \frac{\rho_{\vLtil}}{1-\rho_{\vLtil}}).
		\end{split}
	\end{equation}
	Finally, combining all the results we have so far, we obtain
	\begin{equation}\nn
		\begin{split}
			&J_T(\vx_0, \omega(0), \curlybrackets{\vK_{1:\numSys}, \vSigma_\vz}) - T J(0, \omega(0), \curlybrackets{\vK_{1:\numSys}}) \\
			&= S_{0,T} + S_{\vz,1,T} + S_{\vz,2,T} + S_{\vw,T} - T (S_0 + S_{\vw}) \\
			&\leq \sqrt{\dimSt \numSys} \norm{\vM_{1:\numSys}} \cdot \norm{\vx_0}^2 + \frac{\dimSt \sqrt{\numSys} \tau_{\vLtil}}{1- \rho_{\vLtil}} \norm{\vM_{1:\numSys}} \norm{\vB_{1:\numSys}}^2 \norm{\vSigma_{\vz}}  T \\ 
			&+ \dimSt \norm{\vR_{1:\numSys}}  \norm{\vSigma_\vz}  T + \dimSt \sqrt{\numSys} \tau_{MC} \tau_{\vLtil} \norm{\vM_{1:\numSys}} \norm{\vSigma_\vw} 
			\frac{\rho_{MC}}{\rho_{MC} - \rho_{\vLtil}} (\frac{\rho_{MC}}{1-\rho_{MC}} - \frac{\rho_{\vLtil}}{1-\rho_{\vLtil}})
		\end{split}
	\end{equation}
	which concludes the proof.
\end{proof}

We now provide a uniform upper bound on the regret $J_T(\vx_0, \omega(0), \curlybrackets{\vK_{1:\numSys}, \vSigma_\vz}) - T J^\star$ for any stabilizing controller $\vK_{1:\numSys}$ that is close enough to the optimal controller $\vK^\star_{1:\numSys}$.
\begin{proposition}\label{proposition_regretanalysispart2}
	For every $\vK_{1:\numSys}$ s.t. $\norm{\vK_{1:\numSys} - \vKstar_{1:\numSys}} \leq \bar{\epsilon}_{\vK}$, we have
	\begin{equation} \nn
		\begin{split}
			J_T(\vx_0, \omega(0), \curlybrackets{\vK_{1:\numSys}, \vSigma_\vz}) - T J^\star 
			&\leq C_\vK^J \norm{\vK_{1:\numSys} - \vKstar_{1:\numSys}}^2 \norm{\vSigma_\vw} T+ \sqrt{\dimSt \numSys} M  \norm{\vx_0}^2 \\
			&+ \dimSt \sqrt{\numSys} \frac{2 \tau(\vLtil^\star) \norm{\vB_{1:\numSys}}^2 M }{1-\rho^\star} \norm{\vSigma_{\vz}}  T + \dimSt \norm{\vR_{1:\numSys}}  \norm{\vSigma_\vz}  T \\
			&+ \dimSt \sqrt{\numSys} \frac{2 \tau(\vLtil^\star) \tau_{MC}  M \rho_{MC} }{2 \rho_{MC} - 1 - \rho^\star} (\frac{\rho_{MC}}{1-\rho_{MC}} {-} \frac{1+\rho^*}{1-\rho^*}) \norm{\vSigma_\vw},
		\end{split}
	\end{equation}
	where $M:=\norm{\vQ_{1:\numSys}} + 4 \norm{\vR_{1:\numSys}} \norm{\vKstar_{1:\numSys}}^2$, and $\bar{\epsilon}_\vK$ and $C_\vK^J$ are defined in Table \ref{table_Notation_Perturbation}.
\end{proposition}
\parentchild{Lemma \ref{lemma_spectralradiusperturbation}, Lemma \ref{lemma_tac_1}, Lemma \ref{lemma_regretanalysispart1}}{Proposition \ref{proposition_epochRegret}}{Zhe}{Zhe}

\begin{proof}
	When $\norm{\vK_{1:\numSys} - \vKstar_{1:\numSys}} \leq \bar{\epsilon}_{\vK}$, from Lemma \ref{lemma_spectralradiusperturbation}, we know $\norm{\vLtil^k} \leq \tau(\vLtil^\star) \parentheses{\frac{1+\rho^\star}{2}}^k$, thus we could set $\tau_{\vLtil}$ and $\rho_{\vLtil}$ to be $\tau(\vLtil^\star)$ and $\frac{1 + \rho^\star}{2}$. By definition, we know $\bar{\epsilon}_{\vK} \leq \norm{\vKstar_{1:\numSys}}$, thus $\norm{\vM_{1:\numSys}} \leq \norm{\vQ_{1:\numSys}} + \norm{\vR_{1:\numSys}} \norm{\vK_{1:\numSys}}^2 \leq \norm{\vQ_{1:\numSys}} + \norm{\vR_{1:\numSys}} (\norm{\vKstar_{1:\numSys}} + \bar{\epsilon}_{\vK})^2 \leq \norm{\vQ_{1:\numSys}} + 4 \norm{\vR_{1:\numSys}} \norm{\vKstar_{1:\numSys}}^2 = M$. Then applying Lemma \ref{lemma_regretanalysispart1}, we have
	\begin{equation}\label{eq_tac_25}
		\begin{split}
			&J_T(\vx_0, \omega(0), \curlybrackets{\vK_{1:\numSys}, \vSigma_\vz}) - T J(0, \omega(0), \curlybrackets{\vK_{1:\numSys}}) \\
			&\qquad\leq \sqrt{\dimSt \numSys} M  \norm{\vx_0}^2 + \dimSt \sqrt{\numSys} \frac{2 \tau(\vLtil^\star) \norm{\vB_{1:\numSys} }^2 M}{1-\rho^\star}  \norm{\vSigma_{\vz}}  T
			+ \dimSt \norm{\vR_{1:\numSys}}  \norm{\vSigma_\vz}  T \\
			&\qquad+ \dimSt \sqrt{\numSys} \frac{2 \tau(\vLtil^\star) \tau_{MC}  M \rho_{MC} }{2 \rho_{MC} - 1 - \rho^\star} (\frac{\rho_{MC}}{1-\rho_{MC}} {-} \frac{1+\rho^*}{1-\rho^*}) \norm{\vSigma_\vw}
		\end{split}
	\end{equation}
	When $\norm{\vK_{1:\numSys} - \vKstar_{1:\numSys}} \leq \bar{\epsilon}_{\vK}$, we have $J(0, \omega(0), \curlybrackets{\vK_{1:\numSys}}) - J^\star \leq C_\vK^J \norm{\vSigma_\vw} \norm{\vK_{1:\numSys} - \vKstar_{1:\numSys}}^2$ using Lemma \ref{lemma_tac_1}. Combining this with \eqref{eq_tac_25}, we conclude the proof.
\end{proof}

\subsection{Stitching Every Epoch}\label{appendix_MJSRegretAnalysis_Stitching}
In this section, we stitch the upper bounds on $\text{Regret}_\epkidx$ for every epoch $\epkidx$ and build a bound on the overall regret $\text{Regret}(T)$. 

We define the estimation error after epoch $\epkidx$ as $ \epsilon_{\vA, \vB}^{(\epkidx)} = \max \curlybrackets{\norm{\vA^{(\epkidx)}_{1:\numSys} - \vA_{1:\numSys}}, \norm{\vB^{(\epkidx)}_{1:\numSys}  - \vB_{1:\numSys}}}$, $\epsilon_{\vT}^{(\epkidx)} = \norm{\vT^{(\epkidx)}  - \vT}_{\infty}$. Furthermore, we also define $\epsilon^{(\epkidx)}_{\vK} := \norm{\vK^{(\epkidx)}_{1:\numSys} - \vKstar_{1:\numSys}}$ where $\vKstar_{1:\numSys}$ is the optimal controller for the infinite-horizon MJS-LQR($\vA_{1:\numSys}, \vB_{1:\numSys}, \vT, \vQ_{1:\numSys}, \vR_{1:\numSys}$).
We define the following events for every epoch $\epkidx$.
\begin{equation}\label{eq_regretEventsDefUniformStability}
	\begin{split}
		\Acal_\epkidx 
		&{:= }\bigg\{ 
		\text{Regret}_\epkidx \leq  \Ocal\bigg(\numSys \dimInput \left( \epsilon_{\vA, \vB}^{(\epkidx-1)} + \epsilon_{\vT}^{(\epkidx-1)} \right)^2 \sigma_\vw^2 T_\epkidx 
		+ \sqrt{\dimSt \numSys} \norm{\vx_0^{(\epkidx)}}^2
		+ \frac{\dimSt \sqrt{\numSys}}{1-\rho^\star} \sigma_{\vz,\epkidx}^2 T_\epkidx
		+ c_{\Acal} \bigg) 
		\bigg\} \\
		\Bcal_\epkidx 
		&{:=} \left\{ \epsilon_{\vA, \vB}^{(\epkidx)} \leq \bar{\epsilon}_{\vA, \vB, \vT}, \epsilon_{\vT}^{(\epkidx)} \leq \bar{\epsilon}_{\vA, \vB, \vT}, 
		\epsilon_{\vK}^{(\epkidx+1)} \leq \bar{\epsilon}_{\vK} \right\} \\
		\Ccal_\epkidx 
		&{:=} \bigg\{
		\epsilon_{\vA, \vB}^{(\epkidx)} \leq \Ocal\parenthesesbig{ \log(\frac{1}{\delta_{id, \epkidx}}) \frac{\sigma_{\vz,\epkidx} + \sigma_\vw}{\sigma_{\vz,\epkidx}} \mysqrt[1pt]{\frac{(n+p)\log(T_q)}{\pi_{\min}(1 - \varrho)T_q}} },  \\
		& \qquad  \epsilon_{\vT}^{(\epkidx)} \leq \Ocal \bigg( \log(\frac{1}{\delta_{id, \epkidx}})  \frac{1}{\pi_{\min}}\mysqrt[1pt]{\frac{\log(T_\epkidx)}{T_\epkidx}} \bigg) \bigg\}\\
		\Dcal_\epkidx 
		&{:=} \curlybracketsbig{\norm{\vx^{(\epkidx+1)}_{0}}^2 = \norm{\vx^{(\epkidx)}_{T_\epkidx}}^2 \leq  \frac{\bar{x}_0^2}{\delta_{\vx_0, \epkidx}}} \\
		\Ecal_\epkidx 
		&{:= }\Acal_{\epkidx+1} \cap \Bcal_\epkidx \cap \Ccal_\epkidx \cap \Dcal_\epkidx.
	\end{split}
\end{equation}
where $c_\Acal, \bar{x}_0$ are constants, $\bar{\epsilon}_{\vA, \vB, \vT}$, $\bar{\epsilon}_{\vK}$ and $\varrho$ are defined in Table \ref{table_Notation_Perturbation}, 
and $\delta_{id, \epkidx}$ and $\delta_{\vx_0, \epkidx}$ within $[0,1]$ denotes the failure probability for event $\Ccal_\epkidx$ and $\Dcal_\epkidx$.
Note that $\Ocal(\cdot)$ hides terms that are invariant to epochs such as $\rho^\star, \norm{\vA_{1:\numSys}}, \norm{\vB_{1:\numSys}}$, etc.

Event $\Acal_\epkidx$ describes how epoch $\epkidx$ regret depends on initial state $\norm{\vx_0^{(\epkidx)}}^2$, exploration noise variance $\sigma_{\vz,\epkidx}^2$, and the accuracy of the estimated MJS dynamics $\vA_{1:\numSys}^{(\epkidx-1)}, \vB_{1:\numSys}^{(\epkidx-1)}, \vT^{(q-1)}$ after epoch $\epkidx-1$, which is used to compute epoch $\epkidx$ controller $\vK_{1:\numSys}^{(\epkidx)}$. Event $\Bcal_\epkidx$ indicates whether the estimated dynamics and resulting controllers are good enough. $\Ccal_\epkidx$ describes the dynamics estimation error after epoch $\epkidx$, and when epoch $T_\epkidx$ is chosen appropriately, $\Bcal_\epkidx$ can be implied. Lastly, event $\Dcal_\epkidx$ bounds the initial state of each epoch, as the initial state plays a vital role in regret upper bound $\Acal_\epkidx$.
We see events $\Acal_{\epkidx+1}, \Bcal_\epkidx, \Ccal_\epkidx, \Dcal_\epkidx$ are $\Fcal_\epkidx$-measurable, i.e. these events can be determined using random variables $\vx_0, \vw_t, \vz_t, \omega(t)$ up to epoch $\epkidx$. Note that even though $\Acal_{\epkidx+1}$ is for the conditional expected regret of the epoch $\epkidx+1$, with randomness coming from $\vx_0^{(\epkidx+1)} = \vx_{T_\epkidx}^{(\epkidx)}, \omega^{(\epkidx+1)}(0) = \omega^{(\epkidx)}(T_\epkidx)$, and controller $\vK_{1:\numSys}^{(\epkidx+1)}$ computed from $\vA_{1:\numSys}^{(\epkidx)}, \vB_{1:\numSys}^{(\epkidx)}, \vT^{(\epkidx)}$, thus $\Acal_{\epkidx+1}$ is $\Fcal_\epkidx$-measurable.

Then, we have the following results regarding the conditional probabilities of these events. First, Proposition \ref{proposition_epochInitState} says that given the event $\Bcal_{\epkidx-1}$ (a good controller is applied during epoch $\epkidx$) and event $\Dcal_{\epkidx-1}$ (the initial state of epoch $\epkidx$, $\vx_{0}^{(\epkidx)}$ is bounded), then $\Dcal_\epkidx$ could occur, i.e., $\vx_{T_\epkidx}^{(\epkidx)}$ the final state of epoch $\epkidx$, alternately $\vx_0^{(\epkidx+1)}$ the initial state of epoch $\epkidx+1$, is also bounded.

\begin{proposition}\label{proposition_epochInitState}
	Suppose 
	$\frac{\sqrt{\dimSt \numSys} \bar{\tau} \bar{\rho}^{T_\epkidx}}{\delta_{\vx_0, \epkidx-1}} < 1$ and $\splitatcommas{\bar{x}_0^2 \geq { \frac{\dimSt \sqrt{\numSys} (\norm{\vB_{1:\numSys}}^2 + 1)\sigma_{\vw}^2 \bar{\tau} }{(1-\bar{\rho})(1-\sqrt{\dimSt \numSys} \cdot \bar{\tau} \bar{\rho}^{T_\epkidx} / \delta_{\vx_0, \epkidx-1} )}}}~~ \textnormal{for}~~i\geq1.$
	Then,
	\begin{equation*}
		\P(\Dcal_\epkidx \mid \cap_{j=0}^{\epkidx-1} \Ecal_j) = \P(\Dcal_\epkidx \mid \Bcal_{\epkidx-1}, \Dcal_{\epkidx-1}) > 1-\delta_{\vx_0, \epkidx},
	\end{equation*}
	and $\P(\Dcal_0) \geq 1 - \delta_{\vx_0, 0}$.
\end{proposition}
\parentchild{Lemma \ref{lemma_statebound}, Lemma \ref{lemma_spectralradiusperturbation}}{Theorem \ref{thrm_mainthrm}}{Zhe}{Zhe}
\begin{proof}
	For epoch $\epkidx = 1,2, \dots$, given event $\Bcal_{\epkidx-1}$, we know $\epsilon_{\vK}^{(\epkidx)} \leq \bar{\epsilon}_{\vK}$. Let $\vLtil^{(\epkidx)}$ denotes the augmented closed-loop state matrix. By Lemma \ref{lemma_spectralradiusperturbation}, we know $\norm{(\vLtil^{(\epkidx)})^k} \leq \tau(\vLtil^\star) (\frac{1+\rho^\star}{2})^k$. Thus, if we pick $\bar{\tau} := \max \curlybrackets{\tau(\vLtil^{(0)}), \tau(\vLtil^\star)}, \bar{\rho} := \max\curlybrackets{\rho(\vLtil^{(0)}), \frac{1+\rho^\star}{2}}$, this can be generalized to $\epkidx = 0$ case, i.e. for every epoch $\epkidx = 0,1,2, \dots$, we have $\norm{(\vLtil^{(\epkidx)})^k} \leq \bar{\tau} \bar{\rho}^k$.
	
	For $\epkidx = 1,2,\dots$, event $\Dcal_{\epkidx-1}$ implies $\norm{\vx_0^{(\epkidx)}}^2 \leq \frac{\bar{x}_0^2}{\delta_{\vx_0, \epkidx-1}}$. Then, according to Lemma \ref{lemma_statebound}, we know
	\begin{equation}\label{eq_tac_24}
		\begin{split}
			\expctn[\norm{\vx^{(\epkidx)}_{T_\epkidx}}^2 \mid \Bcal_{\epkidx-1}, \Dcal_{\epkidx-1} ]  
			&\leq
			\sqrt{\dimSt \numSys} \cdot \bar{\tau} \bar{\rho}^{T_\epkidx} \frac{\bar{x}_0^2}{\delta_{\vx_0, \epkidx-1}} + \dimSt \sqrt{\numSys} (\norm{\vB_{1:\numSys}}^2 \frac{\sigma_{\vw}^2}{\sqrt{T_\epkidx}} + \sigma_{\vw}^2) \frac{ \bar{\tau}}{1-\bar{\rho}} \\
			&\leq
			\frac{\sqrt{\dimSt \numSys} \cdot \bar{\tau} \bar{\rho}^{T_\epkidx}}{\delta_{\vx_0, \epkidx-1}} \bar{x}_0^2
			+ (1- \frac{\sqrt{\dimSt \numSys} \bar{\tau} \bar{\rho}^{T_\epkidx}}{\delta_{\vx_0, \epkidx-1}}) \bar{x}_0^2\\
			&\leq \bar{x}_0^2,
		\end{split}    
	\end{equation}
	where the second line follows from the assumptions in the proposition statement. Using Markov inequality, we have
	\begin{equation*}
		\P(\norm{\vx^{(\epkidx)}_{T_\epkidx}}^2 \leq \frac{\bar{x}_0^2}{\delta_{\vx_0, \epkidx}} \mid \Bcal_{\epkidx-1}, \Dcal_{\epkidx-1}) \geq 1 - \delta_{\vx_0, \epkidx},
	\end{equation*}
	which implies 
	$
	\P(\Dcal_\epkidx \mid \Bcal_{\epkidx-1}, \Dcal_{\epkidx-1}) \geq 1 - \delta_{\vx_0, \epkidx}.
	$
	For $\epkidx = 0$, similarly, we have 
	$\expctn[\norm{\vx^{(0)}_{T_0}}^2 ] \leq \dimSt \sqrt{\numSys} (\norm{\vB_{1:\numSys}}^2 \frac{\sigma_{\vw}^2}{\sqrt{T_\epkidx}} + \sigma_{\vw}^2) \frac{ \bar{\tau}}{1-\bar{\rho}} \leq \bar{x}_0^2$, thus $\P(\Dcal_0)\geq 1-\delta_{\vx_0, \epkidx}$.
	
	Finally, note that given a good stabilizing controller (event $\Bcal_{\epkidx-1}$) and a bounded initial state (event $\Dcal_{\epkidx-1}$) for epoch $\epkidx$, the final state of epoch $\epkidx$ only depends on randomness in epoch $\epkidx$, thus $\P(\Dcal_\epkidx \mid \cap_{j=0}^{\epkidx-1} \Ecal_j) = \P(\Dcal_\epkidx \mid \Bcal_{\epkidx-1}, \Dcal_{\epkidx-1})$.
\end{proof}

Proposition \ref{proposition_epochEstErrGuarantee} describes that given the event $\Ccal_\epkidx$ (the estimated MJS dynamics after epoch $\epkidx$ has estimation errors decays with $T_\epkidx$), when epoch $\epkidx$ has length $T_\epkidx$ large enough, then the event $\Bcal_\epkidx$ (the estimated dynamics and controllers computed with it will be good enough) occurs.
\begin{proposition}\label{proposition_epochEstErrGuarantee}
	Suppose every epoch $\epkidx$ has length $T_\epkidx \geq \underline{T}_{rgt, \bar{\epsilon}}(\delta_{id, \epkidx}, T_\epkidx)$. Then,
	\begin{equation*}
		\P(\Bcal_\epkidx \mid \Ccal_\epkidx, \cap_{j=0}^{\epkidx-1} \Ecal_j) = \P(\Bcal_\epkidx \mid \Ccal_\epkidx) = 1
	\end{equation*}
\end{proposition}
\parentchild{Lemma \ref{lemma_tac_1}}{Theorem \ref{thrm_mainthrm}}{Zhe}{Zhe}
\begin{proof}
	When $\Ccal_\epkidx$ occurs, since $\sigma_{\vz,\epkidx}^2 = \frac{\sigma_{\vw}^2}{\sqrt{T_\epkidx}}$, we have
	\begin{align*}
	\epsilon_{\vA, \vB}^{(\epkidx)} &\leq {\Ocal}\bigg(\log(\frac{1}{\delta_{id, \epkidx}})  \mysqrt[1pt]{\frac{(n+p)\log(T_q)}{\pi_{\min}(1 - \varrho)T_q^{0.5}}}\bigg), \\
	 \epsilon_{\vT}^{(\epkidx)} &\leq \Ocal\bigg(\log(\frac{1}{\delta_{id, \epkidx}}) \frac{1}{\pi_{\min}} \mysqrt[1pt]{\frac{\log(T_\epkidx)}{T_\epkidx}}\bigg).
	\end{align*}
	 When $T_\epkidx \geq \Ocal( \log(\frac{1}{\delta_{id, \epkidx}}) \frac{(\dimSt + \dimInput)^2}{\pi_{\min}^2 (1-\varrho)^2\bar{\epsilon}_{\vA,\vB,\vT}^{4}}   \log^2(T_\epkidx)) =: \underline{T}_{rgt, \bar{\epsilon}}(\delta_{id, \epkidx}, T_\epkidx)$, we have $\epsilon_{\vA, \vB}^{(\epkidx)} \leq \bar{\epsilon}_{\vA, \vB, \vT}$, $ \epsilon_{\vT}^{(\epkidx)} \leq \bar{\epsilon}_{\vA, \vB, \vT}$. Then according to Lemma \ref{lemma_tac_1}, we have $\epsilon_{\vK}^{(\epkidx+1)} \leq \bar{\epsilon}_{\vK}$. Thus $\P(\Bcal_\epkidx \mid \Ccal_\epkidx) = 1$. Finally, note that given the estimation error sample complexity in $\Ccal_\epkidx$ for epoch $\epkidx$, events happen before epoch $\epkidx$ does not influence $\Bcal_\epkidx$, hence $\P(\Bcal_\epkidx \mid \Ccal_\epkidx, \cap_{j=0}^{\epkidx-1} \Ecal_j) = \P(\Bcal_\epkidx \mid \Ccal_\epkidx)$=1.
\end{proof}

Next, Proposition \ref{proposition_epochsysid} says given the event $\Bcal_{\epkidx-1}$ (a good controller is used in epoch $\epkidx$), then the event $\Ccal_\epkidx$ could occur, i.e., dynamics learned using the trajectory of epoch $\epkidx$, will be accurate enough. \begin{proposition}\label{proposition_epochsysid}
	For $~T_\epkidx \geq \max \big\{\underline{T}_{MC,1}(\frac{\delta_{id, \epkidx}}{8}),$  $ \underline{T}_{id, N}(\frac{\delta_{id, \epkidx}}{2})\big\}$,
		we have for $\epkidx = 1,2,\dots,$
	\begin{equation}
		\P(\Ccal_\epkidx \mid \cap_{j=0}^{\epkidx-1} \Ecal_j)
		=\P(\Ccal_\epkidx \mid \Bcal_{\epkidx-1}) \geq 1-\delta_{id, \epkidx}.
	\end{equation}
	And $\P(\Ccal_0) \geq 1-\delta_{id,0}$. 
\end{proposition}
\parentchild{Lemma \ref{lemma_ConcentrationofMC}, Lemma \ref{lemma_spectralradiusperturbation}}{Theorem \ref{thrm_mainthrm}}{Zhe}{}
\begin{proof}
	By Lemma \ref{lemma_ConcentrationofMC}, we know, for every epoch $\epkidx = 0,1,\dots$, when $T_\epkidx \geq \underline{T}_{MC,1}(\frac{\delta_{id, \epkidx}}{8})$, we have with probability at least $1 - \frac{\delta_{id, \epkidx}}{2}$,
	$
	\epsilon_{\vT}^{(\epkidx)} \leq \Ocal \parenthesesbig{ \log(\frac{1}{\delta_{id, \epkidx}}) \frac{1}{\pi_{\min}}\sqrt{\frac{\log(T_\epkidx)}{T_\epkidx}}}.
	$
	
	Next, for epoch $\epkidx = 1,2, \dots$, given event $\Bcal_{\epkidx-1}$, we know $\epsilon_{\vK}^{(\epkidx)} \leq \bar{\epsilon}_{\vK}$. Let $\vLtil^{(\epkidx)}$ denote the augmented closed-loop state matrix. By Lemma \ref{lemma_spectralradiusperturbation}, we know $\norm{(\vLtil^{(\epkidx)})^k} \leq \tau(\vLtil^\star) (\frac{1+\rho^\star}{2})^k$. Thus, if we pick $\bar{\tau} := \max \curlybrackets{\tau(\vLtil^{(0)}), \tau(\vLtil^\star)}, \bar{\rho} := \max\curlybrackets{\rho(\vLtil^{(0)}), \frac{1+\rho^\star}{2}}$, this can be generalized to $\epkidx = 0$ case, i.e., for every epoch $\epkidx = 0,1,2, \dots$, we have $\norm{(\vLtil^{(\epkidx)})^k} \leq \bar{\tau} \bar{\rho}^k$. 
	
	Suppose $T_\epkidx \geq \underline{T}_{id, N}(\frac{\delta_{id, \epkidx}}{2})$ hold for $\epkidx = 0,1,\dots$. Then, from Theorem \ref{thrm:learn_dynamics_complete}, we know for every $\epkidx = 0,1,\dots$, with probability at least $1-\frac{\delta_{id, \epkidx}}{2}$,
		$
		\epsilon_{\vA,\vB}^{(\epkidx)} \leq 
		\Ocal\parenthesesbig{ \log(\frac{1}{\delta_{id, \epkidx}}) \frac{\sigma_{\vz,\epkidx} + \sigma_\vw}{\sigma_{\vz,\epkidx}} \mysqrt[1pt]{\frac{(n+p)\log(T_q)}{\pi_{\min}(1 - \varrho)T_q}} }.
		$
Applying union bound to $\epsilon_\vT^{(\epkidx)}$ and $\epsilon_{\vA, \vB}^{(\epkidx)}$, we could show $\P(\Ccal_0) \geq 1-\delta_{id, \epkidx}$ and $\P(\Ccal_\epkidx \mid \Bcal_{\epkidx-1}, \Dcal_{\epkidx-1}) \geq 1 - \delta_{id, \epkidx} $. Finally, note that given a good stabilizing controller (event $\Bcal_{\epkidx-1}$) and bounded initial state (event $\Dcal_{\epkidx-1}$) for epoch $\epkidx$, the estimation error sample complexity (event $\Ccal_\epkidx$) does not depend on events happen before epoch $\epkidx$, so $\P(\Ccal_\epkidx \mid \cap_{j=0}^{\epkidx-1} \Ecal_j) )
	=\P(\Ccal_\epkidx \mid \Bcal_{\epkidx-1}, \Dcal_{\epkidx-1})$.
\end{proof}

Finally, Proposition \ref{proposition_epochRegret} simply describes how the regret of epoch $\epkidx$ depends on the accuracy of the estimated dynamics after epoch $\epkidx-1$.
\begin{proposition}\label{proposition_epochRegret}
	For $\Acal_\epkidx$-- $\Ecal_\epkidx$ given in \eqref{eq_regretEventsDefUniformStability}, we have
	\begin{equation*}
		\P(\Acal_\epkidx \mid \Bcal_{\epkidx-1}, \Ccal_{\epkidx-1}, \Dcal_{\epkidx-1}, \cap_{j=0}^{\epkidx-2} \Ecal_j) 
		=\P(\Acal_\epkidx \mid \Bcal_{\epkidx-1}) = 1.
	\end{equation*}
\end{proposition}
\parentchild{Lemma \ref{lemma_tac_1}, Proposition \ref{proposition_regretanalysispart2}}{Theorem \ref{thrm_mainthrm}}{Zhe}{Zhe}
\begin{proof}
	From Proposition \ref{proposition_regretanalysispart2}, we know that for every epoch $\epkidx = 1,2, \dots$, given $ \norm{\vK_{1:\numSys}^{(\epkidx)} - \vKstar_{1:\numSys} } \leq \bar{\epsilon}_{\vK}$ in $\Bcal_{\epkidx-1}$, we have with probability $1$,
	\begin{equation}\label{eq_tac_26}
		\begin{split}
			\text{Regret}_\epkidx 
			&\leq C_\vK^J \norm{\vK_{1:\numSys}^{(\epkidx)} - \vKstar_{1:\numSys}}^2 \sigma_\vw^2 T_\epkidx + \sqrt{\dimSt \numSys} M  \norm{\vx_0^{(\epkidx)}}^2 \\
			&+ \dimSt \sqrt{\numSys} \frac{2 \tau(\vLtil^\star) \norm{\vB_{1:\numSys}}^2 M }{1-\rho^\star} \sigma_{\vz,\epkidx}^2  T_\epkidx + \dimSt \norm{\vR_{1:\numSys}}  \sigma_{\vz,\epkidx}^2  T_\epkidx \\
			&+ \dimSt \sqrt{\numSys} \frac{2 \tau(\vLtil^\star) \tau_{MC}  M \rho_{MC} }{2 \rho_{MC} - 1 - \rho^\star} (\frac{\rho_{MC}}{1-\rho_{MC}} {-} \frac{1+\rho^*}{1-\rho^*}) \sigma_\vw^2
		\end{split}
	\end{equation}
	Let $c_\Acal$ denotes the last term in \eqref{eq_tac_26}, which is a constant over epochs. Note that from $\epsilon_{\vA, \vB}^{(\epkidx-1)} \leq \bar{\epsilon}_{\vA, \vB, \vT}, \epsilon_{\vT}^{(\epkidx-1)} \leq \bar{\epsilon}_{\vA, \vB, \vT}$ in event $\Bcal_{\epkidx-1}$, we know $\norm{\vK_{1:\numSys}^{(\epkidx)} - \vKstar_{1:\numSys}} \leq C_{\vA, \vB, \vT
	}^\vK ( \epsilon_{\vA, \vB}^{(\epkidx-1)} + \epsilon_{\vT}^{(\epkidx-1)} )$ by Lemma \ref{lemma_tac_1}. Plugging this into \eqref{eq_tac_26}, we have
	\begin{equation} \label{eq_tac_54}
		\begin{split}
		\text{Regret}_\epkidx &\leq 
		\Ocal\bigg(
		\numSys \cdot \dimInput \left( \epsilon_{\vA, \vB}^{(\epkidx-1)} + \epsilon_{\vT}^{(\epkidx-1)} \right)^2 \sigma_\vw^2 T_\epkidx 
		+ \sqrt{\dimSt \numSys} \norm{\vx_0^{(\epkidx)}}^2
		+ \frac{\dimSt \sqrt{\numSys}}{1-\rho^\star} \sigma_{\vz,\epkidx}^2 T_\epkidx
		+ c_{\Acal}
		\bigg) 
		\end{split}
	\end{equation}
	where term $\numSys \cdot \dimInput$ comes from term $\numSys \min\{\dimSt, \dimInput\}$ in the definition of  $C_{\vK}^J$ in Appendix \ref{appendix_MJSRegretAnalysis_InfHorizonPerturb}. This shows $\P(\Acal_\epkidx \mid \Bcal_{\epkidx-1}) = 1$. Finally, note that given a good controller (event $\Bcal_{\epkidx-1}$) for epoch $\epkidx$, the regret for epoch $\epkidx$ can be upper bounded (event $\Acal_\epkidx$) without dependence on other events, thus $\P(\Acal_\epkidx \mid \Bcal_{\epkidx-1}, \Ccal_{\epkidx-1}, \Dcal_{\epkidx-1}, \cap_{j=0}^{\epkidx-2} \Ecal_j) 
	=\P(\Acal_\epkidx \mid \Bcal_{\epkidx-1})$.
\end{proof}

\subsubsection{Proof for Theorem \ref{thrm_mainthrm}} \label{appendix_proofforregret}
\begin{theorem}[Theorem \ref{thrm_mainthrm} complete version]\label{thrm_mainthrm_complete}
	Assume that the initial state $\vx_0 = 0$, and Assumption \ref{asmp_ergodicity_MSS} holds. 
		Suppose $\splitatcommas{T_0 \geq \Ocal(\underline{T}_{rgt} (\delta, T_0))}$, and $\splitatcommas{\bar{x}_0^2 = { \frac{\dimSt \sqrt{\numSys} (\norm{\vB_{1:\numSys}}^2 + 1)\sigma_{\vw}^2 \bar{\tau} }{(1-\bar{\rho})(1-\sqrt{\dimSt \numSys} \cdot \bar{\tau} \bar{\rho}^{T_0 \gamma}  \pi^2/ 3\delta )}}}$.
		Then, with probability at least $1-\delta$, Algorithm \ref{Alg_adaptiveMJSLQR} achieves
		\begin{equation}\label{eq_ExactRegretBound}
			\begin{split}
			\textnormal{Regret}(T) &{\leq} \Ocal \bigg( \frac{\numSys \dimInput (\dimSt + \dimInput) \sigma_\vw^2 }{\pi_{\min} (1-\varrho \lor \rho^\star)} \log\big(\frac{\log^2(T)}{\delta}\big)
			\log(T) \sqrt{T}  +
			\frac{\sqrt{\dimSt \numSys}\log^3(T)}{\delta} \bigg) .
			\end{split}
	\end{equation}
\end{theorem}
\parentchild{Proposition \ref{proposition_epochEstErrGuarantee}, Proposition \ref{proposition_epochsysid}, Proposition \ref{proposition_epochInitState}, Proposition \ref{proposition_epochRegret}, }{}{Zhe}{}

\begin{proof}
	In this proof, we will first show the intersected event $\cap_\epkidx \Ecal_\epkidx = \cap_\epkidx \curlybrackets{\Acal_{\epkidx+1} \cap \Bcal_\epkidx \cap \Ccal_\epkidx \cap \Dcal_\epkidx}$ implies the desired regret bound, then we evaluate the occurrence probability of $\cap_\epkidx \Ecal_\epkidx$ using Propositions \ref{proposition_epochEstErrGuarantee} to \ref{proposition_epochRegret}.  In the following, we set $\delta_{id, \epkidx} = \delta_{\vx_0, \epkidx} = \frac{3}{\pi^2} \cdot \frac{\delta}{ (\epkidx+1)^2}$. With the choices $T_\epkidx = \gamma T_{\epkidx-1}$, $\sigma_{\vz,\epkidx}^2 = \frac{\sigma_\vw^2}{\sqrt{T_\epkidx}}$, and $\delta_{id, \epkidx} = \delta_{\vx_0, \epkidx} = \frac{3}{\pi^2} \cdot \frac{\delta}{ (\epkidx+1)^2}$, event $\Ecal_\epkidx = \Acal_{\epkidx+1} \cap \Bcal_\epkidx \cap \Ccal_\epkidx \cap \Dcal_\epkidx$ implies the following,
		\begin{equation}\label{eq_108}
			\begin{split}
				\text{Regret}_{\epkidx+1} 
				&\leq \Ocal(1)\log\parenthesesbig{\frac{(\epkidx{+}1)^2}{\delta}} \numSys \dimInput \bigg(\frac{\sigma_{\vz,\epkidx} {+} \sigma_\vw}{\sigma_{\vz,\epkidx}}  \mysqrt[1pt]{\frac{(n{+}p)\log(T_q)}{\pi_{\min}(1 {-} \varrho)T_q}} {+} \frac{\sqrt{\log(T_\epkidx)}}{\pi_{\min} \sqrt{T_\epkidx}}\bigg)^2 \sigma_\vw^2 T_{\epkidx+1} \\ 
				 &+ \Ocal \parenthesesbig{\frac{(\epkidx{+}1)^2}{\delta}} \sqrt{\dimSt \numSys} \bar{x}_0^2 {+} \Ocal(\frac{\dimSt \sqrt{\numSys}}{1{-}\rho^\star} \sigma_{\vz,\epkidx+1}^2 T_{\epkidx+1}) {+} \Ocal(1), \\
				&\leq \Ocal(1) \log\parenthesesbig{\frac{(\epkidx{+}1)^2}{\delta}} 
				\frac{\numSys \dimInput (\dimSt {+} \dimInput) \gamma}{\pi_{\min} (1{-}\varrho)} 
				\frac{(\sigma_{\vz,\epkidx} {+} \sigma_\vw)^2}{\sigma_{\vz,\epkidx}^2} \sigma_\vw^2 \log(T_\epkidx) \\
				& + \Ocal \parenthesesbig{\frac{(\epkidx{+}1)^2}{\delta}} \sqrt{\dimSt \numSys} \bar{x}_0^2 {+} \Ocal(\frac{\dimSt \sqrt{\numSys}}{1{-}\rho^\star} \sigma_{\vz,\epkidx+1}^2 T_{\epkidx+1}), \\
				&\leq \Ocal(1) \log\parenthesesbig{\frac{(\epkidx{+}1)^2}{\delta}} 
				\frac{\numSys \dimInput (\dimSt {+} \dimInput) \gamma}{\pi_{\min} (1{-}\varrho \lor \rho^\star)}
				\parenthesesbig{\frac{\sigma_\vw^4}{\sigma_{\vz,\epkidx}^2} \log(T_\epkidx) {+} \sigma_{\vz,\epkidx+1}^2 T_\epkidx} \\
                &+ \Ocal \parenthesesbig{\frac{(\epkidx{+}1)^2}{\delta}} \sqrt{\dimSt \numSys} \bar{x}_0^2,\\
				&\leq \Ocal(1) \log\parenthesesbig{\frac{(\epkidx{+}1)^2}{\delta}} 
				\frac{\numSys \dimInput (\dimSt {+} \dimInput) \gamma}{\pi_{\min} (1{-}\varrho \lor \rho^\star)} \sigma_\vw^2 \sqrt{T_\epkidx} \log(T_\epkidx) 
				 + \Ocal \parenthesesbig{\frac{(\epkidx{+}1)^2}{\delta}} \sqrt{\dimSt \numSys} \bar{x}_0^2.
			\end{split}
		\end{equation}
		We have $M:= \Ocal(\log_\gamma (\frac{T}{T_0}))$ epochs at time $T$. Using $T_\epkidx=\Ocal(T_0 \gamma^\epkidx)$, event $\cap_{\epkidx = 0}^{M-1} \Ecal_\epkidx$ implies
		\begin{equation} \label{eq_regretmaintheormm_eq1}
			\begin{split}
				&\text{Regret}(T) = \Ocal(\sum_{\epkidx = 1}^{M} \text{Regret}_\epkidx) \\
				&\leq  \Ocal(1) \log\big(\frac{\log^2(T)}{\delta}\big) \frac{\numSys \dimInput (\dimSt + \dimInput) \sigma_\vw^2}{\pi_{\min} (1-\varrho \lor \rho^\star)}  
				\parenthesesbig{ \gamma \sum_{\epkidx = 1}^{M} \sqrt{T_\epkidx} \log(T_\epkidx) }+
				\Ocal\big(\frac{\sqrt{\dimSt \numSys}\log^3(T)}{\delta}\big)
			\end{split}    
		\end{equation}
		For the term $\gamma \sum_{\epkidx = 1}^{M} \sqrt{T_\epkidx} \log(T_\epkidx)$, we have
		\begin{equation}
			\begin{split}
				\gamma \sum_{\epkidx = 1}^{M} \sqrt{T_\epkidx} \log(T_\epkidx)   
				&\leq \Ocal(1) \gamma \sqrt{T_0}\parenthesesbig{\log(T_0) \sum_{\epkidx = 1}^M \sqrt{\gamma}^\epkidx {+} \log(\gamma) \sum_{\epkidx = 1}^{M}  q \sqrt{\gamma}^\epkidx } \\
				&\leq \Ocal(1) \gamma \sqrt{T_0} \parenthesesbig{\log(T_0) \frac{\sqrt{\gamma}^{M+1}}{\sqrt{\gamma}{-} 1} {+} \log(\gamma) \frac{M\sqrt{\gamma}^{M+2}}{(\sqrt{\gamma}{-} 1)^2} } \\
				&\leq \Ocal(1) \gamma \sqrt{T_0 \gamma^M} \parenthesesbig{\log(T_0) \frac{\sqrt{\gamma}}{\sqrt{\gamma}{-} 1} {+} \log(\gamma) \frac{M\sqrt{\gamma}^{2}}{(\sqrt{\gamma}{-} 1)^2} } \\
				&\leq \Ocal(1) \gamma \sqrt{T} \bigg(\log(T_0) \frac{\sqrt{\gamma}}{\sqrt{\gamma}{-}1} 
				{+} \log(\gamma) \frac{\log(T/T_0)}{\log(\gamma)}\frac{\sqrt{\gamma}^{2}}{(\sqrt{\gamma}{-} 1)^2} \bigg) \\
				&\leq \Ocal(1) \sqrt{T}\log(T) \parenthesesbig{ \frac{\gamma\sqrt{\gamma}}{\sqrt{\gamma}{-}1} {+} \frac{\gamma^{2}}{(\sqrt{\gamma}{-} 1)^2} } \\
				&\leq \Ocal(\log(T) \sqrt{T} ).
\end{split}
		\end{equation}
		Plugging this back into \eqref{eq_regretmaintheormm_eq1}, we have
		\begin{equation}\label{eq_100}
			\begin{split}
				\text{Regret}(T)
				&\leq  \Ocal \bigg( \frac{\numSys \dimInput (\dimSt {+} \dimInput) \sigma_\vw^2 }{\pi_{\min} (1{-}\varrho \lor \rho^\star)} \log\big(\frac{\log^2(T)}{\delta}\big) \log(T) \sqrt{T}  
				+
				\frac{\sqrt{\dimSt \numSys}\log^3(T)}{\delta} \bigg), \\
\end{split}
		\end{equation}
		which shows the regret bound in \eqref{eq_ExactRegretBound}.
	
	Now we are only left to show the occurrence probability of regret bound \eqref{eq_ExactRegretBound} is larger than $1 - \delta$. To do this, we will combine Proposition \ref{proposition_epochInitState}, \ref{proposition_epochEstErrGuarantee}, \ref{proposition_epochsysid}, and \ref{proposition_epochRegret} over all $\epkidx = 0, 1, \dots, M-1$. Note that for each individual $\epkidx$, these propositions hold only when certain prerequisite conditions on hyper-parameters $T_0$ and $\bar{x}_0$ are satisfied. We first show that under the choices $T_\epkidx = \gamma T_{\epkidx-1}$, $\sigma_{\vz,\epkidx}^2 = \frac{\sigma_\vw^2}{\sqrt{T_\epkidx}}$, and $\delta_{id, \epkidx} = \delta_{\vx_0, \epkidx} = \frac{3}{\pi^2} \cdot \frac{\delta}{ (\epkidx+1)^2}$ these hyper-parameter conditions can be satisfied for all $\epkidx = 0, 1, \dots, M-1$.
	
		\begin{itemize}
			\item Proposition \ref{proposition_epochInitState} requires that for $\epkidx = 1,2,\dots$; $\frac{\sqrt{\dimSt \numSys} \bar{\tau} \bar{\rho}^{T_0 \gamma^\epkidx} \epkidx^2 \pi^2}{3\delta} < 1$ and $\splitatcommas{\bar{x}_0^2 \geq { \frac{\dimSt \sqrt{\numSys} (\norm{\vB_{1:\numSys}}^2 + 1)\sigma_{\vw}^2 \bar{\tau} }{(1-\bar{\rho})(1-\sqrt{\dimSt \numSys} \cdot \bar{\tau} \bar{\rho}^{T_0 \gamma^\epkidx} \epkidx^2 \pi^2/ 3\delta )}}}$ need to be satisfied. Choosing $T_0 \geq \frac{1}{\gamma \log(1/\bar{\rho})} \max \curlybrackets{\frac{2}{\log(\gamma)}, \log(\frac{\pi^2 \sqrt{\dimSt \numSys} \bar{\tau}}{3 \delta})} =: \underline{T}_{\vx_0}(\delta)$, and picking $\splitatcommas{\bar{x}_0^2 \geq { \frac{\dimSt \sqrt{\numSys} (\norm{\vB_{1:\numSys}}^2 + 1)\sigma_{\vw}^2 \bar{\tau} }{(1-\bar{\rho})(1-\sqrt{\dimSt \numSys} \cdot \bar{\tau} \bar{\rho}^{T_0 \gamma}  \pi^2/ 3\delta )}}}$ would suffice for this.
			
			\item Proposition \ref{proposition_epochEstErrGuarantee} requires that for $\epkidx = 0,1,\dots$, condition $T_0 \gamma^\epkidx \geq \underline{T}_{rgt, \bar{\epsilon}}(\frac{3\delta}{\pi^2 (\epkidx+1)^2}, T_0 \gamma^\epkidx)$ holds, which can be satisfied when one chooses $T_0 \geq \Ocal(\underline{T}_{rgt, \bar{\epsilon}}(\delta, T_0))$.
			
			\item Proposition \ref{proposition_epochsysid} require $T_0 \gamma^\epkidx {\geq} \max \splitatcommas{\big\{\underline{T}_{MC,1}(\frac{3 \delta}{8 \pi^2 \epkidx^2}), \underline{T}_{id, N}(\frac{3 \delta}{2 \pi^2 (\epkidx+1)^2})\big\}}$, which can be satisfied when $T_0 \geq \Ocal(\max \curlybrackets{\underline{T}_{MC,1}(\delta), \underline{T}_{id, N}(\delta)})$.    
			\item Proposition \ref{proposition_epochRegret} requires no conditions on hyper-parameters.
		\end{itemize}
		Therefore, when $T_0 \geq \Ocal (\max \curlybrackets{\underline{T}_{\vx_0}(\delta), \underline{T}_{rgt, \bar{\epsilon}}(\delta, T_0),$ $ \underline{T}_{MC,1}(\delta), \underline{T}_{id, N}(\delta)}) =: \Ocal(\underline{T}_{rgt} (\delta, T_0))$, we can apply Propositions~\ref{proposition_epochInitState}, \ref{proposition_epochEstErrGuarantee}, \ref{proposition_epochsysid}, and \ref{proposition_epochRegret} to every epoch $\epkidx = 0, 1, \dots, M-1$. First note that Propositions \ref{proposition_epochInitState} and \ref{proposition_epochsysid} give the following
	\begin{gather*}    
		\P(\Dcal_\epkidx \mid \cap_{j=0}^{\epkidx-1} \Ecal_j) = \P(\Dcal_\epkidx \mid \Bcal_{\epkidx-1} \Dcal_{\epkidx-1}) > 1- \frac{3 \delta}{\pi^2 (\epkidx+1)^2}, \\
		\P(\Ccal_\epkidx \mid \cap_{j=0}^{\epkidx-1} \Ecal_j)
		=\P(\Ccal_\epkidx \mid \Bcal_{\epkidx-1}) \geq 1-\frac{3 \delta}{\pi^2 (\epkidx+1)^2}, \\ 
		  \P(\Dcal_0) \geq 1 - \frac{3 \delta}{\pi^2} \qquad \P(\Ccal_0) \geq 1 - \frac{3 \delta}{\pi^2}.
	\end{gather*}
	Then combining the probability bounds in Propositions \ref{proposition_epochInitState}, \ref{proposition_epochEstErrGuarantee}, \ref{proposition_epochsysid}, and \ref{proposition_epochRegret}, we have
	\begin{equation}\label{eq_111}
		\begin{split}
			&\P \parenthesesbig{\text{Regret bounds in }\eqref{eq_ExactRegretBound} \text{ holds}} \\
			& \geq \P ( \cap_{\epkidx = 0}^{M-1} \Ecal_\epkidx) \\
			& = \P( \Acal_M, \Bcal_{M-1}, \Ccal_{M-1}, \Dcal_{M-1} \mid \cap_{\epkidx = 0}^{M-2} \Ecal_\epkidx) \cdot \P(\cap_{\epkidx = 0}^{M-2} \Ecal_\epkidx ) \\
			& = \P( \Ccal_{M-1}, \Dcal_{M-1} \mid \cap_{\epkidx = 0}^{M-2} \Ecal_\epkidx) \cdot \P(\cap_{\epkidx = 0}^{M-2} \Ecal_\epkidx ) \\
			&  \geq  \parenthesesbig{1-\delta_{id, M-1} - \delta_{\vx_0, M-1}} \cdot \P(\cap_{\epkidx = 0}^{M-2} \Ecal_\epkidx ) \\
& \geq \prod_{\epkidx = 0}^{M-1}\parenthesesbig{1-\delta_{id, \epkidx} - \delta_{\vx_0, \epkidx}}\\
			& \geq 1 - \sum_{\epkidx = 0}^{M-1} (\delta_{id, \epkidx} + \delta_{\vx_0,\epkidx})\\
			& \geq  1 - \delta.
		\end{split}
	\end{equation}
	where the last line holds since $\sum_{\epkidx = 0}^{M-1} \frac{1}{(\epkidx+1)^2} \leq \frac{\pi^2}{6}$.
\end{proof}

\subsection{Regret Under Uniform Stability}
\subsubsection{Proof for Theorem \ref{thrm_regretUniformStability}} \label{appendix_regretUniformStability}
As we discussed in Section \ref{sec:logreg}, under MSS, the regret upper bound in Theorem \ref{thrm_mainthrm} (or the complete version Theorem \ref{thrm_mainthrm_complete}) involves $\frac{1}{\delta}$ dependency on failure probability $\delta$. By checking the proof for Theorem \ref{thrm_mainthrm_complete}, we can see the only source for $\frac{1}{\delta}$ is event $\Dcal_\epkidx$ in \eqref{eq_regretEventsDefUniformStability} and the corresponding Proposition \ref{proposition_epochInitState}, which provides $1-\delta$ probability bound for event $\Dcal_\epkidx$ -- the initial state $\vx_0^{(\epkidx+1)}$ of epoch $\epkidx+1$, alternately the final state $\vx_{T_\epkidx}^{(\epkidx)}$ of epoch $\epkidx$, is bounded by $\norm{\vx_0^{(\epkidx+1)}}^2 = \norm{\vx_{T_\epkidx}^{(\epkidx)}}^2 \leq \Ocal(\frac{1}{\delta})$. In Proposition \ref{proposition_epochInitState}, we get this bound using Markov inequality $\norm{\vx_{T_\epkidx}^{(\epkidx)}}^2 \leq \expctn[\norm{\vx_{T_\epkidx}^{(\epkidx)}}^2]/\delta$ and Lemma \ref{lemma_statebound} which provides an upper bound on the numerator $\expctn[\norm{\vx_{T_\epkidx}^{(\epkidx)}}^2]$ under MSS.
From event $\Acal_\epkidx$ in \eqref{eq_regretEventsDefUniformStability} we see the regret of epoch $\epkidx$ directly depends on its epoch initial state $\norm{\vx_0^{(\epkidx)}}^2$, thus in the final cumulative regret, the cumulative impact of initial states from all epochs, $\sum_\epkidx \norm{\vx_0^{(\epkidx)}}^2$ with order $\frac{1}{\delta}$, will show up, as given in \eqref{eq_regretmaintheormm_eq1}. Therefore, whether $\frac{1}{\delta}$ terms can be relaxed directly hinges on whether one could refine Proposition \ref{proposition_epochInitState} to get a tighter dependency on $\delta$.

This refinement, however, is not possible under the MSS assumption only, and we can easily construct a toy example to show that the $\frac{1}{\delta}$ dependency resulting from the Markov inequality cannot be improved. Consider a two-mode, one-dimensional, autonomous MJS:
$$
\left\{\begin{matrix}
	x_{t+1} = 2 x_{t} \\ 
	x_{t+1} = 0.5 x_{t}
\end{matrix}\right. \text{ with Markov matrix  }
\vT = \begin{bmatrix}
	0.1 & 0.9\\ 
	0.1 & 0.9
\end{bmatrix}
$$
with $x_0 \sim \N(0,1)$, and $\P(\omega(0)=1)=0.1$. It is easy to check this MJS is MSS by the spectral radius criterion discussed below Definition \ref{def_mss}. Also note that with probability $0.1^t$, $\omega(0:t-1)=1$ and $x_t = 2^t x_0$. Therefore, for any $a>0$, 
	\begin{align}
		&\P(x_t \geq a) = \sum_{\omega(0:t-1)} \P(x_t \geq a \mid \omega(0:t{-}1)) \P(\omega(0:t{-}1))\nn \\
		&\geq \P(x_t \geq a \mid \omega(0:t-1)=1) \P(\omega(0:t{-}1)=1)\nn \\
		 &= 0.1^t \cdot \P(x_0 \geq 2^{-t}a) \label{eq:mostimprobable}
	\end{align}
	where the inequality in \eqref{eq:mostimprobable} is extremely loose since we condition only on the most improbable event. 
	For standard Gaussian $x_0$, $\P(x_0 \geq a) \geq \frac{C}{a} \exp(-\frac{a^2}{2})$ for some absolute constant $C$. Thus $\P(x_t \geq a) \geq C  \frac{0.2^t}{a} \exp(-\frac{2^{-2t} a^2}{2})$. From this, we see that for any $a>0$, any $t \geq \log(a)/\log(2)$, we have $\P(x_t \geq a) \geq C  \frac{0.2^t}{\sqrt{e}a}$. We can observe that though when $t$ grows slower than $\log(a)$, the tail of $x_t$ has exponential decay, the Markov inequality decay, i.e. $\frac{1}{a}$, will eventually show up when $t$ gets larger. Interpretation from failure probability $\delta$ perspective is the following: letting $\delta = C  \frac{0.2^t}{\sqrt{e}a}$, we have $\P(x_t \leq C \frac{0.2^t}{\sqrt{e} \delta}) \leq 1-\delta$, which means any $\delta$ dependency lighter than $\frac{1}{\delta}$ must have probability less than $1-\delta$.
	This further implies that in the regret analysis of adaptive control, in order to obtain better probability dependency, the time horizon has to be limited, which greatly impairs its value in practice. 
	
	Intuitively, MSS assumption only provides us with stable behavior of $\norm{\vx_t}^2$ in the expectation (w.r.t. mode switchings) sense, and having only this first-order moment information  is of little use compared with the deterministic Lyapunov stability typically used for LTI systems, which allows one to bound $\norm{\vx_t}^2$ with only $\log(\frac{1}{\delta})$ dependence (\citep[Lemma C.5]{dean2018regret}). 
	Then, one may wonder naturally: Does there exist a deterministic version of stability for switched systems? Can this stability (if exists) help build similar dependence for switched systems? The answers to both questions are yes and will be discussed in this appendix. In short, if there exists uniform stability for the MJS, we can adapt Proposition \ref{proposition_epochInitState} such that $\norm{\vx_0^{(\epkidx)}}^2$ can instead be bounded much more tightly by $\norm{\vx_0^{(\epkidx)}}^2 \leq \Ocal(\log(\frac{1}{\delta}))$, thus the $\frac{1}{\delta}$ dependency can improve to $\log(\frac{1}{\delta})$ in the regret bound \eqref{eq_mainRegretUpperBd} (or \eqref{eq_ExactRegretBound}). The final improved regret bound is presented in Theorem \ref{thrm_regretUniformStability}. In order to show it, we will need to adapt Proposition \ref{proposition_epochInitState} together with several related results (Lemma \ref{lemma_statebound}, Lemma \ref{lemma_spectralradiusperturbation}, Lemma \ref{lemma_tac_1}) to the uniform stability case, and we append suffix ``a'' in the result label to denote the adapted versions. 
	
	$\vK_{1:\numSys}^\star$ is the optimal controller for infinite-horizon MJS-LQR($\vA_{1:\numSys},  \vB_{1:\numSys}, \vT, \vQ_{1:\numSys}, \vR_{1:\numSys}$) and define the closed-loop state matrix $\vL_i^\star = \vA_i + \vB_i \vK_i^\star$ for all $i$. We let $\theta^\star$ denote the joint spectral radius of $\vL_{1:\numSys}^\star$, i.e. $\theta^\star := \lim_{l \rightarrow \infty} \max_{\omega_{1:l} \in [\numSys]^{l}}\norm{\vL_{\omega_1}^\star \cdots \vL_{\omega_l}^\star}^{\frac{1}{l}}$. We say $\vL_{1:\numSys}^\star$ is uniformly stable if and only if $\theta^\star < 1$. Similar to Def. \ref{def_tau}, define $\kappa^\star {:=} \sup_{l\in \mathbb{N}} \max_{\omega_{1:l} \in [\numSys]^{l}}\norm{\vL_{\omega_1}^\star \cdots \vL_{\omega_l}^\star} / {(\theta^\star)}^l$. Note that the pair $\curlybrackets{\theta^\star, \kappa^\star}$ for uniform stability is just the counterpart of $\curlybrackets{\rho^\star, \tau(\vLtil^\star)}$ for MSS defined in Appendix \ref{appendix_MJSRegretAnalysis}. Similar as before, Table \ref{table_Notation_UniformStability} lists all the shorthand notations to be used in this appendix for quick reference.
	
	{ 
		\renewcommand{\arraystretch}{1.2}
		\begin{table}[!t]
			\caption{Notations --- Uniform Stability}
			\label{table_Notation_UniformStability}
			\centering
			\begin{adjustbox}{max width=\columnwidth}
				\begin{tabular}{||l||l||}
\hline
					\multirow{2}{*}{$\bar{\sigma}^2 $}
					& \qquad $ \norm{\vB_{1:\numSys}}^2 \norm{\vSigma_\vz} + \norm{\vSigma_\vw}$ \\
					& \qquad or $ \norm{\vB_{1:\numSys}}^2 \sigma_{\vz,0}^2 + \sigma_\vw^2$ \\    
					$\bar{\theta}
					$ & \qquad $ (1+\theta^\star)/2 $
					\\
					$\bar{\kappa}
					$ & \qquad $ \kappa^\star $
					\\        
					$\bar{\epsilon}^{us}_\vK
					$ & \qquad $ \frac{1-\rho^\star}{2 \kappa^\star \norm{\vB_{1:\numSys}}} $
					\\
					$\barbarepsilon _\vK
					$ & \qquad $ \min \curlybrackets{\bar{\epsilon}^{us}_\vK, \bar{\epsilon}_\vK} $
					\\
					$\barbarepsilon _{\vA, \vB, \vT}
					$ & \qquad $ \min \curlybrackets{\bar{\epsilon}_{\vA, \vB, \vT}, \frac{\barbarepsilon _\vK}{2 C_{\vA, \vB, \vT}^\vK}} $
					\\
					$\bar{x}^{us}
					$ & \qquad $ 2 \bar{\kappa}^2 \bar{\sigma}^2
					(6 \max \curlybrackets{\sqrt{\dimSt} e^{3 \dimSt}, \sqrt{\dimInput} e^{3 \dimInput}} + \frac{5}{(1-\bar{\theta})^2})^2 $
					\\
					$\underline{T}_{\vx_0}^{us}(\delta)
					$ & \qquad $ \max \curlybrackets{
						\frac{54 \bar{\kappa}^4 \bar{\sigma}^2}{(1 - \bar{\theta}) \bar{x}^{us} \log(1/\bar{\theta}) \log(\gamma)},
						\frac{1}{\gamma \log(1/\bar{\theta})} \log(6 \bar{\kappa}^2 + \frac{54 \dimSt \sqrt{\numSys} \bar{\kappa}^4 \bar{\sigma}^2 \log(\pi^2/3 \delta) }{(1-\bar{\theta})(1-\bar{\rho}) \bar{x}^{us} \delta})
					}$
					\\
					\multirow{2}{*}{$\underline{T}_{rgt, \bar{\epsilon}}^{us}(\delta, T)  $} & 
					\qquad $\Ocal( \log(\frac{1}{\delta}) \frac{(\dimSt + \dimInput)^2}{\pi_{\min}^2 (1-\varrho)^2\barbarepsilon_{\vA,\vB,\vT}^{4}}   \log^2(T))
					$ \\
					& \qquad $\Ocal( \log(\frac{1}{\delta}) \frac{(\dimSt + \dimInput)^2}{\pi_{\min}^2 (1-\varrho)^2\barbarepsilon_{\vA,\vB,\vT}^{2}}   \log(T))
					$ (when $\vB_{1:\numSys}$ is known)
\\
					$\underline{T}_{rgt}^{us} (\delta, T)
					$ & \qquad $ \max \curlybrackets{\underline{T}_{\vx_0}^{us}(\delta), \underline{T}_{rgt, \bar{\epsilon}}^{us}(\delta, T), \underline{T}_{MC,1}(\delta), \underline{T}_{id, N}(\delta)}$
					\\   
                    \hline
\end{tabular}
			\end{adjustbox}
		\end{table}
	}
	
	The following Lemma bounds the state $\vx_t$ under the designed input in this work. Compared with its counterpart Lemma \ref{lemma_statebound} which is only able to bound $\expctn[\norm{\vx_t}^2]$, \ref{lemma_stateboundUniformStability} provides high-probability bound for $\norm{\vx_t}^2$.
	
	\begin{customlemma}{Lemma~\ref{lemma_statebound}a}\label{lemma_stateboundUniformStability}
		Consider an $\textup{MJS}(\vA_{1:\numSys}, \vB_{1:\numSys}, \vT)$ with noise $\vw_t \sim \N(0, \vSigma_\vw)$. Consider controller $\vK_{1:\numSys}$, and let $\vL_{1:\numSys}$ denote the closed-loop state matrices with $\vL_i = \vA_i + \vB_i \vK_i$. Assume there exist constants $\kappa$ and $\theta \in [0,1)$ such that, for any sequence $\omega_{1:l} \in [\numSys]^{l}$ with any length $l$, $\norm{\vL_{\omega_1} \cdots \vL_{\omega_l}} \leq \kappa \theta^l$. Let the input be $\vu_t = \vK_{\omega_t} \vx_t + \vz_t$ with $\vz_t \sim \N(0, \Sigma_{\vz})$. Then, for any $t \geq e^{6 \max\curlybrackets{\dimSt, \dimInput}}$, with probability at least $1-\delta$, we have
		\begin{equation}
			\norm{\vx_t}^2 \leq 3 \kappa^2 \theta^{2t} \norm{\vx_0}^2 + \frac{18 \kappa^2 \bar{\sigma}^2 }{(1 - \theta)^2} \log(\frac{1}{\delta}) + c
		\end{equation}
		where $\bar{\sigma}^2 := \norm{\vSigma_\vw} + \norm{\vB_{1:\numSys}}^2 \norm{\vSigma_{\vz}} $
		and $c:= 2 \kappa^2 \bar{\sigma}^2$ $
		(6 \max \curlybrackets{\sqrt{\dimSt} e^{3 \dimSt}, \sqrt{\dimInput} e^{3 \dimInput}} + \frac{5}{(1-\theta)^2})^2$.
	\end{customlemma}
	\begin{proof}
		From the MJS dynamics \eqref{switched LDS} and plugging in the input $\vu_t = \vK_{\omega(t)} \vx_t + \vz_t$, we have the following. 
		\begin{equation}
			\begin{split}
				\vx_{t}
				&= \bigg(\prod_{h=0}^{t-1} \vL_{\omega(h)} \bigg) \vx_{0} + \sum_{i=0}^{t-2} \bigg( \prod_{h=i+1}^{t-1} \vL_{\omega(h)} \bigg) \vB_{\omega(i)} \vz_{i} 
				+ \vB_{\omega(t-1)} \vz_{t-1}\sum_{i=0}^{t-2} \bigg( \prod_{h=i+1}^{t-1} \vL_{\omega(h)} \bigg) \vw_{i} + \vw_{t-1}.
			\end{split}    
		\end{equation}
		Then, by triangle inequality and the assumption that $\norm{\vL_{\omega_1} \cdots \vL_{\omega_l}} \leq \kappa \theta^l$, we have
		\begin{equation}\label{eq_104}
			\begin{split}
				\norm{\vx_{t}} 
				&{\leq} 
				\kappa \theta^{t} \norm{\vx_0} 
				{+} \kappa \norm{\vB_{1:\numSys}} \sum_{i=0}^{t-1} \theta^{t-i-1} \norm{\vz_{i}}
				{+} \kappa \sum_{i=0}^{t-1} \theta^{t-i-1} \norm{\vw_{i}} \\
				&{=}
				\kappa \theta^{t} \norm{\vx_0} 
				{+} \kappa \norm{\vB_{1:\numSys}} \sum_{i=0}^{t-1} \theta^{i} \norm{\vz_{t-i-1}}
			    {+}\kappa \sum_{i=0}^{t-1} \theta^{i} \norm{\vw_{t-i-1}}.
			\end{split}
		\end{equation}
		For each $\vw_{t-i-1}$, using Lemma \ref{lemma_GaussianTailBound} (replacing $e^{-t}$ with $\delta_i$), we have with probability $1-\delta_i$,
		\begin{equation}
			\norm{\vw_{t-i-1}} \leq \sqrt{3 \norm{\vSigma_\vw}} \log^{0.5} (\frac{1}{\min \curlybrackets{\delta_i, \bar{\delta}_\dimSt}}),
		\end{equation}
		where $\bar{\delta}_\dimSt:= e^{-(3+2\sqrt{2}) \dimSt}$, and $n$ is the dimension of vector $\vw_{t-i-1}$. In the following, for all $i=0,1, \dots, t-1$, we set $\delta_i = \frac{3}{\pi^2} \frac{\delta}{(i+1)^2}$. First note that  when $i \geq \bar{i} :=  \sqrt{\frac{3 \delta}{\pi^2 \bar{\delta}_n}} - 1$, we have $\min \curlybrackets{\delta_i, \bar{\delta}_\dimSt} = \delta_i$, i.e.  $\delta_i \leq \bar{\delta}_n$, and $\min \curlybrackets{\delta_i, \bar{\delta}_\dimSt} = \bar{\delta}_\dimSt$ otherwise. Then, applying union bound for all $i$, we know with probability at least $1-\frac{\delta}{2}$,
		\begin{equation}\label{eq_101}
			\begin{split}
				\sum_{i=0}^{t-1} \theta^{i} \norm{\vw_{t-i-1}}
				&\leq \sqrt{3 \norm{\vSigma_\vw}} \sum_{i=0}^{t-1} \theta^{i}  \log^{0.5} (\frac{1}{\min \curlybrackets{\delta_i, \bar{\delta}_\dimSt}}) \\
				&\leq \sqrt{3 \norm{\vSigma_\vw}} \parenthesesbig{\sum_{i=0}^{t-1} \theta^{i}  \log^{0.5} (\frac{1}{\delta_i}) + (\bar{i}+1) \log^{0.5} (\frac{1}{ \bar{\delta}_\dimSt}) }.
			\end{split}
		\end{equation}
		For $\sum_{i=0}^{t-1} \theta^{i}  \log^{0.5} (\frac{1}{\delta_i})$, we have
		$
		\sum_{i=0}^{t-1} \theta^{i}  \log^{0.5} (\frac{1}{\delta_i})$ $
		= \sum_i \theta^{i} \log^{0.5} (\frac{\pi^2 (i+1)^2}{3 \delta})
		\leq \sum_i \theta^{i} (\log^{0.5}(\frac{1}{\delta}) + \sqrt{2} \log^{0.5}$ $(\frac{\pi(i+1)}{\sqrt{3}}))
		\leq \frac{1}{1-\theta} \log^{0.5}(\frac{1}{\delta}) + \sqrt{2} \sum_i \theta^i \frac{\pi(i+1)}{\sqrt{3}}
		\leq \frac{1}{1-\theta} \log^{0.5}(\frac{1}{\delta})$ $+ \frac{\sqrt{2}\pi}{\sqrt{3}} \frac{1}{(1-\theta)^2}
		$. And for the term $(\bar{i}+1) \log^{0.5} (\frac{1}{ \bar{\delta}_\dimSt})$ in \eqref{eq_101}, by the definitions of $\bar{i}$ and $\bar{\delta}_\dimSt$, we have $(\bar{i}+1) \log^{0.5} (\frac{1}{ \bar{\delta}_\dimSt}) \leq \sqrt{2 \dimSt} e^{3 \dimSt}$. Plugging these two results back into \eqref{eq_101}, we have, with probability at least $1 - \frac{\delta}{2}$,
		\begin{equation}\label{eq_102}
			\begin{split}
			\sum_{i=0}^{t-1} \theta^{i} \norm{\vw_{t-i-1}}
			&\leq \frac{\sqrt{3 \norm{\vSigma_\vw}}}{1 - \theta} \log^{0.5}(\frac{1}{\delta}) 
			+ \frac{5 \sqrt{\norm{\vSigma_\vw}} }{(1 - \theta)^2}
			+ 3 \sqrt{\dimSt} e^{3 \dimSt} \sqrt{\norm{\vSigma_\vw}}.        
			\end{split}
		\end{equation}
		Similarly, with probability at least $1 - \frac{\delta}{2}$,
		\begin{equation}\label{eq_103}
			\begin{split}
			\sum_{i=0}^{t-1} \theta^{i} \norm{\vz_{t-i-1}}
			&\leq \frac{\sqrt{3 \norm{\vSigma_\vz}}}{1 - \theta} \log^{0.5}(\frac{1}{\delta}) 
			+ \frac{5 \sqrt{\norm{\vSigma_\vz}} }{(1 - \theta)^2}
			+ 3 \sqrt{\dimInput} e^{3 \dimInput} \sqrt{\norm{\vSigma_\vz}}.
			\end{split}
		\end{equation}
		Plugging \eqref{eq_102} and \eqref{eq_103} back into \eqref{eq_104} and applying union bound, we have, with probability $1-\delta$,
		\begin{align*}
			\norm{\vx_t} 
			&\leq \kappa \theta^t \norm{\vx_0}
			+ \frac{\sqrt{3} \kappa (\sqrt{\norm{\vSigma_\vw}} + \norm{\vB_{1:\numSys}} \sqrt{\norm{\vSigma_\vz}} )}{(1 - \theta)^2}\log^{0.5}(\frac{1}{\delta})\\ 
			&+ \kappa (\sqrt{\norm{\vSigma_\vw}} + \norm{\vB_{1:\numSys}} \sqrt{\norm{\vSigma_\vz}} )
			\parenthesesbig{3 \max \curlybrackets{\sqrt{\dimSt} e^{3 \dimSt}, \sqrt{\dimInput} e^{3 \dimInput}}
				+ \frac{5}{(1 - \theta)^2}}.
		\end{align*}
		Taking squares of both sides and using Cauchy-Schwartz inequality, we have
		\begin{equation}
			\norm{\vx_t}^2 \leq 3 \kappa^2 \theta^{2t} \norm{\vx_0}^2 + \frac{18 \kappa^2 \bar{\sigma}^2 }{1-\theta} \log(\frac{1}{\delta}) + c
		\end{equation}
		where $\bar{\sigma}^2 := \norm{\vSigma_\vw} + \norm{\vB_{1:\numSys}}^2 \norm{\vSigma_{\vz}} $
		and $c:= 6 \kappa^2 \bar{\sigma}^2$ $
		(3 \max \curlybrackets{\sqrt{\dimSt} e^{3 \dimSt}, \sqrt{\dimInput} e^{3 \dimInput}} + \frac{5}{(1-\theta)^2})^2$.
	\end{proof}
	
	The following Lemma describes that given a set of matrices that have joint spectral radius smaller than $1$, i.e. uniformly stable, moderate perturbation can preserve the uniform stability. On the other hand, its counterpart, Lemma \ref{lemma_spectralradiusperturbation}, considers perturbation results for MSS.
	\begin{customlemma}{Lemma~\ref{lemma_spectralradiusperturbation}a}[Joint Spectral Radius] \label{lemma_JSRperturbation}
		Assume $\theta^\star < 1$. 
		For an arbitrary controller $\vK_{1:\numSys}$ and resulting closed-loop state matrices $\vL_{1:\numSys}$ with $\vL_i = \vA_i + \vB_i \vK_i$, let $\theta(\vL_{1:\numSys})$ denote the joint spectral radius of $\vL_{1:\numSys}$.
		Assume $\norm{\vK_{1:\numSys} - \vK_{1:\numSys}^\star} \leq \bar{\epsilon}^{us}_\vK := \frac{1-\theta^\star}{2 \kappa^\star \norm{\vB_{1:\numSys}}}$, then for any sequence $\omega_{1:l} \in [\numSys]^{l}$ with any length $l$,
		\begin{align}
			\norm{\prod_{j=1}^l \vL_{\omega_j}} & \leq \bar{\kappa} \bar{\theta}^l \label{eq_JSRperturbednorm} \\
			\theta(\vL_{1:\numSys}) & \leq \bar{\theta}. \label{eq_JSRperturbed}
		\end{align}
		where $\bar{\kappa} = \kappa^\star$ and $\bar{\theta} = \frac{1+\theta^\star}{2}$.
	\end{customlemma}
	\begin{proof}
		Let $\vE_i:= \vL_i - \vL_i^\star$, then we see $\norm{\vE_i} \leq \norm{\vB_{1:\numSys}} \bar{\epsilon}^{us}_\vK$ and $\prod_{j=1}^{l} \vL_{\omega_j} = \prod_{j=1}^{l}(\vL_{\omega_j}^\star + \vE_{\omega_j})$. In the expansion of $\prod_{j=1}^{l}(\vL_{\omega_j}^\star + \vE_{\omega_j})$, for each $h=0, 1, \dots, l$, there are $\binom{l}{h}$ terms, each of which is a product where $\vE$ has degree $h$ and $\vL^\star$ has degree $l-h$. We let $\vF_{h,l}$ with $h = 0,1,\dots, l$ and $l \in [\binom{l}{h}]$ to index such terms. Note that $\norm{\vF_{h,l}} \leq (\kappa^\star)^{h+1} (\theta^\star)^{l-h} (\norm{\vB_{1:\numSys}} \bar{\epsilon}^{us}_\vK)^h$. Then, we have
		\begin{equation}
			\begin{split}
				\norm{\prod_{j=1}^{l} \vL_{\omega_j}}
				\leq \sum_{h=0}^l \sum_{l \in [\binom{l}{h}]} \norm{\vF_{h,l}} 
				\leq& \sum_{h=0}^l \binom{l}{h} (\kappa^\star)^{h+1} (\theta^\star)^{l-h} (\norm{\vB_{1:\numSys}} \bar{\epsilon}^{us}_\vK)^h \\
				\leq& \kappa^\star (\kappa^\star \norm{\vB_{1:\numSys}} \bar{\epsilon}^{us}_\vK + \theta^\star)^l.
			\end{split}
		\end{equation}
		Then \eqref{eq_JSRperturbednorm} follows from the fact that $\bar{\epsilon}_{\vK}^{us} \leq \frac{1-\theta^\star}{2 \kappa^\star \norm{\vB_{1:\numSys}}}$ and $\bar{\theta} := \frac{1+\theta^\star}{2}$. To proceed, noticing that $\theta(\vL_{1:\numSys}) = \lim_{l \rightarrow \infty} \max_{\omega_{1:l} \in [\numSys]^{l}}\norm{\prod_{j=1}^l \vL_{\omega_j}}^{\frac{1}{l}}$ and using the result in \eqref{eq_JSRperturbednorm}, we can show \eqref{eq_JSRperturbed}.
	\end{proof}
	In the \ref{lemma_JSRperturbation}, if the controller $\vK_{1:\numSys}$ is obtained by solving the infinite-horizon MJS-LQR($\splitatcommas{\vAhat_{1:\numSys}, \vBhat_{1:\numSys}, \vThat, \vQ_{1:\numSys}, \vR_{1:\numSys}}$) for some estimated MJS($\vAhat_{1:\numSys}, \vBhat_{1:\numSys}, \vThat$), the following result provides the required estimation accuracy such that the resulting $\vK_{1:\numSys}$ is uniformly stabilizing.
	\begin{customlemma}{Lemma~\ref{lemma_tac_1}a}\label{lemma_lemma_tac_1a}
		Under the setup of Lemma \ref{lemma_tac_1}, if $\max \curlybrackets{\bar{\epsilon}_{\vA, \vB}, \bar{\epsilon}_\vT} \leq \barbarepsilon _{\vA, \vB, \vT}$, then we have $\norm{\vK_{1:\numSys} - \vKstar_{1:\numSys}} \leq \bar{\epsilon}_{\vK}$, and \ref{lemma_JSRperturbation} is applicable.
	\end{customlemma}
	
	Recall we defined events $\Acal_\epkidx, \Bcal_\epkidx, \Ccal_\epkidx, \Dcal_\epkidx$ in \eqref{eq_regretEventsDef} to analyze the events happen in each epoch of the regret. To adapt to the uniform stability assumption, we redefine event $\Bcal_\epkidx$ and $\Dcal_\epkidx$ while keep $\Acal_\epkidx$ and $\Ccal_\epkidx$ as before. For easier reference, We list all of them below.
	
	\begin{equation}\label{eq_regretEventsDef}
		\begin{split}
			\Acal_\epkidx 
			&{=} \bigg\{ 
			\text{Regret}_\epkidx {\leq}  \Ocal\bigg(\numSys \dimInput \left( \epsilon_{\vA, \vB}^{(\epkidx-1)} {+} \epsilon_{\vT}^{(\epkidx-1)} \right)^2 \sigma_\vw^2 T_\epkidx
			{+} \sqrt{\dimSt \numSys} \norm{\vx_0^{(\epkidx)}}^2 
			{+} \frac{\dimSt \sqrt{\numSys}}{1-\rho^\star} \sigma_{\vz,\epkidx}^2 T_\epkidx
			{+} c_{\Acal} \bigg) 
			\bigg\} \\
			\Bcal_\epkidx 
			&{=} \left\{ \epsilon_{\vA, \vB}^{(\epkidx)} {\leq} \barbarepsilon _{\vA, \vB, \vT}, \epsilon_{\vT}^{(\epkidx)} {\leq} \barbarepsilon _{\vA, \vB, \vT}, 
			\epsilon_{\vK}^{(\epkidx+1)} {\leq} \barbarepsilon _{\vK} \right\}, \forall \epkidx  \\
			\Ccal_\epkidx 
			&{=} \bigg\{
			\epsilon_{\vA, \vB}^{(\epkidx)} \leq \Ocal\parenthesesbig{ \log(\frac{1}{\delta_{id, \epkidx}}) \frac{\sigma_{\vz,\epkidx} {+} \sigma_\vw}{\sigma_{\vz,\epkidx}} \mysqrt[1pt]{\frac{(n{+}p)\log(T_q)}{\pi_{\min}(1 {-} \varrho)T_q}} } , \\
			& \qquad \ \ \epsilon_{\vT}^{(\epkidx)} {\leq} \Ocal \bigg( \log(\frac{1}{\delta_{id, \epkidx}})  \frac{1}{\pi_{\min}}\mysqrt[1pt]{\frac{\log(T_\epkidx)}{T_\epkidx}} \bigg) \bigg\}\\
			\Dcal_\epkidx 
			&{=} \curlybracketsbig{\norm{\vx^{(\epkidx+1)}_{0}}^2 {=} \norm{\vx^{(\epkidx)}_{T_\epkidx}}^2 {\leq} \frac{18 \bar{\kappa}^2 \bar{\sigma}^2 }{(1 {-} \bar{\theta})^2} \log(\frac{1}{\delta_{\vx_0, \epkidx}}) {+} 2\bar{x}^{us}}, \forall \epkidx \\
			\Dcal_0
			&{=} \curlybracketsbig{\norm{\vx^{(1)}_{0}}^2 {=} \norm{\vx^{(0)}_{T_{0}}}^2 \leq \frac{\dimSt \sqrt{\numSys} \bar{\tau} \bar{\sigma}^2 / (1 {-} \bar{\rho})}{\delta_{\vx_0, 0}} }, 
		\end{split}
	\end{equation}
	where we define the terms $\bar{x}^{us}:= 2 \bar{\kappa}^2 \bar{\sigma}^2
	(6 \max \curlybrackets{\sqrt{\dimSt} e^{3 \dimSt}, \sqrt{\dimInput} e^{3 \dimInput}} + \frac{5}{(1-\bar{\theta})^2})^2$, $\barbarepsilon _\vK := 
	\min \curlybrackets{\bar{\epsilon}^{us}_\vK, \bar{\epsilon}_\vK} $
	, $\barbarepsilon _{\vA, \vB, \vT} :=
	\min \{\bar{\epsilon}_{\vA, \vB, \vT},$ $\frac{\barbarepsilon _\vK}{2 C_{\vA, \vB, \vT}^\vK}\} $ and $\bar{\sigma}^2:= \norm{\vB_{1:\numSys}}^2 \sigma_{\vz,0}^2 + \sigma_\vw^2$. Event $\Dcal_\epkidx$ describes the initial state magnitude of epoch $\epkidx+1$. Since Algorithm \ref{Alg_adaptiveMJSLQR} requires initial MSS stabilizing controller $\vK_{1:\numSys}^{(0)}$ for epoch $0$, and as in the proof for the following \ref{proposition_epochInitStateUniformStability}, epoch $1,2, \dots$ have uniformly stabilizing controller, thus we define $\Dcal_0$ and $\Dcal_1, \Dcal_2, \dots$ separately.

	\begin{customproposition}{Proposition \ref{proposition_epochInitState}a} \label{proposition_epochInitStateUniformStability}
		Assuming that $T_\epkidx \geq \frac{1}{2  \log(1/\bar{\theta})}$ $\log \parenthesesbig{6 \bar{\kappa}^2 + \frac{54 \bar{\kappa}^4 \bar{\sigma}^2}{(1 - \bar{\theta}) \bar{x}^{us}} \log(\frac{1}{\delta_{\vx_0, \epkidx-1}})}$ and $T_1 \geq \frac{1}{2  \log(1/\bar{\theta})}$ $\log \parenthesesbig{\frac{3 \dimSt \sqrt{\numSys} \bar{\kappa}^2 \bar{\tau} \bar{\sigma}^2}{(1-\bar{\rho}) \bar{x}^{us} \delta_{\vx_0, 0} }}$, we have 
		\begin{equation}
			\P(\Dcal_\epkidx \mid \Bcal_{\epkidx-1}, \Dcal_{\epkidx-1}) \geq 1 - \delta_{\vx_0,\epkidx}   
		\end{equation}
		and $\P(\Dcal_0) \geq 1-\delta_{\vx_0, 0}$.
	\end{customproposition}
	\begin{proof}
		For the initial epoch $0$, i.e. $\epkidx = 0$, since we assume in Algorithm \ref{Alg_adaptiveMJSLQR} that the initial controller $\vK_{1:\numSys}^{(0)}$ stabilizes the MJS in the mean-squared sense, similar to the proof for Proposition \ref{proposition_epochInitState}, we have $\expctn[\norm{\vx^{(0)}_{T_0}}^2 ] \leq \dimSt \sqrt{\numSys} (\norm{\vB_{1:\numSys}}^2 \sigma_{\vz,0}^2 + \sigma_{\vw}^2) \frac{ \bar{\tau}}{1-\bar{\rho}}$. Then by Markov inequality, with probability $1-\delta_{\vx_0, 0}$, $\norm{\vx^{(0)}_{T_0}}^2 \leq \frac{\dimSt \sqrt{\numSys} \bar{\tau} \bar{\sigma}^2 / (1 - \bar{\rho})}{\delta_{\vx_0, 0}}$ where $\bar{\sigma}^2:= \norm{\vB_{1:\numSys}}^2 \sigma_{\vz,0}^2 + \sigma_{\vw}^2$. This shows $\P(\Dcal_0) \geq 1-\delta_{\vx_0, 0}$.
		
		For epoch $\epkidx =  1,2, \dots$, given event $\Bcal_{\epkidx-1}$, we know $\epsilon_{\vK}^{(\epkidx)} \leq \barbarepsilon_{\vK} \leq \bar{\epsilon}^{us}_\vK$. Let $\vL_{1:\numSys}^{(\epkidx)}$ denote the closed-loop state matrices for epoch $\epkidx$, then by \ref{lemma_JSRperturbation}, $\epsilon_{\vK}^{(\epkidx)} \leq \bar{\epsilon}^{us}_\vK$ implies that for any $l$ and any sequence $\omega_{1:l} \in [\numSys]^{l}$, $\norm{\prod_{j=1}^l \vL_{\omega_j}^{(\epkidx)}} \leq \bar{\kappa} \bar{\theta}^l$. Then using the bound on $\norm{\vx_t}$ in \ref{lemma_stateboundUniformStability}, we have, with probability $1-\delta_{\vx_0, \epkidx}$,
		\begin{equation}\label{eq_105}
			\norm{\vx_{T_\epkidx}^{(\epkidx)}}^2 \leq \frac{18 \bar{\kappa}^2 \bar{\sigma}^2 }{(1 - \bar{\theta})^2} \log(\frac{1}{\delta_{\vx_0, \epkidx}}) + 3 \bar{\kappa}^2 \bar{\theta}^{2 T_\epkidx} \norm{\vx_0^{(\epkidx)}}^2 +  \bar{x}^{us}
		\end{equation}
		where $\bar{x}^{us}:= 2 \bar{\kappa}^2 \bar{\sigma}^2
		(6 \max \curlybrackets{\sqrt{\dimSt} e^{3 \dimSt}, \sqrt{\dimInput} e^{3 \dimInput}} + \frac{5}{(1-\bar{\theta})^2})^2$. 
		\begin{itemize}
			\item When $\epkidx = 1$, given $D_0$, i.e. $\norm{\vx^{(1)}_{0}}^2 \leq \frac{\dimSt \sqrt{\numSys} \bar{\tau} \bar{\sigma}^2 / (1 - \bar{\rho})}{\delta_{\vx_0, 0}}$, \eqref{eq_105} gives
			$
			\norm{\vx_{T_1}^{(1)}}^2 \leq \frac{18 \bar{\kappa}^2 \bar{\sigma}^2 }{(1 - \bar{\theta})^2} \log(\frac{1}{\delta_{\vx_0, 1}}) + 3 \bar{\kappa}^2 \bar{\theta}^{2 T_1} \frac{\dimSt \sqrt{\numSys} \bar{\tau} \bar{\sigma}^2 / (1 - \bar{\rho})}{\delta_{\vx_0, 0}} + \bar{x}^{us}.
			$
			One can check that when we choose $T_1 \geq \frac{1}{2 \log(1/\bar{\theta})} \log \parenthesesbig{\frac{3 \dimSt \sqrt{\numSys} \bar{\kappa}^2 \bar{\tau} \bar{\sigma}^2}{(1-\bar{\rho}) \bar{x}^{us} \delta_{\vx_0, 0} }}$, we have that $3 \bar{\kappa}^2 \bar{\theta}^{2 T_1} \frac{\dimSt \sqrt{\numSys} \bar{\tau} \bar{\sigma}^2 / (1 - \bar{\rho})}{\delta_{\vx_0, 0}} \leq \bar{x}^{us}$, which gives
			\begin{equation}\label{eq_106}
				\norm{\vx_{T_1}^{(1)}}^2 \leq \frac{18 \bar{\kappa}^2 \bar{\sigma}^2 }{(1 - \bar{\theta})^2} \log(\frac{1}{\delta_{\vx_0, 1}}) + 2\bar{x}^{us}.
			\end{equation}
			\item When $\epkidx = 2, 3, \dots$, given event $\Dcal_{\epkidx-1}$, i.e. $\norm{\vx^{(\epkidx)}_{0}}^2 \leq \frac{18 \bar{\kappa}^2 \bar{\sigma}^2 }{(1 - \bar{\theta})^2} \log(\frac{1}{\delta_{\vx_0, \epkidx-1}}) + 2\bar{x}^{us}$, the above \eqref{eq_105} gives
			$
			\norm{\vx_{T_\epkidx}^{(\epkidx)}}^2 \leq \frac{18 \bar{\kappa}^2 \bar{\sigma}^2 }{(1 - \bar{\theta})^2} \log(\frac{1}{\delta_{\vx_0, \epkidx}}) + 3 \bar{\kappa}^2 \bar{\theta}^{2 T_\epkidx} \big(\frac{18 \bar{\kappa}^2 \bar{\sigma}^2 }{(1 - \bar{\theta})^2}$ $ \log(\frac{1}{\delta_{\vx_0, \epkidx-1}}) + 2\bar{x}^{us}\big)+ \bar{x}^{us}.
			$
			Similarly, when $T_\epkidx \geq \frac{1}{2 \log(1/\bar{\theta})} \log \parenthesesbig{6 \bar{\kappa}^2 + \frac{54 \bar{\kappa}^4 \bar{\sigma}^2}{(1 - \bar{\theta}) \bar{x}^{us}} \log(\frac{1}{\delta_{\vx_0, \epkidx-1}})}$, we further have
			\begin{equation}\label{eq_107}
				\norm{\vx_{T_\epkidx}^{(\epkidx)}}^2 \leq \frac{18 \bar{\kappa}^2 \bar{\sigma}^2 }{(1 - \bar{\theta})^2} \log(\frac{1}{\delta_{\vx_0, \epkidx}}) + 2\bar{x}^{us}.
			\end{equation}    
		\end{itemize}
		Combining \eqref{eq_106} and \eqref{eq_107}, for epoch $\epkidx = 1,2,\dots$, when  $T_1 \geq \frac{1}{2  \log(1/\bar{\theta})} \log \parenthesesbig{\frac{3 \dimSt \sqrt{\numSys} \bar{\kappa}^2 \bar{\tau} \bar{\sigma}^2}{(1-\bar{\rho}) \bar{x}^{us} \delta_{\vx_0, 0} }}$ and $T_\epkidx \geq \frac{1}{2  \log(1/\bar{\theta})} \log \parenthesesbig{6 \bar{\kappa}^2 + \frac{54 \bar{\kappa}^4 \bar{\sigma}^2}{(1 - \bar{\theta}) \bar{x}^{us}} \log(\frac{1}{\delta_{\vx_0, \epkidx-1}})}$, we have $\P(\Dcal_\epkidx \mid \Bcal_{\epkidx-1}, \Dcal_{\epkidx-1}) \geq 1 - \delta_{\vx_0,\epkidx}$.
	\end{proof}
	
	The following \ref{proposition_epochEstErrGuaranteeUniformStability} says that if a good controller is used in epoch $\epkidx$, then the final state $x_{T_\epkidx}^{(\epkidx)}$ of epoch $\epkidx$ (the initial state of epoch $\epkidx+1$) can be bounded.
	\begin{customproposition}{Proposition \ref{proposition_epochEstErrGuarantee}a}\label{proposition_epochEstErrGuaranteeUniformStability}
		Suppose every epoch $\epkidx$ has length $T_\epkidx \geq \underline{T}_{rgt, \bar{\epsilon}}^{us}(\delta_{id, \epkidx}, T_\epkidx)$. Then,
		\begin{equation}
			\P(\Bcal_\epkidx \mid \Ccal_\epkidx, \cap_{j=0}^{\epkidx-1} \Ecal_j) = \P(\Bcal_\epkidx \mid \Ccal_\epkidx) = 1
		\end{equation}
	\end{customproposition}
	Now, we are ready to present the main proof of Theorem \ref{thrm_regretUniformStability}.
\begin{theorem}[Complete version of Thm. \ref{thrm_regretUniformStability}]\label{thrm_mainthrm_completeUniformStability}
		Assume that the initial state $\vx_0 = 0$, and Assumption~\ref{asmp_ergodicity_MSS} hold,
			and $\vL_{1:\numSys}^\star$ is \emph{uniformly stable}.
			Suppose $T_0 \geq \Ocal(\underline{T}_{rgt}^{us} (\delta, T_0))$.
			Then, with probability at least $1-\delta$, Algorithm \ref{Alg_adaptiveMJSLQR} achieves
			\begin{equation}\label{eq_ExactRegretBoundUniformStability}
				\begin{split}
				&\textnormal{Regret}(T)
				\leq \Ocal \parenthesesbig{
					\frac{\numSys \dimInput (\dimSt + \dimInput) \sigma_\vw^2 }{\pi_{\min} (1-\varrho \lor \rho^\star)} \log\big(\frac{\log^2(T)}{\delta}\big) \log(T) \sqrt{T} }.
				\end{split}
		\end{equation}
	\end{theorem}
	\begin{proof}
		The proof is almost the same as the proof for the MSS regret upper bound in Theorem \ref{thrm_mainthrm_complete} in Appendix \ref{appendix_proofforregret}, thus we only present the key steps and omit certain details of intermediate steps. 
			
			In the following, we set $\delta_{id, \epkidx} = \delta_{\vx_0, \epkidx} = \frac{3}{\pi^2} \cdot \frac{\delta}{ (\epkidx+1)^2}$. Similar to the counterpart \eqref{eq_108}, event $\Ecal_\epkidx = \Acal_{\epkidx+1} \cap \Bcal_\epkidx \cap \Ccal_\epkidx \cap \Dcal_\epkidx$ implies the following: for $\epkidx = 1,2,\dots$,
			\begin{equation}\label{eq_109}
				\begin{split} 
					\text{Regret}_{\epkidx+1} 
					&\leq  \Ocal(1)\log \big(\frac{(\epkidx{+}1)^2}{\delta} \big) \numSys \dimInput \bigg(\frac{\sigma_{\vz,\epkidx} {+} \sigma_\vw}{\sigma_{\vz,\epkidx}} 
					 \mysqrt[1pt]{\frac{(n{+}p)\log(T_q)}{\pi_{\min}(1 {-} \varrho)T_q}} {+} \frac{\sqrt{\log(T_\epkidx)}}{\pi_{\min}  \sqrt{T_\epkidx}}\bigg)^2 \sigma_\vw^2 T_{\epkidx+1} \\
                    &+ \Ocal(1) \log\big(\frac{\epkidx{+}1}{\delta}\big) 
					\frac{18 \sqrt{\dimSt \numSys} \bar{\kappa}^2 \bar{\sigma}^2 }{(1 {-} \bar{\theta})^2} {+} \Ocal(\frac{\dimSt \sqrt{\numSys}}{1{-}\rho^\star} \sigma_{\vz,\epkidx+1}^2 T_{\epkidx+1}) {+} \Ocal(1) \\
					&\leq \Ocal(1) \log\big(\frac{(\epkidx{+}1)^2}{\delta}\big)
					\frac{\numSys \dimInput (\dimSt {+} \dimInput) \gamma}{\pi_{\min} (1{-}\varrho \lor \rho^\star)} \sigma_\vw^2 \sqrt{T_\epkidx} \log(T_\epkidx) \\
					&+ \Ocal(1)  \log\big(\frac{(\epkidx{+}1)^2}{\delta}\big) \frac{18 \sqrt{\dimSt \numSys} \bar{\kappa}^2 \bar{\sigma}^2 }{(1 {-} \bar{\theta})^2},
				\end{split}
			\end{equation}
			and for $\epkidx = 0$,
			\begin{equation}\label{eq_110}
				\begin{split}
				\text{Regret}_{1}
				&{\leq} \Ocal(1) \log\big(\frac{1}{\delta}\big)
				\frac{\numSys \dimInput (\dimSt {+} \dimInput) \gamma}{\pi_{\min} (1{-}\varrho \lor \rho^\star)} \sigma_\vw^2 \sqrt{T_0} \log(T_0) 
				{+} \Ocal(1) \big(\frac{1}{\delta}\big) \frac{\dimSt^{1.5} \numSys \bar{\tau} \bar{\sigma}^2}{1 {-} \bar{\rho}}.
				\end{split}			
			\end{equation}
			Note that the difference between \eqref{eq_109} ($\epkidx = 1,2,\dots$) and \eqref{eq_110} ($\epkidx = 0$) is due to the difference between the event $\Dcal_\epkidx$ for $\epkidx = 1,2,\dots$ and event $\Dcal_0$. Compared with the MSS counterpart \eqref{eq_108}, we see the $\frac{(\epkidx+1)^2}{\delta}$ dependence in \eqref{eq_108} is now replaced with $\log\big(\frac{\epkidx{+}1}{\delta}\big)$. For all $M:= \Ocal(\log_\gamma (\frac{T}{T_0}))$ epochs, similar to the counterpart \eqref{eq_regretmaintheormm_eq1}, 
			event $\cap_{\epkidx = 0}^{M-1} \Ecal_\epkidx$ implies
			\begin{equation}\nn
				\begin{split}
					&\text{Regret}(T) = \Ocal(\sum_{\epkidx = 1}^{M} \text{Regret}_\epkidx) \\
					&\leq  \Ocal \bigg( \frac{\numSys \dimInput (\dimSt + \dimInput) \sigma_\vw^2 }{\pi_{\min} (1-\varrho \lor \rho^\star)} \log\big(\frac{\log^2(T)}{\delta}\big) \sqrt{T} \log(T) 
					{+}
					\frac{18 \sqrt{\dimSt \numSys} \bar{\kappa}^2 \bar{\sigma}^2 }{(1 - \bar{\theta})^2} \log\big(\frac{\log^2(T)}{\delta}\big) \log(T) \bigg)
				\end{split}    
			\end{equation}
			which shows the main result \eqref{eq_ExactRegretBoundUniformStability}. Note that in the above summation, we have omit $\frac{1}{\delta}$ term in $\text{Regret}_1$ since it does not scale with time and can be dominated by the rest.
			
			Now we are only left to show the occurrence probability of regret bound \eqref{eq_ExactRegretBoundUniformStability} is larger than $1 - \delta$. To do this, we will combine \ref{proposition_epochInitStateUniformStability}, \ref{proposition_epochEstErrGuaranteeUniformStability}, Proposition \ref{proposition_epochsysid}, and Proposition \ref{proposition_epochRegret} over all $\epkidx = 0, 1, \dots, M-1$. Note that for each individual $\epkidx$, these propositions hold only when certain prerequisite conditions on hyper-parameters $c_\vx$, $c_\vz$, and $T_0$ are satisfied. 
			We first show that under the choices $T_\epkidx = \gamma T_{\epkidx-1}$, $\sigma_{\vz,\epkidx}^2 = \frac{\sigma_\vw^2}{\sqrt{T_\epkidx}}$, and $\delta_{id, \epkidx} = \delta_{\vx_0, \epkidx} = \frac{3}{\pi^2} \cdot \frac{\delta}{ (\epkidx+1)^2}$ these hyper-parameter conditions can be satisfied for all $\epkidx = 0, 1, \dots, M-1$.
			
			\begin{itemize}
				\item \ref{proposition_epochInitStateUniformStability} requires these to hold: $T_0 \gamma^\epkidx \geq \frac{1}{2  \log(1/\bar{\theta})} \log \parenthesesbig{6 \bar{\kappa}^2 + \frac{54 \bar{\kappa}^4 \bar{\sigma}^2}{(1 - \bar{\theta}) \bar{x}^{us}} \log(\frac{i^2 \pi^2}{3\delta})}$ and $T_0 \gamma \geq \frac{1}{2  \log(1/\bar{\theta})} \log \parenthesesbig{\frac{\pi^2 \dimSt \sqrt{\numSys} \bar{\kappa}^2 \bar{\tau} \bar{\sigma}^2}{(1-\bar{\rho}) \bar{x}^{us} \delta }}$. 
				One can check that,
				$\splitatcommas{T_0 \geq \max \curlybrackets{
						\frac{54 \bar{\kappa}^4 \bar{\sigma}^2}{(1 - \bar{\theta}) \bar{x}^{us} \log(1/\bar{\theta}) \log(\gamma)},
						\frac{1}{\gamma \log(1/\bar{\theta})} \log(6 \bar{\kappa}^2 + \frac{54 \dimSt \sqrt{\numSys} \bar{\kappa}^4 \bar{\sigma}^2 \log(\pi^2/3 \delta) }{(1-\bar{\theta})(1-\bar{\rho}) \bar{x}^{us} \delta})
				}} =: \underline{T}_{\vx_0}^{us}(\delta)$ would suffice.
				
				\item \ref{proposition_epochEstErrGuaranteeUniformStability} requires that for $\epkidx = 0,1,\dots$, condition $T_0 \gamma^\epkidx \geq \underline{T}_{rgt, \bar{\epsilon}}^{us}(\frac{3\delta}{\pi^2 (\epkidx+1)^2}, T_0 \gamma^\epkidx)$ holds, which can be satisfied when one chooses $T_0 \geq \Ocal(\underline{T}_{rgt, \bar{\epsilon}}^{us}(\delta, T_0))$.    
				
				\item Proposition \ref{proposition_epochsysid} require $T_0 \gamma^\epkidx {\geq} \max \splitatcommas{\big\{\underline{T}_{MC,1}(\frac{3 \delta}{8 \pi^2 \epkidx^2}), \underline{T}_{id, N}(\frac{3 \delta}{2 \pi^2 (\epkidx+1)^2})\big\}}$, which can be satisfied when we have $T_0 \geq \Ocal(\max \curlybrackets{\underline{T}_{MC,1}(\delta), \underline{T}_{id, N}(\delta)})$.    
				\item Proposition \ref{proposition_epochRegret} requires no conditions on hyper-parameters.
			\end{itemize}
			Therefore, when $\splitatcommas{T_0 \geq \Ocal (\max \curlybrackets{\underline{T}_{\vx_0}^{us}(\delta), \underline{T}_{rgt, \bar{\epsilon}}^{us}(\delta, T_0), \underline{T}_{MC,1}(\delta), \underline{T}_{id, N}(\delta)}) =: \Ocal(\underline{T}_{rgt}^{us} (\delta, T_0))}$, we can apply \ref{proposition_epochInitStateUniformStability}, \ref{proposition_epochEstErrGuaranteeUniformStability}, Proposition \ref{proposition_epochsysid}, and Proposition \ref{proposition_epochRegret} to every epoch $\epkidx = 0, 1, \dots, M-1$. Similar to \eqref{eq_111}, this gives $\P \parenthesesbig{\text{Regret bounds in }\eqref{eq_ExactRegretBoundUniformStability} \text{ holds}} \geq \P ( \cap_{\epkidx = 0}^{M-1} \Ecal_\epkidx) \geq 1 - \delta$.
	\end{proof}
	
		\subsubsection{Proof for Theorem \ref{thrm_randomRegretUniformStability}}\label{proof_thrm_randomRegretUniformStability}
		Since Theorem \ref{thrm_regretUniformStability} shows that $\textnormal{Regret}(T) := \sum_q J_{(\epkidx)} - T J^\star \leq \Ocal(\sqrt{T} \log(\frac{1}{\delta}))$, to upper bound $\textnormal{Regret}^\circ(T) := \sum_q J^\circ_{(\epkidx)} - T J^\star$ in Theorem \ref{thrm_randomRegretUniformStability}, it suffices to upper bound each summand $J^\circ_{(\epkidx)} - J_{(\epkidx)}$. By definition, we further have 
		\begin{align*}
			J^\circ_{(\epkidx)} - J_{(\epkidx)} &= J^\circ_{(\epkidx)} - \expctn[J^\circ_{(\epkidx)} \mid \Fcal_{\epkidx-1}] 
			= J^\circ_{(\epkidx)} - \expctn[J^\circ_{(\epkidx)} \mid \vx_0^{(\epkidx)}, \omega^{(\epkidx)}(0), \vK^{(\epkidx)}_{1:\numSys}]
		\end{align*}
		Hence, we only need to study the deviation of the random cost $J^\circ_{(\epkidx)} $ from its conditional mean $\expctn[J^\circ_{(\epkidx)} \mid \vx_0^{(\epkidx)}, \omega^{(\epkidx)}(0), \vK^{(\epkidx)}_{1:\numSys}]$. Before presenting this result in Lemma \ref{lemma_concentrationLQRCost}, we first provide several supporting results from high-dimensional statistics. In this section, $c$ denotes an absolute constant.
		\begin{lemma}[Theorem 1.1 in \citet{rudelson2013hanson}]\label{lemma_Concentration1}
			Consider a random vector $\vx \in \R^n$ such that $\vx \sim \N(0, \vSigma_\vx)$ and an arbitrary matrix $\vS \in \R^{n \times n}$. Then, with probability at least $1-\delta$,
			\begin{equation}
				|\vx^\top \vS \vx - \expctn[\vx^\top \vS \vx]| \leq c \norm{\vSigma_\vx} \norm{\vS}_\fro \log(\frac{3}{\delta}).
			\end{equation}
		\end{lemma}
		\begin{lemma}[Proposition 5.10 in \citet{vershynin2010introduction}]\label{lemma_Concentration2}
			Consider a random vector $\vx \in \R^n$ such that $\vx \sim \N(0, \vSigma_\vx)$ and an arbitrary vector $\va \in \R^n$. Then, with probability at least $1-\delta$,
			\begin{equation}
				|\va^\top \vx| \leq c \sqrt{\norm{\vSigma_\vx}} \norm{\va} \sqrt{\log(\frac{3}{\delta})}.
			\end{equation}
		\end{lemma}
		\begin{lemma}\label{lemma_Concentration3}
			Consider two independent random vectors $\vx \in \R^{n_\vx}, \vy \in \R^{n_\vy}$ such that $\vx \sim \N(0, \vSigma_\vx)$ and $\vy \sim \N(0, \vSigma_\vy)$, and an arbitrary matrix $\vS \in \R^{n_\vx \times n_\vy}$, then with probability at least $1-\delta$,
			\begin{equation}
				|\vx^\top \vS \vy| \leq c \sqrt{\min\{n_\vx, n_\vy\}} \sqrt{\norm{\vSigma_\vx} \norm{\vSigma_\vy}} \norm{\vS} \log(\frac{6}{\delta}).
			\end{equation}
		\end{lemma}
		\begin{proof}
			By Lemma \ref{lemma_Concentration2}, with probability at least $1-\delta/2$, $\vx^\top \vS \vy \leq c \sqrt{\norm{\vS \vSigma_\vy \vS^\top}} \norm{\vx} \sqrt{\log(\frac{6}{\delta})}$. By Lemma \ref{lemma_Concentration1}, with probability at least $1-\delta/2$, $\norm{\vx}^2 \leq \tr(\vSigma_\vx) + c \norm{\vSigma_\vx} \sqrt{n_\vx} \log(\frac{6}{\delta})$, which further gives $~~~\norm{\vx} \leq$ $ c \sqrt{n_\vx \norm{\vSigma_\vx} \log(\frac{6}{\delta})}$. Combining these two results shows $|\vx^\top \vS \vy| \leq $ $c \sqrt{n_\vx} \sqrt{\norm{\vSigma_\vx} \norm{\vSigma_\vy}}$ $ \norm{\vS} \log(\frac{6}{\delta})$. Similarly, we can show $|\vx^\top \vS \vy| \leq c \sqrt{n_\vy} \sqrt{\norm{\vSigma_\vx} \norm{\vSigma_\vy}} \norm{\vS} \log(\frac{6}{\delta})$, which completes the proof.
		\end{proof}
		\begin{lemma}\label{lemma_Concentration4}
			Consider a vector $\vv := [\vv_1^\top, \vv_2^\top, \vv_3^\top]^\top$ where $\vv_1 \in \R^{n_1}$ is deterministic with $\norm{\vv_1} \leq \bar{v}_1$, and $\vv_2 \in \R^{n_2}$, $\vv_3 \in \R^{n_3}$ are random vectors such that $\vv_2 \sim \N(0, \vSigma_2)$, $\vv_3 \sim \N(0, \vSigma_3)$. Consider an arbitrary symmetric matrix 
			$\vS = \begin{bmatrix}
				\vS_{11} & \vS_{12} & \vS_{13} \\
				\vS_{21} & \vS_{22} & \vS_{23} \\
				\vS_{31} & \vS_{32} & \vS_{33}
			\end{bmatrix}$
			where $\vS_{11} \in \R^{n_1 \times n_1}, \vS_{22} \in \R^{n_2 \times n_2}, \vS_{33} \in \R^{n_3 \times n_3}$. Then, with probability at least $1-\delta$,
			\begin{multline} \nn
				|\vv^\top \vS \vv - \expctn[\vv^\top \vS \vv]| \leq c \Big(\norm{\vSigma_2} \norm{\vS_{22}}_\fro + \norm{\vSigma_3} \norm{\vS_{33}}_\fro 
				+ \sqrt{\min\{n_2, n_3 \}} \sqrt{\norm{\vSigma_2} \norm{\vSigma_3}} \norm{\vS_{23}}  \Big) \log(\frac{18}{\delta})
				\\
				+ c \Big(\sqrt{\norm{\vSigma_2}} \norm{\vS_{12}} + \sqrt{\norm{\vSigma_3}} \norm{\vS_{13}} \Big) \bar{v}_1 \sqrt{\log(\frac{18}{\delta})}.
			\end{multline}
		\end{lemma}
		\begin{proof}
			By triangle inequality,
			\begin{equation}
				|\vv^\top \vS \vv - \expctn[\vv^\top \vS \vv]| \leq d_{22} + d_{33} + 2 d_{23} + 2 d_{12} + 2 d_{23},
			\end{equation}
			where $d_{ij} = |\vv_i^\top \vS_{ij} \vv_j - \expctn[\vv_i^\top \vS_{ij} \vv_j]|$. Then
			\begin{itemize}[leftmargin=1.2em]
				\item By Lemma \ref{lemma_Concentration1}, with probability at least $1-\frac{\delta}{6}$, $d_{22} \leq c \norm{\vSigma_{2}} \norm{\vS_{22}}_\fro \log(\frac{18}{\delta})$.
				\item By Lemma \ref{lemma_Concentration1}, with probability at least $1-\frac{\delta}{6}$, $d_{33} \leq c \norm{\vSigma_{3}} \norm{\vS_{33}}_\fro \log(\frac{18}{\delta})$.
				\item By Lemma \ref{lemma_Concentration3}, with probability at least $1-\frac{\delta}{3}$, $d_{23} {\leq} c \sqrt{\min\{n_2, n_3\} \norm{\vSigma_2} \norm{\vSigma_3}} \norm{\vS_{23}} \log(\frac{18}{\delta})$.
				\item By Lemma \ref{lemma_Concentration2}, with probability at least $1-\frac{\delta}{6}$, $d_{12} \leq c \sqrt{\norm{\vSigma_2}} \norm{\vS_{12}} \bar{v}_1 \sqrt{\log(\frac{18}{\delta})}$. 
				\item By Lemma \ref{lemma_Concentration2}, with probability at least $1-\frac{\delta}{6}$, $d_{13} \leq c \sqrt{\norm{\vSigma_3}} \norm{\vS_{13}} \bar{v}_1 \sqrt{\log(\frac{18}{\delta})}$.     
			\end{itemize}
			Combining these with the union bound concludes the proof.
		\end{proof}
		
		With Lemma \ref{lemma_Concentration4}, we can analyze the concentration of the MJS-LQR cumulative cost around its mean under uniform stability.
		\begin{lemma}\label{lemma_concentrationLQRCost}
			Consider \textup{MJS-LQR}($\vA_{1:\numSys}, \vB_{1:\numSys}, \vT, \vQ_{1:\numSys},$ $ \vR_{1:\numSys}$) with process noise  $\N(0, \sigma_\vw^2 \vI)$, given initial mode $\omega(0)$ and initial state $\vx_0$ such that $\norm{\vx_0} \leq \bar{x}_0$. 
			For a controller $\vK_{1:\numSys}$, the input is given by $\vu_t = \vK_{\omega(t)} \vx_t + \vz_t$ where $\vz_t \sim \N(0, \sigma_\vz^2 \vI)$. Let $\vL_i = \vA_i + \vB_i \vK_i$ for all $i$. Assume there exists $\kappa \geq 1$ and $\theta \in [0,1)$ such that for any sequence $\omega_{1:l} \in [\numSys]^l$ with any $l \in \mathbb{N}$ such that $\norm{\prod_{j=1}^l \vL_{\omega_j}} \leq \kappa \theta^l$. Let $J_T = \sum_{t=0}^T \vx_t^\top \vQ_{\omega(t)} \vx_t + \vu_t^\top \vR_{\omega(t)} \vu_t$ denote the cumulative cost over time horizon $T$. Then, with probability at least $1-\delta$,
			\begin{multline}
				|J_T - \expctn[J_T \mid \omega(0), \vx_0, \vK_{1:\numSys}]| \leq  \frac{c (\dimSt \dimInput)^{1.5} \kappa^2}{(1-\theta)^2} 
				\bigg[
				(\gamma_1 \sigma_\vw^2 + \gamma_2 \sigma_\vz^2 + \gamma_3 \sigma_\vw \sigma_\vz) \sqrt{T} \log(\frac{18}{\delta}) \\
				+ (\gamma_1 \sigma_\vw + \gamma_3 \sigma_\vz) \bar{x}_0 \sqrt{\log(\frac{18}{\delta})} \bigg],
			\end{multline}
			where $\gamma_1:= \norm{\vM_{1:\numSys}}$ for $\vM_i:=\vQ_i + \vK_i^\top \vR_i \vK_i$, $\gamma_2:=\norm{\vM_{1:\numSys}} \norm{\vB_{1:\numSys}} + \norm{\vR_{1:\numSys}}\norm{\vK_{1:\numSys}}$, and $\gamma_3 := \norm{\vM_{1:\numSys}} \norm{\vB_{1:\numSys}}^2 + 2 \norm{\vB_{1:\numSys}} \norm{\vR_{1:\numSys}} \norm{\vK_{1:\numSys}} + \norm{\vR_{1:\numSys}}$.
		\end{lemma}
		\begin{proof}
			First we define a few notations that can convert $J_T$ into the form of vector-matrix multiplications.
			Define the block-diagonal matrix $\vK$ with $T+1$ diagonal blocks such that the $t$-th block is given by $\vK_{\omega(t-1)}$ for all $t$. Similarly, define $\vQ$ for $\vQ_{\omega(0): \omega(T)}$, $\vR$ for $\vR_{\omega(0): \omega(T)}$, and $\vM$ for $\vM_{\omega(0): \omega(T)}$. For all $t$, define
\begin{equation}
				\begin{aligned}
					&\vG_{0,0}^{(0)} := \vI_\dimSt, \  \vG_{0,t}^{(0)} := \prod_{h=0}^{t-1} \vL_{\omega(h)} \\
					&\vG_{t,t}^{(\vw)} := 0, \ \vG_{t-1,t}^{(\vw)} := \vI_\dimSt, \\
				 & \vG_{r,t}^{(\vw)} := \prod_{h=r+1}^{t-1} \vL_{\omega(h)}, \forall r \leq t-2; \\
					&\vG_{t,t}^{(\vz)} := 0, \ \vG_{t-1,t}^{(\vz)} := \vB_{\omega(t-1)}, \\  
					& \vG_{r,t}^{(\vz)} := (\prod_{h=r+1}^{t-1} \vL_{\omega(h)}) \vB_{\omega(r)} , \forall r \leq t-2
				\end{aligned}    
			\end{equation}
			Then, it is easy to derive that
			\begin{equation}
				\vx_t = \vG_{0,t}^{(0)} \vx_0 + \sum_{r=0}^{t} \vG_{r,t}^{(\vw)} \vw_r + \sum_{r=0}^{t} \vG_{r,t}^{(\vz)} \vz_r 
			\end{equation}
			and $\norm{\vG_{0,t}} \leq \kappa \theta^t$, $\norm{\vG_{r,t}^{(\vw)}} \leq \kappa \theta^{t-r-1}$, and $\norm{\vG_{r,t}^{(\vz)}} \leq \kappa \theta^{t-r-1} \norm{\vB_{1:\numSys}}$. Define the following vectors by concatenation.
			\begin{equation}\nn
				\begin{aligned}
					\vx & := [\vx_0^\top, \vx_1^\top, \dots, \vx_T^\top]^\top, \; \vu := [\vu_0^\top, \vu_1^\top, \dots, \vu_T^\top]^\top\\
					\vw &:= [\vw_0^\top, \vw_1^\top, \dots, \vw_T^\top]^\top,  \;\vz := [\vz_0^\top, \vz_1^\top, \dots, \vz_T^\top]^\top, \\
					\vphi &:= [\vx^\top, \vu^\top]^\top, \; \vv := [\vx_0^\top, \vw^\top, \vz^\top]^\top
				\end{aligned}
			\end{equation}
			Define the following block matrices.
			\begin{equation}
				\begin{aligned}
					&\vG^{(0)} := \begin{bmatrix}
						\vG_{0,0}^{(0)} \\ \vG_{0,1}^{(0)} \\ \vdots \\ \vG_{0,T}^{(0)}
					\end{bmatrix}, \
					\vG^{(\vw)} := \begin{bmatrix}
						\vG_{0,0}^{(\vw)} & & &\\ \vG_{0,1}^{(\vw)} & \vG_{1,1}^{(\vw)} & &  \\ \vdots & & \ddots & \\ \vG_{0,T}^{(\vw)} & \vG_{1,T}^{(\vw)} & \dots & \vG_{T,T}^{(\vw)}
					\end{bmatrix}, \\
					&\vG^{(\vz)} := \begin{bmatrix}
						\vG_{0,0}^{(\vz)} & & &\\ \vG_{0,1}^{(\vz)} & \vG_{1,1}^{(\vz)} & &  \\ \vdots & & \ddots & \\ \vG_{0,T}^{(\vz)} & \vG_{1,T}^{(\vz)} & \dots & \vG_{T,T}^{(\vz)}
					\end{bmatrix}, \\
					&\vG := [\vG^{(0)}, \vG^{(\vw)}, \vG^{(\vz)}], \; \vItil := [0_{(T+1)\dimInput \x (T+2)\dimSt}, \vI_{(T+1)\dimInput}].
				\end{aligned}
			\end{equation}
			One can see $\vx  = \vG \vv$, $\vu = \vK \vx + \vz = (\vK \vG + \vItil) \vv$, $\vphi = \begin{bmatrix} \vG \\ \vK \vG  + \vItil \end{bmatrix} \vv$, and $J_T = \vphi^\top \begin{bmatrix} \vQ & \\ & \vR \end{bmatrix} \vphi$. Let $\vS := \vG^\top \vQ \vG + (\vK \vG + \vItil)^\top \vR (\vK \vG + \vItil)$, then these relations give
			\begin{equation}
				J_T = \vv^\top \vS \vv.
			\end{equation}
			Block-partition $\vS$ by $\vS = \begin{bmatrix}
				\vS^{(0,0)} & \vS^{(0,\vw)} & \vS^{(0,\vz)} \\
				\vS^{(\vw,0)} & \vS^{(\vw,\vw)} & \vS^{(\vw,\vz)} \\
				\vS^{(\vz,0)} & \vS^{(\vz,\vw)} & \vS^{(\vz,\vz)}
			\end{bmatrix}$ such that $\vS^{(0,0)} \in \R^{\dimSt \times \dimSt}$, $\vS^{(\vw,\vw)} \in \R^{(T+1)\dimSt \times (T+1)\dimSt}$, $\vS^{(\vz, \vz)} \in \R^{(T+1)\dimInput \times (T+1)\dimInput}$. Then, we have 
			\begin{equation}\nn
				\begin{aligned}
					\vS^{(0,0)} &= {\vG^{(0)}}^\top \vM \vG^{(0)} \\
					\vS^{(\vw,\vw)} &= {\vG^{(\vw)}}^\top \vM \vG^{(\vw)}, \\ 
					\vS^{(\vz,\vz)} &= {\vG^{(\vz)}}^\top \vM \vG^{(\vz)} + \vR \vK \vG^{(\vz)} + {\vG^{(\vz)}}^\top \vK^\top \vR + \vR, \\
					\vS^{(\vw,\vz)} &= {\vG^{(\vw)}}^\top \vM \vG^{(\vz)} + {\vG^{(\vw)}}^\top \vK^\top \vR, \\
					\vS^{(0, \vw)} &= {\vG^{(0)}}^\top \vM \vG^{(\vw)}, \ \ \vS^{(0, \vz)} \\ 
					&= {\vG^{(0)}}^\top \vM \vG^{(\vz)} + {\vG^{(0)}}^\top \vK^\top \vR.
				\end{aligned}
			\end{equation}
			Matrices $\vG^{(0)}, \vG^{(\vw)}, \vG^{(\vz)}$ can be bounded as follows.
			\begin{equation}\nn
				\begin{aligned}
					\norm{\vG^{(0)}} &{\leq} \norm{\vG^{(0)}}_\fro {\leq} \sqrt{\sum_{i=0}^T \kappa^2 \theta^{2i}} {\leq} \frac{\kappa}{1-\theta}, \\
					\norm{\vG^{(\vw)}} &{\leq} \sqrt{\norm{\vG^{(\vw)}}_1 \norm{\vG^{(\vw)}}_\infty} {\leq} \sqrt{\frac{\sqrt{\dimSt} \kappa}{1-\theta} \cdot \frac{\sqrt{\dimSt} \kappa}{1-\theta} } {=} \frac{\sqrt{\dimSt} \kappa}{1-\theta}, \\
					\norm{\vG^{(\vz)}} &{\leq} \sqrt{\norm{\vG^{(\vz)}}_1 \norm{\vG^{(\vz)}}_\infty} \nn \\
					&{\leq} \sqrt{\frac{\sqrt{\dimSt} \kappa \norm{\vB_{1:\numSys}}}{1-\theta}  \frac{\sqrt{\dimInput} \kappa \norm{\vB_{1:\numSys}}}{1-\theta} } {=} \frac{(\dimSt \dimInput)^{0.25} \kappa \norm{\vB_{1:\numSys}}}{1-\theta}.
				\end{aligned}
			\end{equation}
			These results further give
			\begin{equation}
				\begin{aligned}
					&\norm{\vS^{(\vw,\vw)}} \leq \frac{\dimSt \kappa^2 \gamma_1 }{(1-\theta)^2},  \ \ 
					\norm{\vS^{(\vz,\vz)}} \leq  \frac{(\dimSt \dimInput)^{0.5} \kappa^2 \gamma_2 }{(1-\theta)^2} , \nn \\
					&\norm{\vS^{(\vw,\vz)}} \leq \frac{\dimSt^{0.75} \dimInput^{0.25} \kappa^2 \gamma_3 }{(1-\theta)^2} , \nn\\
					&\norm{\vS^{(0, \vw)}} \leq  \frac{\dimSt^{0.5} \kappa^2 \gamma_1 }{(1-\theta)^2} , \ \ 
					\norm{\vS^{(0, \vz)}} \leq \frac{(\dimSt \dimInput)^{0.25} \kappa^2 \gamma_3 }{(1-\theta)^2}. \nn
				\end{aligned}
			\end{equation}
			Finally, we can conclude the proof by invoking Lemma \ref{lemma_Concentration4}.
		\end{proof}
		Now, we are ready to present the main proof of Theorem \ref{thrm_randomRegretUniformStability}.
		
		\begin{proof}
			Following from \ref{lemma_lemma_tac_1a} \ref{proposition_epochInitStateUniformStability}, and the proof of Theorem \ref{thrm_mainthrm_completeUniformStability}, we know with probability at least $1-\delta/2$, for all epochs $\epkidx$,
			\begin{equation}\label{eq_resultsunderUniformStability}
				\begin{aligned}
					&\norm{\vK_{1:\numSys}^{(\epkidx)}} \leq 2 \norm{\vK_{1:\numSys}^\star}, \\
					&\norm{\prod_{j=1}^l \vL_{\omega_j}^{(\epkidx)}} \leq \kappa \theta^l, \quad \forall \omega_{1:l} \in [\numSys]^l, \forall l \in \mathbb{N}, \\
					&\norm{\vx_0^{(\epkidx)}} \leq \Ocal\bigg(\mysqrt[1pt]{\frac{\bar{\kappa} \bar{\sigma}^2}{(1-\bar{\vtheta})^2} \log(\frac{q^2}{\delta})} \bigg), \quad (\epkidx \geq 2).
				\end{aligned}
			\end{equation}
			Under these conditions, and applying Lemma \ref{lemma_concentrationLQRCost}, we know for epoch $\epkidx$ with probability at least $1-\frac{3}{\pi^2} \cdot \frac{\delta}{\epkidx^2}$,
			\begin{equation} \nn
				\begin{split}
					&\big|J^\circ_{(\epkidx)} - J_{(\epkidx)}\big| = \big|J^\circ_{(\epkidx)} - \expctn[J^\circ_{(\epkidx)} \mid \Fcal_{\epkidx-1}]\big| \\
					\leq &
					\Ocal \bigg( \frac{ (\dimSt \dimInput)^{1.5} \bar{\kappa}^2}{(1-\bar{\theta})^2} 
					\bigg[
					(\sigma_\vw^2 + \sigma_{\vz, \epkidx}^2) \sqrt{T_\epkidx} \log(\frac{\epkidx^2}{\delta})
					+ (\sigma_\vw + \sigma_{\vz, \epkidx}) \norm{\vx_0^{(\epkidx)}} \sqrt{\log(\frac{\epkidx^2}{\delta})} \bigg] \bigg)  \\
					\leq &
					\Ocal \bigg( \frac{ (\dimSt \dimInput)^{1.5} \bar{\kappa}^2}{(1-\bar{\theta})^2} 
					\bigg[
					\sigma_\vw^2 \sqrt{\gamma}^\epkidx \log(\frac{\epkidx^2}{\delta})
					+ \sigma_\vw^2 \frac{\sqrt{\bar{\kappa}}}{1-\bar{\theta}} \log(\frac{\epkidx^2}{\delta}) \bigg] \bigg),    
				\end{split}
			\end{equation}
			where the second line follows from $\sigma_{\vz,\epkidx}^2 = \frac{\sigma_\vw^2}{\sqrt{T_\epkidx}}$, $T_\epkidx = \Ocal(\gamma^\epkidx)$, and the bound of $\norm{\vx_0^{(\epkidx)}}$ in \eqref{eq_resultsunderUniformStability}.
			Taking the summation over all $M = \Ocal(\log(T))$ epochs (for simplicity, epoch $0$ and $1$ are ignored) and applying the union bound, we obtain with probability $1-\delta$,
			\begin{equation} \nn
				\begin{split}
				\big|\sum_{\epkidx} J^\circ_{(\epkidx)} - J_{(\epkidx)}\big|
				&\leq 
				\bigg( \frac{ (\dimSt \dimInput)^{1.5} \bar{\kappa}^2 \sigma_\vw^2}{(1-\bar{\theta})^2} 
				\bigg[
				\sqrt{T} \log(\frac{\log^2(T)}{\delta}) 
				+ \frac{\sqrt{\bar{\kappa}}}{1-\bar{\theta}} \log(\frac{\log^2(T)}{\delta}) \bigg] \bigg).
				\end{split}
			\end{equation}
			Combining this with the upper bound on $\textnormal{Regret}(T):=\sum_{\epkidx} J^\circ_{(\epkidx)} - T J^\star$ provided in Theorem \ref{thrm_regretUniformStability} completes the proof.
		\end{proof}

\end{document}